\documentclass[11pt,letterpaper,hidelinks]{article}

\usepackage[margin=1in]{geometry}

\usepackage[T1]{fontenc}
\usepackage{amsmath,amssymb,amsfonts,amsthm,mathtools}
\usepackage{graphicx}
\usepackage{algorithm}
\usepackage{algpseudocode}

\usepackage{enumitem}
\usepackage{hyperref}
\usepackage[nameinlink,capitalise]{cleveref}
\usepackage{natbib}
\usepackage{bbm}
\usepackage{authblk}

\newtheorem{theorem}{Theorem}
\newtheorem{lemma}[theorem]{Lemma}
\newtheorem{corollary}[theorem]{Corollary}
\newtheorem{proposition}[theorem]{Proposition}
\newtheorem{definition}[theorem]{Definition}
\newtheorem{remark}[theorem]{Remark}
\newtheorem{assumption}[theorem]{Assumption}

\title{Asymptotically Optimal Learning 
for Parametric Prophet Inequalities}

\author[1]{Jung-hun Kim}
\author[2]{Anna Grebennikova}
\author[1,3]{Vianney Perchet}

\affil[1]{FairPlay Team, CREST, ENSAE, Institut Polytechnique de Paris}
\affil[2]{UFR IM2AG, Université Grenoble Alpes}
\affil[3]{Criteo AI Lab}
\date{
\small
\texttt{junghun.kim@ensae.fr} \quad
\texttt{anna.grebennikova@etu.univ-grenoble-alpes.fr} \quad
\texttt{vianney.perchet@normalesup.org}
}

\begin{document}

\maketitle
\begin{abstract}
We study learning in prophet inequalities with i.i.d. rewards drawn from an
exponential-type parametric family with an unknown parameter $\theta$, a class that
includes exponential, Pareto, and bounded-support power-family distributions.
We first characterize the optimal full-information asymptotic competitive
ratio for this family. In the unbounded-support case, the limit is $
{\left({\theta}/({\theta-c_+})\right)^{c_+/\theta}}/
{\Gamma(1-c_+/\theta)},$ while in the bounded-support case, the limit is $1$. We then propose a confidence-based
dynamic-programming policy for online learning.  By exploiting
the explicit parametric structure, the policy achieves the same optimal
asymptotic competitive ratio using only online observations, without external
offline samples.
We further derive distribution-specific convergence rates for canonical
examples. Finally, numerical experiments on synthetic instances illustrate
the performance of our algorithm.
\end{abstract}

\section{Introduction}
\label{sec:introduction}

Prophet inequalities are a fundamental model of online decision-making and
optimal stopping \citep{hill1992survey}. A decision-maker sequentially observes
rewards and must irrevocably decide when to stop and accept the current reward.
The benchmark is the prophet, who observes all realizations in advance and
obtains the maximum reward. The goal is to design a stopping rule whose expected
reward is competitive with this prophet benchmark. This framework has attracted
significant attention due to its rich mathematical structure and its broad range
of applications, including posted-price mechanisms \citep{lucier2017economic},
online ad allocation \citep{alaei2012online}, and hiring processes in labor
markets \citep{arsenis2022individual}.

Classical prophet inequalities assume that the reward distributions are known.
Under this assumption, for i.i.d.\ reward distributions,
\citet{hill1982comparisons} established the classical \(1-1/e\) guarantee, and
subsequent work characterized the optimal full-information i.i.d.\ worst-case
constant \(\kappa\approx0.745\) \citep{correa2017posted,correa2019recent}.

However, full distributional knowledge is rarely available in applications.
This has motivated a growing literature on prophet inequalities with unknown
distributions. A central message of this literature is that, for arbitrary
unknown distributions, learning is severely limited without sufficient offline
samples: in the i.i.d.\ setting, one cannot improve over the competitive ratio
\(1/e\) with \(o(n)\) offline samples \citep{correa2019prophet}. Achieving
near-optimal worst-case guarantees for arbitrary distributions requires
substantial distributional information, such as \(\Theta(n)\) or more offline
samples \citep{rubinstein2019optimal}. Thus, in the absence of structure, learning
prophet inequalities is constrained by strong sample-complexity barriers.

A complementary perspective is to seek distribution-specific asymptotic
performance beyond distribution-free worst-case guarantees.
\citet{goldenshluger2022optimal} pursue this direction using nonparametric
relative-rank-based rules, establishing first-order asymptotic optimality in the
Gumbel and reverse-Weibull domains. However, because their rules use only
relative ranks, they cannot exploit parametric information in the observed
reward values. This restriction is especially pronounced in the Fr\'echet
heavy-tailed regime: no relative rank-based rule can attain the
full-information Fr\'echet
heavy-tailed limit, as discussed later. Even where the relative rank-based framework
applies, its nonparametric rates can be slower than those achievable under
stronger parametric structure.


Another closely related direction is online learning in structured prophet
inequality models. \cite{kimlearning} study prophet inequalities under a linear
reward structure with an unknown parameter and a known feature distribution.
Their objective, however, is to recover the classical distribution-agnostic
guarantee of \(1-1/e\). This guarantee is robust but not distribution-specific:
it can be strictly below the optimal ratio achievable within a given
distribution class, and is also below the optimal full-information i.i.d.\
worst-case constant \(\kappa\approx0.745\).

These works motivate the following question:
\emph{Can one exploit parametric distributional structure to learn the unknown parameter online and match the full-information distribution-specific
benchmark?} We answer this question
affirmatively for an exponential-type parametric family, which includes
light-tailed, heavy-tailed, and bounded-support models such as exponential,
Pareto, and bounded-support power-family distributions. Our goal is not to recover a distribution-agnostic worst-case constant, but to
learn the parameter from online observations alone and match the full-information
distribution-specific asymptotic competitive ratio.
Our main contributions
are summarized as follows.
\begin{itemize}
    \item \textbf{Optimal full-information asymptotics.}
    We characterize the asymptotically optimal competitive ratio for the proposed
    exponential-type family when the distribution parameter \(\theta\) is known.
    In the unbounded-support case, the limiting optimal ratio is
    \(
    \frac{1}{\Gamma(1-c_+/\theta)}
    \left(\frac{\theta}{\theta-c_+}\right)^{c_+/\theta},
    \)
    which is governed by the endpoint-growth parameter \(c_+\). In the
    bounded-support case, the optimal competitive ratio converges to \(1\).

   \item \textbf{Online learning with asymptotically optimal guarantees.}
    We propose a confidence-based dynamic-programming policy for online
    learning. The policy first estimates the unknown distribution parameter
    from an initial exploration phase, constructs an upper confidence bound,
    and then applies the corresponding plug-in DP thresholds. By exploiting the
    explicit parametric structure, the proposed policy achieves the same
    asymptotically optimal competitive ratio as the optimal known-parameter
    policy, using only online observations and no external offline samples.

    \item \textbf{Distribution-specific convergence guarantees.}
       We derive refined convergence guarantees for representative distributions,
    including exponential, Pareto, and bounded-support power-family rewards.
    For exponential and bounded-support examples, our policy matches the
    convergence rates of the corresponding full-information optimal policies.
    For Pareto heavy-tailed rewards, we obtain a Pareto-specific convergence
    guarantee showing convergence to the full-information limiting ratio. These results show how tail or endpoint structure
    governs the speed of convergence, and are supported by numerical
    experiments.
\end{itemize}







\section{Related Work}
\label{sec:related-work}

\paragraph{Prophet inequalities with known distributions.}
The classical prophet inequality was initiated by 
\cite{krengel1978semiamarts}. Building on this line of work, \cite{samuel1984comparison} showed that a single
threshold achieves the optimal \(1/2\) competitive ratio for independent non-identical rewards.
For i.i.d. rewards, \cite{hill1982comparisons} established stronger guarantees of $1-1/e$,
and later work characterized the optimal full-information i.i.d. competitive ratio as constant
\(\kappa\approx 0.745\) \citep{correa2017posted}.
These works assume that the reward distribution is known, whereas our focus is on learning the
unknown parameter of a structured distribution online.

\paragraph{Unknown distributions and sample-based prophet inequalities.}
\cite{correa2019prophet} studied i.i.d. prophet inequalities with an
unknown distribution and showed that, with no samples or even \(o(n)\) samples, the best possible
ratio is \(1/e\).  They also showed that \(n-1\) samples yield a \(1-1/e\) guarantee and that
\(O(n^2)\) samples suffice to approach the optimal full-information constant $\kappa$. \cite{rubinstein2019optimal} improved the i.i.d. sample complexity to
\(O(n/\varepsilon^6)\) for an \(\kappa-O(\varepsilon)\) guarantee. 
Our work is
complementary: instead of using external offline samples from an arbitrary distribution, we exploit a
parametric form and learn the unknown parameter from the online observations.

\paragraph{Rank-based stopping and extreme-value theory.}
\citet{goldenshluger2022optimal} studied optimal stopping for i.i.d. observations
from an unknown continuous distribution using relative-rank-based rules. Their
analysis is closely connected to extreme-value theory: they establish first-order
asymptotic optimality in the Gumbel and reverse-Weibull domains, whereas in the
Fr\'echet domain vanishing relative regret is impossible. Our endpoint regularity
conditions admit a similar extreme-value interpretation. In contrast to their nonparametric, rank-based approach, we
exploit parametric rewards to attain optimal distribution-specific
asymptotic ratios across light-tailed, heavy-tailed, and bounded-support models
within our family, with sharper convergence rates in canonical examples.


\paragraph{Structured and parametric learning in prophet inequalities.}
\citet{kimlearning} study prophet inequalities
under a linear reward structure with an unknown parameter and a known feature
distribution, and recover
the classical \(1-1/e\) guarantee. Their setting is structurally different
from ours, and the resulting guarantee is distribution-agnostic rather than
distribution-specific. In particular, it can be below the optimal ratio for a
given distribution class and the optimal worst-case full-information i.i.d.
constant \(\kappa\approx0.745\). Parametric learning in optimal
stopping was also studied by \citet{martinsek1984approximations} for
exponential rewards with an unknown mean. In contrast, we consider
finite-horizon prophet inequalities for a parametric exponential-type family
with an unknown distribution parameter, and learn this parameter online to
achieve the optimal distribution-specific competitive~ratio.

\section{Problem Statement}

We consider the classical prophet-inequality
framework over a finite horizon $n\in\mathbb{N}$, where an i.i.d. nonnegative random
reward $X_i\ge 0$ is revealed sequentially to the decision maker (or
{gambler}) at each stage $i\in[n]$. For the reward distribution, we consider a \emph{exponential-type
parametric family}, defined as follows.

\begin{definition}[Exponential-Type Parametric Family]\label{def:distribution}
We consider a parameter family with latent $\theta>0$ and an  upper support
$x_F\in (x_0,\infty]$, whose CDF is
\[
F_\theta(x)=
\begin{cases}
0, & x<x_0,\\[1mm]
1-\exp(-\theta \phi(x)), & x_0\le x < x_F,\\[1mm]
1, & x\ge x_F,
\end{cases}
\]
where $\phi:[x_0,x_F)\to[0,\infty)$ is a known strictly increasing function with
$\phi(x_0)=0$ and
$
\phi(x)\to\infty \text{ as } x\uparrow x_F.$

\end{definition}

\paragraph{Online stopping objective.}
Let $X_1,\dots,X_n$ be i.i.d.\  rewards drawn from $F_\theta$ under an \emph{unknown} parameter $\theta$.
A policy is a stopping time $\tau\in\{1,\dots,n\}\cup\{n+1\}$ adapted to the sequential observations
$(X_t)_{t\ge 1}$: after observing $X_t$, the decision-maker either stops and receives reward $X_t$,
or continues to $t+1$ (and can never return to $t$). We define $X_{n+1}:=0$ to be the reward for the case when the stopping rule continues over all periods in $[n]$. The benchmark is the prophet value
$\mathrm{OPT}(\theta):=\mathbb{E}_\theta[\max_{t\in[n]} X_t]$, and the goal is to design a policy (not knowing $\theta$)
with a large competitive ratio
\[
\mathrm{CR}_n(\tau;\theta)= \frac{\mathbb{E}_\theta[X_\tau]}{\mathrm{OPT}(\theta)},
\]
where \(\mathbb{E}_\theta\) denotes expectation under
\(X_1,\dots,X_n \stackrel{\mathrm{i.i.d.}}{\sim} F_\theta\); we write
\(\mathbb{E}\) when \(\theta\) is clear.


\section{Preliminary Results}

In this section, we derive an estimation guarantee for the
exponential-type parametric family and introduce benchmark policies for
comparison.

\subsection{Estimator and Confidence Bound}

Let \(Y_i := \phi(X_i)\) for \(i \in [n]\). Then \(Y_i \sim \mathrm{Exp}(\theta)\), where
\(\mathrm{Exp}(\theta)\) denotes the exponential distribution with rate \(\theta\). Since an exponential random variable with rate \(\theta\) has mean \(1/\theta\), we have
$\theta=\frac{1}{\mathbb{E}[Y_i]}.$
This motivates the plug-in estimator
\(
\widehat{\theta}_m
=\left(\frac{1}{m}\sum_{i=1}^m Y_i\right)^{-1},
\)
and the corresponding estimated cumulative distribution function
$
F_{\widehat{\theta}_m}(x)=1-\exp\bigl(-\widehat{\theta}_m\phi(x)\bigr).
$
We next provide a confidence bound for \(\widehat{\theta}_m\); see Appendix~\ref{app:confidence} for the proof.

\begin{lemma}\label{lem:confi} For $\varepsilon\in(0,1)$, when $m\ge \frac{4}{\varepsilon^2}\log(2/\delta)$ for $\delta\in(0,1]$, with probability at least $1-\delta$, \[\left|\frac{\theta}{\widehat{\theta}_m}-1\right|\le \varepsilon.\]
\end{lemma}

\subsection{Baseline Policies}

\label{subsec:benchmarks}
We next discuss two natural baseline policies. The first is a
distribution-specific nonparametric baseline based on relative ranks. The
second is a distribution-agnostic sample-based baseline using the confidence
bound from the previous section.
\paragraph{Rank-based distribution-specific baseline.}
A natural nonparametric baseline is the relative-rank-based rule of
\citet{goldenshluger2022optimal}, which uses only the relative ranks of the
observations. It achieves first-order asymptotic optimality, in the sense of
vanishing relative regret, in the Gumbel and reverse-Weibull domains.
Specifically, for the exponential distribution, their guarantees imply
\(
\mathrm{CR}_n(\tau;\theta)
\ge
1-O\!\left({\log\log\log n}/{\log n}\right),
\)
whereas for the uniform distribution on a bounded interval, they imply
\(
\mathrm{CR}_n(\tau;\theta)
\ge
1-O\!\left({\log n}/{n}\right),
\)
as \(n\to\infty\). These convergence rates towards the optimal limit of $1$ are useful nonparametric baselines, but they are not tailored to
parametric subfamilies where sharper rates are possible. Moreover, as
discussed later, the relative rank-based framework does not recover the full-information
distribution-specific limit in the Fr\'echet heavy-tailed domain.

\paragraph{Distribution-agnostic sample-based baseline.}
Another natural reference is the \texttt{Samples-CFHOV} rule of
\citet{rubinstein2019optimal}. We adapt it to our online setting by estimating
the unknown parameter from the initial observations and then applying
sample-based thresholds under the estimated model. With a suitable exploration
length, this baseline achieves a competitive ratio arbitrarily close to the
optimal i.i.d.\ worst-case constant \(\kappa\approx0.745\); see
Appendix~\ref{app:baseline}. However, it remains distribution-agnostic and does
not target the sharper full-information distribution-specific baseline.

Our main policy instead uses confidence-based DP thresholds for online learning, tailored to the
exponential-type family. This allows us to exploit the parametric structure and
achieve the optimal distribution-specific asymptotic competitive ratio,
including for heavy-tailed Pareto-type rewards.

\section{Optimal Competitive Ratio}

To characterize the full-information asymptotic benchmark, we impose the
following endpoint regularity condition.

\begin{assumption}[Endpoint regularity]\label{ass:c_conv}
Let $\psi:=\phi^{-1}$, so that $\psi:[0,\infty)\to[x_0,x_F)$. Assume that $\psi$ is eventually twice differentiable and that one of the
following two support regimes holds.

\smallskip
\noindent
\textit{(i) Unbounded-support case:} \(x_F=\infty\).  
There exist upper-tail growth constants  \(c_+,c_+'\in[0,\theta)\) s.t.
\[
\lim_{y\to\infty}\frac{\psi'(y)}{\psi(y)}=c_+,
\qquad
\lim_{y\to\infty}\frac{\psi''(y)}{\psi'(y)}=c_+'.
\]

\smallskip
\noindent
\textit{(ii) Bounded-support case:} \(x_F<\infty\).   Let \(\bar\psi(y):=x_F-\psi(y)\).
There exist endpoint-gap decay constants \(c_-,c_-'\in[0,\infty)\) such that
\[
\lim_{y\to\infty}\frac{\bar{\psi}'(y)}{\bar{\psi}(y)}=-c_-,
\qquad
\lim_{y\to\infty}\frac{\bar{\psi}''(y)}{\bar{\psi}'(y)}=-c'_-.
\]
\end{assumption}

\begin{remark}[The two endpoint constants coincide]
Assumption~\ref{ass:c_conv} is written with two limits in each endpoint
regime in order to separate the endpoint-growth condition from the
corresponding smoothness condition. However, the two constants are not
independent. Whenever both limits exist, they necessarily coincide:
\[
c_+=c_+'
\qquad\text{and}\qquad
c_-=c_-'.
\]
The proof is given in Appendix~\ref{app:c_equal}. Thus, throughout the
paper, we write simply \(c_+\) in the unbounded-endpoint case and \(c_-\)
in the bounded-endpoint case.
\end{remark}

\paragraph{Endpoint indices and finite prophet benchmarks.}
 Since \(\phi\) is strictly increasing, its inverse \(\psi\) is also
strictly increasing. Hence \(\psi'(y)\ge 0\) wherever the derivative
exists, and the endpoint indices \(c_+\) and \(c_-\) are necessarily
nonnegative. The restriction \(c_+<\theta\) is imposed only in the unbounded-endpoint
case. It is an integrability condition: if \(c_+>\theta\), then\footnote{In such cases, the prophet benchmark is infinite, so the standard
competitive-ratio formulation is not well defined.}
\(\mathbb{E}[X_i]=\infty\), and in the boundary case \(c_+=\theta\), the
mean may still be infinite in general; see Appendix~\ref{app:c_just}.
In contrast, when \(x_F<\infty\), the support is bounded above, so
$
\mathbb{E}[X_i]\le x_F<\infty$. Therefore no analogous restriction on \(c_-\) is needed.

\paragraph{Canonical examples.}
 This family includes several common distributions as special cases.
In the unbounded-endpoint case, the exponential distribution corresponds to
\(\phi(x)=x\) on \([0,\infty)\), for which \(\psi(y)=y\) and hence \(c_+=0\). Another example is the Pareto family, obtained from
\(\phi(x)=\log(x/x_{0})\) on \([x_{0},\infty)\). In this case
\(\psi(y)=x_{0}e^y\), so \(c_+=1\). In the bounded-support case, a natural example is bounded-support power-family distributions with
$
\phi(x)=\log\!\left(\frac{x_F-x_0}{x_F-x}\right)$ for $
 x\in[x_0,x_F),$
for which \(\bar\psi(y)=(x_F-x_0)e^{-y}\) and hence \(c_-=1\).
This includes the uniform distribution on \([x_0,x_F]\) as the special case \(\theta=1\).
\vspace{2mm}

\begin{remark}[Extreme-value interpretation]
\label{rm:evt_interpretation}
The endpoint conditions above have a natural interpretation in extreme-value
theory. Let
\(
U(t)=F_\theta^{-1}(1-1/t)=\psi\!\left({\log t}/{\theta}\right).
\)
If \(x_F=\infty\) and \(\psi'(y)/\psi(y)\to c_+\), then \(U\) is regularly
varying with index \(c_+/\theta\). Hence, when \(c_+>0\), \(F_\theta\) belongs
to the Fr\'echet domain of attraction with extreme-value index
\(\gamma=c_+/\theta\); the case \(c_+=0\) is consistent with Gumbel-type
behavior. If \(x_F<\infty\) and
\(\bar\psi'(y)/\bar\psi(y)\to -c_-\), then \(x_F-U(t)\) is regularly varying
with index \(-c_-/\theta\), corresponding to the reverse-Weibull domain with
extreme-value index \(\gamma=-c_-/\theta\). The derivation is given in
Appendix~\ref{app:evt_connection}.
\end{remark}

Under Assumption~\ref{ass:c_conv}, we now analyze the asymptotically optimal
competitive ratio for the exponential-type family. The characterization depends
on the tail parameter \(c_+\) in the unbounded-support case and on the endpoint
parameter \(c_-\) in the bounded-support case. We first provide the optimal ratio for the unbounded support case, and the proof is provided in 
Appendix~\ref{app:general-phi}.



\begin{theorem}[Optimal ratio for the unbounded-support case]\label{thm:general-phi}
Suppose Assumption~\ref{ass:c_conv}(i) holds. Let \(\tau^*\) be the optimal full-information stopping rule. Then, for \(c_+\in[0,\theta)\), we have
\[
\mathbb E_\theta[X_{\tau^*}]
=
\left(\frac{\theta}{\theta-c_+}\right)^{c_+/\theta}
\psi\!\left(\frac{\log n}{\theta}\right)(1+o(1)),\]\[
\mathrm{OPT}(\theta)
=
\Gamma\!\left(1-\frac{c_+}{\theta}\right)
\psi\!\left(\frac{\log n}{\theta}\right)(1+o(1)),
\]
and therefore
\[
\lim_{n\to\infty}\mathrm{CR}_n(\tau^*;\theta)
=
\frac{\left(\frac{\theta}{\theta-c_+}\right)^{c_+/\theta}}
{\Gamma\!\left(1-\frac{c_+}{\theta}\right)}:=\rho(c_+,\theta).
\]
\end{theorem}
We next provide the optimal ratio for the bounded-support case; the proof is in
Appendix~\ref{app:general-phi-finite}.

\begin{theorem}[Optimal ratio for the bounded-support case]\label{thm:general-phi-finite}
Suppose Assumption~\ref{ass:c_conv}(ii) holds. Let \(\tau^*\) be the optimal full-information stopping rule. Then, for \(c_-\ge 0\),
\hspace{-2mm}
\[
\mathbb E_\theta[X_{\tau^*}]
=
x_F-
\left(\frac{\theta+c_-}{\theta}\right)^{c_-/\theta}
\bar\psi\!\left(\frac{\log n}{\theta}\right)(1+o(1)),\]\[
\mathrm{OPT}(\theta)
=
x_F-
\Gamma\!\left(1+\frac{c_-}{\theta}\right)
\bar\psi\!\left(\frac{\log n}{\theta}\right)(1+o(1)),
\]
and, since $\bar{\psi}(\log n/\theta) \to 0$, 
\[
\lim_{n\to\infty}\mathrm{CR}_n(\tau^*;\theta)=1.
\]
\end{theorem}

\begin{remark}[Connection to extreme-value asymptotics]
The full-information limits above are consistent with EVT-based
characterizations of i.i.d. prophet inequalities
\citep{kennedy1991asymptotic,livanos2025minimization}. In the
unbounded-support case, the extreme-value index is
\(\gamma=c_+/\theta\in[0,1)\), so
\(
\rho(c_+,\theta)
=
{(1-\gamma)^{-\gamma}}/{\Gamma(1-\gamma)}.
\)
In the bounded-support case, \(\gamma=-c_-/\theta\le 0\), and the standard
maximization competitive ratio converges to \(1\). Our main contribution is to
attain these full-information benchmarks under an unknown parameter using only
online observations.
\end{remark}

\section{Confidence-based Dynamic-Programming Policy for Online Learning}

To target the distribution-specific optimal competitive ratio analyzed in the previous section,
we propose a confidence-based dynamic-programming (DP) policy for online learning; see Algorithm~\ref{alg:general-phi-dp}. The algorithm first uses the initial \(m\) periods for exploration, during which it only collects observations. It then computes the estimator \(\widehat{\theta}_m\), constructs an upper confidence bound \(\theta^{(U)}\), and applies the DP policy obtained by plugging \(\theta^{(U)}\) into the value recursion.

Although \(\theta^{(U)}\) is an upper confidence bound, a larger rate parameter
makes the reward distribution stochastically smaller. Thus, the surrogate model
yields conservative continuation values and lower plug-in DP thresholds on the
event \(\theta^{(U)}\ge\theta\). This prevents the policy from being overly
selective under parameter uncertainty.


\begin{algorithm}[t]
\caption{Confidence-based DP for Online Learning  (\texttt{CDP-OL})}\label{alg:general-phi-dp}
\begin{algorithmic}[1]
\Ensure stopping time \(\tau\)

\Statex \textit{// Exploration phase}
\For{\(i=1\) to \(m\)}
    \State observe \(X_i\) and reject it
\EndFor
\State set \(Y_i\gets \phi(X_i)\) for \(i=1,\dots,m\)
\State \(\widehat\theta_m \gets \dfrac{m}{\sum_{i=1}^m Y_i}\);  \(\quad\varepsilon_m \gets \sqrt{\dfrac{4\log(2/\delta)}{m}}\); \(\quad\theta^{(\mathrm U)} \gets (1+\varepsilon_m)\widehat\theta_m\)
\State \(N\gets n-m\)

\Statex \textit{// Backward DP under the surrogate parameter \(\theta^{(\mathrm U)}\)}
\State \(\widehat V_N \gets \mathbb E_{\theta^{(\mathrm U)}}[X_1]\)
\For{\(i=N-1,N-2,\ldots,1\)}
    \State \(\widehat V_i \gets \widehat V_{i+1}
    + r_{\theta^{(\mathrm U)}}(\widehat V_{i+1})\),
    where \(r_\eta(a):=\int_a^{x_F}\exp(-\eta\phi(t))\,dt\)
    \label{line:DP}
\EndFor
\Statex \textit{// Online stopping phase}
\For{\(t=1\) to \(N-1\)}
    \State observe \(X_{m+t}\)
    \If{\(X_{m+t}\ge \widehat V_{t+1}\)}
        \State \Return \(m+t\)
    \EndIf
\EndFor
\State \Return \(n\)
\end{algorithmic}
\end{algorithm}



\subsection{Unbounded-support case}
We first analyze the guarantee of Algorithm~\ref{alg:general-phi-dp} in the unbounded-support case, where \(x_F=\infty\), under Assumption~\ref{ass:c_conv}(i). The proof of the
following theorem is given in Appendix~\ref{app:general-phi-opt-rate}.

\begin{theorem}
\label{cor:general-phi-conv-rate}
For any \(\delta\in(0,1)\), let the exploration length satisfy  $
m=\omega(\log^2 n \log(1/\delta))$  and $ m=o(n)$. 
Then the stopping policy $\tau$ of Algorithm~\ref{alg:general-phi-dp} satisfies, 
\[
\liminf_{n\to\infty}\mathrm{CR}_n(\tau;\theta)
\ge
(1-\delta)\rho(c_+,\theta),
\]
where, recall, $\rho(c_+,\theta)$ is the optimal ratio of ${\left(\frac{\theta}{\theta-c_+}\right)^{c_+/\theta}}
/{\Gamma\!\left(1-\frac{c_+}{\theta}\right)}.$
\end{theorem}

Theorem~\ref{cor:general-phi-conv-rate} shows that
Algorithm~\ref{alg:general-phi-dp} attains the full-information optimal ratio
of Theorem~\ref{thm:general-phi} in the unbounded-support case, up to the
confidence factor \(1-\delta\). Since \(\delta\in(0,1)\) is arbitrary, the
limiting guarantee can be made arbitrarily close to \(\rho(c_+,\theta)\).

When \(c_+>0\), the extreme-value index is
\(\gamma=c_+/\theta\in(0,1)\), placing the model in the Fr\'echet heavy-tailed
domain, where the guarantees of relative rank-based rule in 
\citet{goldenshluger2022optimal} do not apply. Moreover,
Appendix~\ref{app:limit_rank} shows that any relative rank-based rule is
asymptotically bounded away from the full-information Fr\'echet benchmark.
Thus, exploiting parametric reward values is essential for recovering the
optimal heavy-tailed limit.


\subsection{Bounded-support case}

Now we analyze the guarantee of Algorithm~\ref{alg:general-phi-dp} in the bounded-support case, where \(x_F<\infty\), under Assumption~\ref{ass:c_conv}(ii). The proof of the
following theorem is given in Appendix~\ref{app:general-phi-online-rw}.

\begin{theorem}
\label{thm:general-phi-online-rw}
For any \(\delta\in(0,1)\), let the exploration length satisfy $
m=o(n)$ and $
m=\omega\bigl(\log(1/\delta)\log^2 n\bigr).$ Then the stopping policy $\tau$ of Algorithm~\ref{alg:general-phi-dp} satisfies, 
\[
\liminf_{n\to\infty}\mathrm{CR}_n(\tau;\theta)\ge 1-\delta.
\]

\end{theorem}

 In Theorem~\ref{thm:general-phi-online-rw},  we show that our algorithm achieves asymptotically optimal competitive ratio in the bounded-support case, matching the optimal ratio in Theorem~\ref{thm:general-phi-finite}.

\section{Canonical Examples: Exponential, Pareto, and Bounded-Support Families}

In this section, we instantiate the general framework for exponential, Pareto,  and
bounded-support models. In each case, the explicit form of \(\phi\) allows us to
compute the tail integral \(r_\eta(a)\), derive the corresponding plug-in DP
recursion, and obtain sharper model-specific convergence guarantees. We also
compare these rates with those of \cite{goldenshluger2022optimal}.





\subsection{Exponential Distribution}

We first investigate the case of an exponential reward distribution such that
\(X_i\sim \mathrm{Exp}(\theta)\) with $\theta>0$. This corresponds to
\(\phi(x)=x\) on \([0,\infty)\), so that
$F_\theta(x)=1-\exp(-\theta x)$ for $x\ge 0.$
This distribution falls into the unbounded-endpoint case and satisfies
\(c_+=0\). Hence, the limiting optimal competitive ratio is
\(\rho(0,\theta)=1\).

For the estimated value update in Line~\ref{line:DP} of
Algorithm~\ref{alg:general-phi-dp}, note that for any \(a\ge 0\),
\(
r_\eta(a)
=
\int_a^\infty e^{-\eta\phi(t)}\,dt
=
\int_a^\infty e^{-\eta t}\,dt
=
\frac{e^{-\eta a}}{\eta}.
\)
Hence the plug-in DP recursion becomes
\begin{align}\label{eq:DP_exp}
    \widehat V_N
    =
    \mathbb E_{\theta^{(\mathrm U)}}[X_1]
    =
    \frac{1}{\theta^{(\mathrm U)}},\quad
    \widehat V_i
    =
    \widehat V_{i+1}
    +
    \frac{\exp\!\left(-\theta^{(\mathrm U)}\widehat V_{i+1}\right)}
         {\theta^{(\mathrm U)}},
    \qquad i=1,\dots,N-1.
\end{align}


That is, at each stage \(t\in [N]\), the algorithm uses the corresponding plug-in DP threshold determined by the above recursion. 
The following theorem gives its competitive
ratio; the proof is in Appendix~\ref{app:exp_cr}.

\begin{theorem}\label{thm:exp_cr}
 Suppose \(X_1,\dots,X_n \stackrel{\mathrm{i.i.d.}}{\sim} \mathrm{Exp}(\theta)\). For any \(\delta\in(0,1)\), let $
m=o(n)$  and $m=\omega(\log(1/\delta)\log^2n)$.
Then the stopping policy \(\tau\) of Algorithm~\ref{alg:general-phi-dp} satisfies
\[
\mathrm{CR}_n(\tau;\theta)
\ge
(1-\delta)\left(
1
-
O\!\left(\sqrt{\frac{\log(1/\delta)}{m}}\right)
-
O\!\left(\frac{m}{n\log n}\right)-O\left(\frac{1}{\log n}\right)
\right).
\]
\end{theorem}
The three error terms have different origins. The first term comes from
estimating the unknown parameter, the second from discarding the first \(m\)
observations for exploration, and the last term is the intrinsic
full-information finite-horizon gap. In particular, Theorem~\ref{thm:exp_cr} implies that Algorithm~\ref{alg:general-phi-dp} is asymptotically optimal in the exponential case, achieving the optimal competitive ratio \(1\) from Theorem~\ref{thm:general-phi} when \(c_+=0\). Optimizing the exploration length \(m\) then yields the following convergence~rate.

\begin{corollary}\label{cor:exp_cr}
Suppose \(X_1,\dots,X_n \stackrel{\mathrm{i.i.d.}}{\sim} \mathrm{Exp}(\theta)\). For any \(\delta\in(0,1)\), let $
m=(n\log n)^{2/3}\log^{1/3}(1/\delta).$
Then the stopping policy \(\tau\) of Algorithm~\ref{alg:general-phi-dp} satisfies
\[
\mathrm{CR}_n(\tau;\theta)\ge
(1-\delta)\left(1-O\!\left(\frac{\log^{1/3}(1/\delta)}{(n\log n)^{1/3}}\right)-
O\!\left(\frac{1}{\log n}\right)\right).
\]
\end{corollary}
By taking \(\delta=\delta_n=1/n\), the estimation and exploration errors become
lower order, and the overall competitive-ratio gap is dominated by the
intrinsic full-information finite-horizon gap of order \(1/\log n\). Thus,
the rate matches that of the optimal stopping policy when the distribution
parameter is known, as shown below. Compared with the relative rank-based method of \cite{goldenshluger2022optimal}, whose
guarantee for the exponential case is
\(O(\log\log\log n/\log n)\), our parametric approach exploits the
exponential-type structure to attain the full-information-optimal
\(O(1/\log n)\) rate, thereby removing the extra \(\log\log\log n\) factor. The precise full-information convergence rate is
given in the following proposition, whose proof is provided in
Appendix~\ref{app:exp_full_info_rate}.
\begin{proposition}[Full-information convergence rate for exponential rewards]
\label{thm:exp_full_info_rate}
Suppose \(X_1,\dots,X_n\stackrel{\mathrm{i.i.d.}}{\sim}\mathrm{Exp}(\theta)\),
and let \(\tau^*\) be the optimal full-information stopping rule. Then,
\[
\mathrm{CR}_n(\tau^*;\theta)
=
1-\Theta\left(\frac{1}{\log n}\right).
\]
\end{proposition}

\subsection{Pareto Distribution}
We now consider Pareto rewards \(X_i\sim \mathrm{Pareto}(\theta)\) with
shape parameter \(\theta>1\) (otherwise, \(\mathbb{E}[X_i]=\infty\)), and scale parameter \(x_0>0\).
This corresponds to $
\phi(x)=\log(x/x_0)$ for $x\ge x_0,$
so that
$F_\theta(x)=1-\left(\frac{x_0}{x}\right)^\theta$ for $
 x\ge x_0.$
This distribution has an unbounded endpoint and satisfies \(c_+=1\).

For the estimated value update in Line~\ref{line:DP} of
Algorithm~\ref{alg:general-phi-dp}, note that for any \(a\ge x_0\) and
\(\eta>1\), this gives
\(
r_\eta(a)
=
x_0^\eta\int_a^\infty t^{-\eta}\,dt
=
\frac{x_0^\eta}{\eta-1}a^{1-\eta}.
\)
Hence, the plug-in DP recursion becomes
\begin{align}\label{eq:DP_pareto}
    \widehat V_N
    =
    \mathbb E_{\theta^{(\mathrm U)}}[X_1]
    =
    \frac{\theta^{(\mathrm U)}x_0}{\theta^{(\mathrm U)}-1}, \quad
    \widehat V_i
    =
    \widehat V_{i+1}
    +
    \frac{x_0^{\theta^{(\mathrm U)}}}
         {\theta^{(\mathrm U)}-1}
    \widehat V_{i+1}^{\,1-\theta^{(\mathrm U)}},
    \qquad i=1,\dots,N-1.
\end{align}
The following theorem presents the competitive ratio achieved by this algorithm;
the proof is in Appendix~\ref{app:thm_pareto_cr_lower}.

\begin{theorem}\label{thm:pareto_cr_lower}
    Suppose \(X_1,\dots,X_n \stackrel{\mathrm{i.i.d.}}{\sim} \mathrm{Pareto}(\theta)\) for $\theta>1$. For any \(\delta\in(0,1)\), let the exploration length satisfy  $
m=o(n)$ and $m=\omega(\log(1/\delta)\log^2n)$.
Then the stopping policy \(\tau\) of Algorithm~\ref{alg:general-phi-dp} satisfies
    \[\mathrm{CR}_n(\tau; \theta)\ge
(1-\delta)\rho(1,\theta)
\left(1-O\left(\frac{m}{n}+\frac{\log(1/\delta)\log^2 n}{m}\right)\right).
    \]
\end{theorem}
In particular, Theorem~\ref{thm:pareto_cr_lower} implies that Algorithm~\ref{alg:general-phi-dp} is asymptotically optimal in the Pareto case, achieving the optimal competitive ratio of $\rho(1,\theta)=\frac{\left(\frac{\theta}{\theta-1}\right)^{1/\theta}}{\Gamma(1-1/\theta)}$ in Theorem~\ref{thm:general-phi}. Optimizing the exploration length \(m\) then yields the following convergence rate.

\begin{corollary}\label{cor:pareto_cr_lower}
Suppose \(X_1,\dots,X_n \stackrel{\mathrm{i.i.d.}}{\sim} \mathrm{Pareto}(\theta)\) for $\theta>1$. For any \(\delta\in(0,1)\), let $
m=
\sqrt{n\log(1/\delta)}\,\log n.$
Then the stopping policy \(\tau\) of Algorithm~\ref{alg:general-phi-dp} satisfies 
\[
\mathrm{CR}_n(\tau; \theta)
\ge (1-\delta)
\rho(1,\theta)
\left(
1-O\left(
\frac{\sqrt{\log(1/\delta)}\log n}{\sqrt n}\right)
\right).
\]
\end{corollary}
In contrast, the relative rank-based framework of \citet{goldenshluger2022optimal}
does not recover the full-information Pareto-specific limit in the
Fr\'echet heavy-tailed regime: indeed, no relative rank-based rule can attain
\(\rho(1,\theta)\) for Pareto rewards, as shown in
Appendix~\ref{app:limit_rank}. Our algorithm attains this optimal asymptotic
competitive ratio, and this separation is also reflected in the experiments~below.

Unlike the exponential case, the full-information Pareto DP ratio has no
\(1/\log n\) bottleneck; it approaches \(\rho(1,\theta)\) with a smaller
finite-horizon correction of order \(\log n/n\). Thus, the rate in
Corollary~\ref{cor:pareto_cr_lower} is governed by the learning and
exploration errors. The corresponding full-information rate is provided in
Appendix~\ref{app:pareto_full_info_rate}.


\subsection{Bounded-Support Power Family}

We next investigate the bounded-support power family $
\phi(x)=\log\!\left(\frac{x_F-x_0}{x_F-x}\right) $ for $x\in[x_0,x_F),$
for which $
F_\theta(x)
=
1-\left(\frac{x_F-x}{x_F-x_0}\right)^\theta$ for $ x\in[x_0,x_F]$ for $\theta>0$.
 This includes the uniform distribution on \([x_0,x_F]\) as the special case
\(\theta=1\).

For the estimated value update in Line~\ref{line:DP} of Algorithm~\ref{alg:general-phi-dp},
note that for any \(a\in[x_0,x_F]\),
\(
r_\eta(a)
=
\int_a^{x_F} e^{-\eta\phi(t)}\,dt
=
\int_a^{x_F}\left(\frac{x_F-t}{x_F-x_0}\right)^\eta dt
=
\frac{(x_F-a)^{\eta+1}}{(\eta+1)(x_F-x_0)^\eta}.
\)
Hence the plug-in DP recursion becomes
\begin{align}\label{eq:DP_uniform_family}
    \widehat V_N
    &=
    \mathbb E_{\theta^{(\mathrm U)}}[X_1]
    =
    x_0+\frac{x_F-x_0}{\theta^{(\mathrm U)}+1}
    =
    x_F-\frac{\theta^{(\mathrm U)}}{\theta^{(\mathrm U)}+1}(x_F-x_0),\\
    \widehat V_i
    &=
    \widehat V_{i+1}
    +
    \frac{(x_F-\widehat V_{i+1})^{\theta^{(\mathrm U)}+1}}
         {(\theta^{(\mathrm U)}+1)(x_F-x_0)^{\theta^{(\mathrm U)}}},
    \qquad i=1,\dots,N-1.
\end{align}

The following theorem presents the competitive ratio achieved by this algorithm;
the proof is in Appendix~\ref{app:uniform_family_cr}.

\begin{theorem}\label{thm:uniform_family_cr}
Suppose $
X_1,\dots,X_n \stackrel{\mathrm{i.i.d.}}{\sim} F_\theta$ where $
F_\theta(x)=1-\left(\frac{x_F-x}{x_F-x_0}\right)^\theta$ for $x\in[x_0,x_F]$ and $\theta>0$.
For any \(\delta\in(0,1)\), let the exploration length satisfy $m=o(n)$ and $m=\omega (\log(1/\delta)\log^2n)$. Then the stopping policy $\tau$ of Algorithm~\ref{alg:general-phi-dp} satisfies 
\[\mathrm{CR}_n(\tau;\theta)\ge (1-\delta)\left(1-O\left(n^{-1/\theta}\right)\right). \]
\end{theorem}

In particular, for the uniform distribution, corresponding to the special
case \(\theta=1\), Theorem~\ref{thm:uniform_family_cr} with \(\delta=1/n\)
yields an \(O(1/n)\) convergence rate toward the optimal competitive ratio
\(1\). This improves over the \(O(\log n/n)\) rate of the relative rank-based approach
of \cite{goldenshluger2022optimal}.

More generally, the rate \(n^{-1/\theta}\) is already the intrinsic
finite-horizon rate of the optimal full-information DP policy. The following
proposition shows that even when the parameter \(\theta\) is known, the
competitive ratio approaches its limiting value \(1\) at order
\(n^{-1/\theta}\). Thus, up to the multiplicative factor \(1-\delta\), the
rate in Theorem~\ref{thm:uniform_family_cr} matches the full-information
finite-horizon rate. 
\begin{proposition}[Full-information convergence rate for bounded-support power rewards]
\label{lem:bounded_power_full_info_rate}
Suppose $
X_1,\dots,X_n \stackrel{\mathrm{i.i.d.}}{\sim} F_\theta$ where $
F_\theta(x)=1-\left(\frac{x_F-x}{x_F-x_0}\right)^\theta$ for $x\in[x_0,x_F]$ and $\theta>0$. Let \(\tau^*\) be the
optimal full-information stopping rule. Then
\[
\mathrm{CR}_n(\tau^*;\theta)
=
1-\Theta\left(n^{-1/\theta}\right).
\]
\end{proposition}

\section{Experiments}
\begin{figure}[h]
    \centering
    \includegraphics[width=1\linewidth]{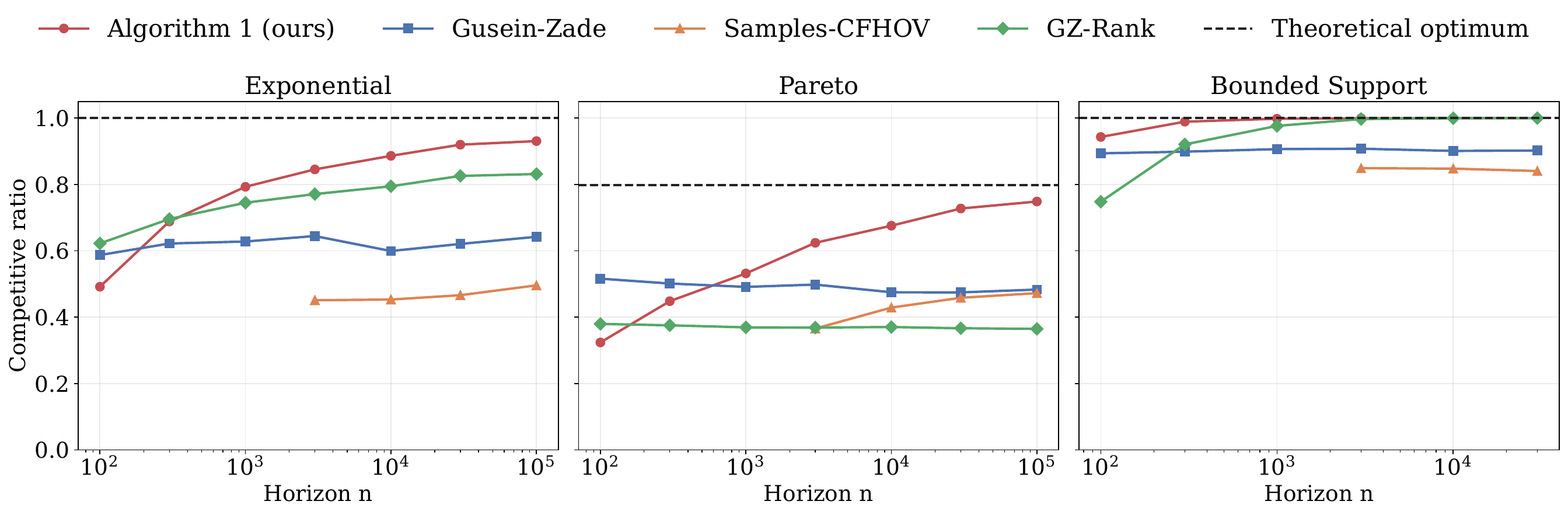}
 \caption{Competitive-ratio curves for exponential, Pareto, and bounded-support rewards.}
    \label{fig:experiments}
\end{figure}
We evaluate Algorithm~\ref{alg:general-phi-dp} on exponential, Pareto, and
bounded-support rewards, and compare it with three baselines: the secretary-type
\texttt{Gusein-Zade} rule \citep{gusein1966problem,correa2019prophet}, the online \texttt{Samples-CFHOV} variant (Algorithm~\ref{alg:baseline}), and the relative rank-based \texttt{GZ-RANK}
rule of \cite{goldenshluger2022optimal}. Figure~\ref{fig:experiments} shows
that our algorithm consistently approaches the corresponding optimal asymptotic
ratio and performs favorably against the baselines, especially in the Pareto case
where rank-based methods do not guarantee asymptotic optimality.  The
\texttt{Samples-CFHOV} curve starts later because its required exploration
length exceeds the horizon for small \(n\). The detailed setup is described in Appendix~\ref{app:experiment-details}.



\section{Conclusion}
We studied prophet inequalities for an exponential-type parametric family with
an unknown parameter. We characterized the full-information
distribution-specific asymptotic benchmark and proposed a confidence-based
DP learning policy that attains it using only online observations. Our
example-specific analyses and experiments show how parametric structure enables
distribution-specific optimal performance beyond distribution-agnostic
worst-case guarantees, including in heavy-tailed settings.

\textbf{Limitations.}
Our guarantees rely on an exponential-type parametric family with a one-dimensional unknown parameter. An interesting direction for future work is to extend this approach
to richer parametric, contextual, or misspecified models.

\bibliographystyle{plainnat}

\bibliography{mybib}

@article{goldenshluger2022optimal,
  title={Optimal stopping of a random sequence with unknown distribution},
  author={Goldenshluger, Alexander and Zeevi, Assaf},
  journal={Mathematics of Operations Research},
  volume={47},
  number={1},
  pages={29--49},
  year={2022},
  publisher={INFORMS}
}

@article{correa2019recent,
  title={Recent developments in prophet inequalities},
  author={Correa, Jose and Foncea, Patricio and Hoeksma, Ruben and Oosterwijk, Tim and Vredeveld, Tjark},
  journal={ACM SIGecom Exchanges},
  volume={17},
  number={1},
  pages={61--70},
  year={2019},
  publisher={ACM New York, NY, USA}
}

@article{krengel1978semiamarts,
  title={On semiamarts, amarts, and processes with finite value},
  author={Krengel, Ulrich and Sucheston, Louis},
  journal={Probability on Banach spaces},
  volume={4},
  number={197-266},
  pages={1--2},
  year={1978},
  publisher={Dekker New York}
}

@article{hill1982comparisons,
  title={Comparisons of stop rule and supremum expectations of iid random variables},
  author={Hill, Theodore P and Kertz, Robert P},
  journal={The Annals of Probability},
  pages={336--345},
  year={1982},
  publisher={JSTOR}
}

@article{samuel1984comparison,
  title={Comparison of threshold stop rules and maximum for independent nonnegative random variables},
  author={Samuel-Cahn, Ester},
  journal={The Annals of Probability},
  pages={1213--1216},
  year={1984},
  publisher={JSTOR}
}

@inproceedings{correa2017posted,
  title={Posted price mechanisms for a random stream of customers},
  author={Correa, Jos{\'e} and Foncea, Patricio and Hoeksma, Ruben and Oosterwijk, Tim and Vredeveld, Tjark},
  booktitle={Proceedings of the 2017 ACM Conference on Economics and Computation},
  pages={169--186},
  year={2017}
}

@inproceedings{correa2019prophet,
  title={Prophet inequalities for iid random variables from an unknown distribution},
  author={Correa, Jos{\'e} and D{\"u}tting, Paul and Fischer, Felix and Schewior, Kevin},
  booktitle={Proceedings of the 2019 ACM Conference on Economics and Computation},
  pages={3--17},
  year={2019}
}

@article{hill1992survey,
  title={A survey of prophet inequalities in optimal stopping theory},
  author={Hill, Theodore P and Kertz, Robert P},
  journal={Contemporary Mathematics},
  volume={125},
  number={1},
  pages={191},
  year={1992}
}

@article{lucier2017economic,
  title={An economic view of prophet inequalities},
  author={Lucier, Brendan},
  journal={ACM SIGecom Exchanges},
  volume={16},
  number={1},
  pages={24--47},
  year={2017},
  publisher={ACM New York, NY, USA}
}

@inproceedings{alaei2012online,
  title={Online prophet-inequality matching with applications to ad allocation},
  author={Alaei, Saeed and Hajiaghayi, MohammadTaghi and Liaghat, Vahid},
  booktitle={Proceedings of the 13th ACM Conference on Electronic Commerce},
  pages={18--35},
  year={2012}
}

@article{arsenis2022individual,
  title={Individual fairness in prophet inequalities},
  author={Arsenis, Makis and Kleinberg, Robert},
  journal={arXiv preprint arXiv:2205.10302},
  year={2022}
}

@article{gusein1966problem,
  title={The problem of choice and the optimal stopping rule for a sequence of independent trials},
  author={Gusein-Zade, SM},
  journal={Theory of Probability \& Its Applications},
  volume={11},
  number={3},
  pages={472--476},
  year={1966},
  publisher={SIAM}
}

@article{rubinstein2019optimal,
  title={Optimal Single-Choice Prophet Inequalities from Samples},
  author={Rubinstein, Aviad and Wang, Jack Z and Weinberg, S Matthew},
  journal={Innovations in Theoretical Computer Science},
  year={2020}
}

@article{martinsek1984approximations,
  title={Approximations to optimal stopping rules for exponential random variables},
  author={Martinsek, Adam T},
  journal={The Annals of Probability},
  volume={12},
  number={3},
  pages={876--881},
  year={1984},
  publisher={Institute of Mathematical Statistics}
}

@inproceedings{kimlearning,
  title={Learning in Prophet Inequalities with Noisy Observations},
  author={Kim, Jung-hun and Perchet, Vianney},  year={2026},
  booktitle={The Fourteenth International Conference on Learning Representations}
}

@inproceedings{livanos2025minimization,
  title={Minimization iid prophet inequality via extreme value theory: A unified approach},
  author={Livanos, Vasilis and Mehta, Ruta},
  booktitle={Proceedings of the 26th ACM Conference on Economics and Computation},
  pages={1157--1179},
  year={2025}
}

@article{kennedy1991asymptotic,
  title={The asymptotic behavior of the reward sequence in the optimal stopping of iid random variables},
  author={Kennedy, Douglas P and Kertz, Robert P},
  journal={The Annals of Probability},
  volume={19},
  number={1},
  pages={329--341},
  year={1991},
  publisher={Institute of Mathematical Statistics}
}

@article{ferguson1989solved,
  title={Who solved the secretary problem?},
  author={Ferguson, Thomas S},
  journal={Statistical science},
  volume={4},
  number={3},
  pages={282--289},
  year={1989},
  publisher={Institute of Mathematical Statistics}
}

\appendix


\newpage

\appendix

\section{Appendix}
\subsection{Proof of Lemma~\ref{lem:confi}}\label{app:confidence}
    Let $S=\sum_{i=1}^m Y_i$ then $S$ follows a Gamma distribution with mean $m/\theta$. Using multiplicative Chernoff bound, for $\varepsilon\in(0,1)$, we have
$\Pr\left(\left|\frac{S}{m/\theta}-1\right|\ge \varepsilon\right)\le 2\exp(-m\varepsilon^2/4).$
Therefore, if \(m\ge \frac{4}{\varepsilon^2}\log(2/\delta)\), then the above probability is at most \(\delta\), which proves the lemma.

\subsection{Equality of the asymptotic constants}\label{app:c_equal}

We prove the two endpoint regimes separately.

\medskip
\noindent\textbf{Unbounded-support case.}
Define
\[
f(y):=\frac{\psi'(y)}{\psi(y)},
\qquad
g(y):=\frac{\psi''(y)}{\psi'(y)}.
\]
Then \(f(y)>0\) eventually, \(f(y)\to c\), and \(g(y)\to c'\). Differentiating \(f\), we obtain
\[
f'(y)
=
\frac{\psi''(y)\psi(y)-(\psi'(y))^2}{\psi(y)^2}
=
\frac{\psi'(y)}{\psi(y)}
\left(
\frac{\psi''(y)}{\psi'(y)}
-
\frac{\psi'(y)}{\psi(y)}
\right)
=
f(y)\bigl(g(y)-f(y)\bigr).
\]

If \(c>0\), then
\[
f'(y)\to c(c'-c).
\]
Since \(f(y)\) converges to the finite limit \(c\), its derivative cannot converge to a nonzero constant; otherwise \(f\) would eventually grow or decrease linearly. Hence \(c(c'-c)=0\), and since \(c>0\), we get \(c'=c\).

It remains to consider \(c=0\). Suppose, toward a contradiction, that \(c'>0\). Then for all sufficiently large \(y\),
\[
0<f(y)\le \frac{c'}{4},
\qquad
g(y)\ge \frac{c'}{2}.
\]
Therefore,
\[
f'(y)
=
f(y)\bigl(g(y)-f(y)\bigr)
\ge
\frac{c'}{4}f(y).
\]
Equivalently, \((\log f(y))'\ge c'/4\) for all sufficiently large \(y\). Hence \(f(y)\) grows at least exponentially along the tail, contradicting \(f(y)\to0\). Thus \(c'=0\), and therefore \(c=c'\) also when \(c=0\).

\medskip
\noindent\textbf{Bounded-support case.}
Let
\[
w(y):=\bar\psi(y)=x_F-\psi(y).
\]
Then \(w(y)>0\), \(w(y)\to0\), and \(w'(y)<0\) eventually. Define
\[
p(y):=-\frac{w'(y)}{w(y)},
\qquad
q(y):=-\frac{w''(y)}{w'(y)}.
\]
The assumptions imply \(p(y)\to c\) and \(q(y)\to c'\). Moreover, \(p(y)>0\) eventually. Differentiating \(p\), we get
\[
p'(y)
=
-\frac{w''(y)}{w(y)}
+
\left(\frac{w'(y)}{w(y)}\right)^2.
\]
Since
\[
\frac{w''(y)}{w(y)}
=
\frac{w''(y)}{w'(y)}
\frac{w'(y)}{w(y)}
=
\bigl(-q(y)\bigr)\bigl(-p(y)\bigr)
=
p(y)q(y),
\]
we obtain
\[
p'(y)=p(y)\bigl(p(y)-q(y)\bigr).
\]

If \(c>0\), then
\[
p'(y)\to c(c-c').
\]
Since \(p(y)\) converges to the finite limit \(c\), the derivative \(p'(y)\) cannot converge to a nonzero constant. Hence \(c(c-c')=0\), and because \(c>0\), we get \(c'=c\).

Now suppose \(c=0\). We claim that \(c'=0\). Assume, toward a contradiction, that \(c'>0\). Then for all sufficiently large \(y\),
\[
0<p(y)\le \frac{c'}{4},
\qquad
q(y)\ge \frac{c'}{2}.
\]
Thus
\[
p'(y)
=
p(y)\bigl(p(y)-q(y)\bigr)
\le
-\frac{c'}{4}p(y).
\]
Consequently, for some \(y_0\) and all \(y\ge y_0\),
\[
p(y)\le p(y_0)\exp\left(-\frac{c'}{4}(y-y_0)\right).
\]
In particular,
\[
\int_{y_0}^{\infty} p(s)\,ds < \infty.
\]
But \(p(y)=-w'(y)/w(y)\), so
\[
\log w(y)
=
\log w(y_0)-\int_{y_0}^{y} p(s)\,ds.
\]
The finiteness of the integral implies that \(\log w(y)\) remains bounded below as \(y\to\infty\), and hence \(w(y)\) cannot converge to \(0\). This contradicts \(w(y)=\bar\psi(y)\to0\). Therefore \(c'=0\).

Thus \(c=c'\) in the bounded-support case as well.







\subsection{Why we assume $c,c'\in[0,\theta)$}\label{app:c_just}

\begin{lemma}\label{lem:c-greater-theta-infinite-mean}
Let
\[
Y\sim \mathrm{Exp}(\theta),\qquad X=\psi(Y),
\]
where $\theta>0$ and $\psi:[0,\infty)\to[x_0,\infty)$ is strictly increasing. Define
\[
c:=\lim_{y\to\infty}\frac{\psi'(y)}{\psi(y)},
\qquad
c':=\lim_{y\to\infty}\frac{\psi''(y)}{\psi'(y)},
\]
whenever the limits exist. If either $c>\theta$ or $c'>\theta$, then
\[
\mathbb E[X]=\infty.
\]
\end{lemma}

\begin{proof}
Since $Y\sim \mathrm{Exp}(\theta)$ and $X=\psi(Y)$, we have
\[
\mathbb E[X]
=
\mathbb E[\psi(Y)]
=
\theta\int_0^\infty e^{-\theta y}\psi(y)\,dy.
\]
Thus it suffices to show that the integral on the right-hand side diverges.

\medskip
\noindent\textbf{Case 1: $c>\theta$.}
Choose $\varepsilon>0$ such that
\[
c-\varepsilon>\theta.
\]
Since
\[
\frac{\psi'(y)}{\psi(y)}\to c,
\]
there exists $y_0$ such that
\[
\frac{\psi'(y)}{\psi(y)}\ge c-\varepsilon,
\qquad y\ge y_0.
\]
Using
\[
\frac{\psi'(y)}{\psi(y)}=\frac{d}{dy}\log \psi(y),
\]
we integrate from $y_0$ to $y\ge y_0$ and obtain
\[
\log \psi(y)-\log \psi(y_0)
=
\int_{y_0}^y \frac{\psi'(u)}{\psi(u)}\,du
\ge
(c-\varepsilon)(y-y_0).
\]
Hence
\[
\psi(y)\ge \psi(y_0)e^{-(c-\varepsilon)y_0}e^{(c-\varepsilon)y}
=:C_1 e^{(c-\varepsilon)y},
\qquad y\ge y_0,
\]
for some constant $C_1>0$. Therefore
\[
\mathbb E[X]
\ge
\theta C_1\int_{y_0}^\infty e^{-\theta y}e^{(c-\varepsilon)y}\,dy
=
\theta C_1\int_{y_0}^\infty e^{(c-\varepsilon-\theta)y}\,dy.
\]
Since $c-\varepsilon-\theta>0$, the last integral diverges. Thus
\[
\mathbb E[X]=\infty.
\]

\medskip
\noindent\textbf{Case 2: $c'>\theta$.}
Choose $\varepsilon>0$ such that
\[
c'-\varepsilon>\theta.
\]
Since
\[
\frac{\psi''(y)}{\psi'(y)}\to c',
\]
there exists $y_1$ such that
\[
\frac{\psi''(y)}{\psi'(y)}\ge c'-\varepsilon,
\qquad y\ge y_1.
\]
Using
\[
\frac{\psi''(y)}{\psi'(y)}=\frac{d}{dy}\log \psi'(y),
\]
we integrate from $y_1$ to $y\ge y_1$ and obtain
\[
\log \psi'(y)-\log \psi'(y_1)
=
\int_{y_1}^y \frac{\psi''(u)}{\psi'(u)}\,du
\ge
(c'-\varepsilon)(y-y_1).
\]
Hence
\[
\psi'(y)\ge \psi'(y_1)e^{-(c'-\varepsilon)y_1}e^{(c'-\varepsilon)y}
=:C_2 e^{(c'-\varepsilon)y},
\qquad y\ge y_1,
\]
for some constant $C_2>0$. Integrating once more,
\begin{align*}
\psi(y)-\psi(y_1)
&=
\int_{y_1}^y \psi'(u)\,du \\
&\ge
C_2\int_{y_1}^y e^{(c'-\varepsilon)u}\,du \\
&=
\frac{C_2}{c'-\varepsilon}\Bigl(e^{(c'-\varepsilon)y}-e^{(c'-\varepsilon)y_1}\Bigr).
\end{align*}
Therefore, for all sufficiently large $y$,
\[
\psi(y)\ge C_3 e^{(c'-\varepsilon)y}
\]
for some constant $C_3>0$. Consequently,
\[
\mathbb E[X]
\ge
\theta C_3\int_{y_2}^\infty e^{-\theta y}e^{(c'-\varepsilon)y}\,dy
=
\theta C_3\int_{y_2}^\infty e^{(c'-\varepsilon-\theta)y}\,dy
\]
for some $y_2$ large enough. Since $c'-\varepsilon-\theta>0$, this integral diverges, and hence
\[
\mathbb E[X]=\infty.
\]

This completes the proof.
\end{proof}
\begin{remark}
When $c=\theta$, the mean $\mathbb E[X]$ may be infinite. For example, if
 $\psi(y)=e^{\theta y}$, then
\[
\mathbb E[X]
=
\theta\int_0^\infty e^{-\theta y}\psi(y)\,dy
=
\theta\int_0^\infty 1\,dy
=
\infty.
\]
\end{remark}

Hence, in our asymptotic analysis we restrict attention to the regime $c,c'\in[0,\theta)$.
\subsection{Extreme-value interpretation of the endpoint regularity condition}
\label{app:evt_connection}

In this appendix, we justify the connection between Assumption~\ref{ass:c_conv}
and the classical extreme-value domains of attraction.

Recall that
\[
F_\theta(x)=1-\exp(-\theta \phi(x)),
\qquad
x\in[x_0,x_F),
\]
with \(\psi:=\phi^{-1}\), and define the upper quantile function
\[
U(t):=F_\theta^{-1}(1-1/t), \qquad t>1.
\]
Since
\[
1-\frac1t = 1-\exp(-\theta \phi(U(t))),
\]
we obtain
\[
\exp(-\theta \phi(U(t)))=\frac1t,
\qquad\text{hence}\qquad
\phi(U(t))=\frac{\log t}{\theta}.
\]
Therefore
\[
U(t)=\psi\!\left(\frac{\log t}{\theta}\right).
\]

We now analyze the two endpoint regimes separately.

\begin{lemma}[Quantile growth under endpoint regularity]
\label{lem:evt_connection}
Let \(U(t)=\psi((\log t)/\theta)\).

\begin{enumerate}
    \item[(i)] Suppose \(x_F=\infty\) and
    \[
    \frac{\psi'(y)}{\psi(y)} \to c_+\ge 0
    \qquad\text{as } y\to\infty.
    \]
    Then for every \(x>0\),
    \[
    \frac{U(tx)}{U(t)} \to x^{c_+/\theta}
    \qquad\text{as } t\to\infty.
    \]
    In particular:
    \begin{itemize}
        \item if \(c_+>0\), then \(U\) is regularly varying at infinity with index \(c_+/\theta\);
        \item if \(c_+=0\), then \(U\) is slowly varying at infinity.
    \end{itemize}

    \item[(ii)] Suppose \(x_F<\infty\), define
    \[
    \bar\psi(y):=x_F-\psi(y),
    \]
    and assume
    \[
    \frac{\bar\psi'(y)}{\bar\psi(y)} \to -c_-
    \qquad\text{as } y\to\infty
    \]
    for some \(c_-\ge 0\). Then for every \(x>0\),
    \[
    \frac{x_F-U(tx)}{x_F-U(t)} \to x^{-c_-/\theta}
    \qquad\text{as } t\to\infty.
    \]
    In particular, if \(c_->0\), then \(x_F-U(t)\) is regularly varying at infinity with
    index \(-c_-/\theta\).
\end{enumerate}
\end{lemma}

\begin{proof}
We start with the unbounded-endpoint case.
Fix \(x>0\). Since
\[
U(t)=\psi\!\left(\frac{\log t}{\theta}\right),
\]
we have
\[
\log\frac{U(tx)}{U(t)}
=
\log\frac{
\psi\!\left(\frac{\log t+\log x}{\theta}\right)
}{
\psi\!\left(\frac{\log t}{\theta}\right)
}.
\]
By the fundamental theorem of calculus,
\[
\log\frac{U(tx)}{U(t)}
=
\int_{\log t/\theta}^{(\log t+\log x)/\theta}
\frac{\psi'(u)}{\psi(u)}\,du.
\]
Let
\[
a_t:=\frac{\log t}{\theta},
\qquad
b_t:=\frac{\log x}{\theta}.
\]
Then \(a_t\to\infty\) as \(t\to\infty\), while \(b_t\) is constant. Hence
\[
\int_{a_t}^{a_t+b_t}\frac{\psi'(u)}{\psi(u)}\,du
-
c_+ b_t
=
\int_{a_t}^{a_t+b_t}\left(\frac{\psi'(u)}{\psi(u)}-c_+\right)\,du.
\]
Given \(\varepsilon>0\), by the assumption \(\psi'(u)/\psi(u)\to c_+\), there exists \(M\)
such that
\[
\left|\frac{\psi'(u)}{\psi(u)}-c_+\right|\le \varepsilon
\qquad\text{for all }u\ge M.
\]
For \(t\) large enough, the whole interval \([a_t,a_t+b_t]\) lies beyond \(M\) if \(b_t\ge0\),
and similarly the same estimate applies when \(b_t<0\). Therefore
\[
\left|
\int_{a_t}^{a_t+b_t}\left(\frac{\psi'(u)}{\psi(u)}-c_+\right)\,du
\right|
\le \varepsilon |b_t|.
\]
Since \(\varepsilon\) is arbitrary,
\[
\log\frac{U(tx)}{U(t)} \to c_+ b_t = \frac{c_+}{\theta}\log x.
\]
Exponentiating yields
\[
\frac{U(tx)}{U(t)} \to x^{c_+/\theta}.
\]
This proves part (i).

For the bounded-support case, note that
\[
x_F-U(t)
=
x_F-\psi\!\left(\frac{\log t}{\theta}\right)
=
\bar\psi\!\left(\frac{\log t}{\theta}\right).
\]
Hence
\[
\log\frac{x_F-U(tx)}{x_F-U(t)}
=
\log\frac{
\bar\psi\!\left(\frac{\log t+\log x}{\theta}\right)
}{
\bar\psi\!\left(\frac{\log t}{\theta}\right)
}
=
\int_{\log t/\theta}^{(\log t+\log x)/\theta}
\frac{\bar\psi'(u)}{\bar\psi(u)}\,du.
\]
Using exactly the same argument as above together with
\[
\frac{\bar\psi'(u)}{\bar\psi(u)}\to -c_-,
\]
we obtain
\[
\log\frac{x_F-U(tx)}{x_F-U(t)}
\to -\frac{c_-}{\theta}\log x.
\]
Exponentiating gives
\[
\frac{x_F-U(tx)}{x_F-U(t)} \to x^{-c_-/\theta}.
\]
This proves part (ii).
\end{proof}

\begin{corollary}[Connection with extreme-value domains]
\label{cor:evt_connection}
Under the assumptions of Lemma~\ref{lem:evt_connection}:

\begin{enumerate}
    \item[(i)] If \(x_F=\infty\) and \(c_+>0\), then \(F_\theta\) belongs to the
    Fr\'echet domain of attraction, with extreme-value index
    \[
    \gamma=\frac{c_+}{\theta}.
    \]

    \item[(ii)] If \(x_F<\infty\) and \(c_->0\), then \(F_\theta\) belongs to the
    reverse-Weibull domain of attraction, with extreme-value index
    \[
    \gamma=-\frac{c_-}{\theta}.
    \]
\end{enumerate}

If \(c_+=0\), then \(U\) is slowly varying, which is consistent with Gumbel-type behavior.
\end{corollary}

\begin{proof}
Part (i) follows from the standard quantile characterization of the Fr\'echet domain:
for distributions with unbounded upper endpoint, \(F\) belongs to the Fr\'echet domain of
attraction with index \(\gamma>0\) if and only if its upper quantile function \(U\) is regularly
varying with index \(\gamma\).

Part (ii) follows from the corresponding characterization of the reverse-Weibull domain:
for distributions with finite upper endpoint \(x_F\), \(F\) belongs to the reverse-Weibull
domain of attraction with index \(\gamma<0\) if and only if \(x_F-U(t)\) is regularly varying
with index \(\gamma\).

The final statement is immediate from Lemma~\ref{lem:evt_connection}(i).
\end{proof}
\subsection{Baseline Algorithm}\label{app:baseline}

We provide a pseudocode in Algorithm~\ref{alg:baseline}.  The following proposition gives its competitive
ratio guarantee.

\begin{proposition}
\label{prop:baseline}
Let
\[
m=
\frac{4\log^2(n/\eta^2)}{\log^2(1+\eta)}
\log\!\left(\frac{2Cn\log(n/\eta^2)}{\eta^2}\right)
\]
for some \(\eta>0\) and a sufficiently large constant \(C>0\). Then the
stopping time \(\tau\) returned by Algorithm~\ref{alg:baseline}, run with
exploration length \(m\), satisfies
\[
\liminf_{n\to\infty}\mathrm{CR}_n(\tau;\theta)
\ge
\kappa\frac{1-\eta}{(1+\eta)^3},
\]
where \(\kappa\approx0.745\).
\end{proposition}

This baseline is useful as a worst-case reference: as \(\eta\downarrow0\), its
guarantee approaches the optimal full-information i.i.d.\ worst-case constant
\(\kappa\) \citep{correa2017posted}.

The proof of Proposition~\ref{prop:baseline} is provided in the following.

\begin{algorithm}[t]
\caption{\texttt{ETC} with \texttt{Samples-CFHOV} \citep{rubinstein2019optimal}}
\label{alg:baseline}
\begin{algorithmic}[1]

\State Set
\[
m=\frac{4\log^2(n/\eta^2)}{\log^2(1+\eta)}
\log\!\left(\frac{2Cn\log(n/\eta^2)}{\eta^2}\right).
\]

\State Following the policy with quantiles \(q_1,q_2,\dots,q_n\) of \cite{correa2017posted}, define monotone increasing probabilities
\[
0 \le p_{m+1}(=q_1)\le \cdots \le p_n(=q_{n-m})\le 1,
\]
with \(p_i=0\) for all \(i\in[m]\).

\State For all \(i\) such that \(p_i \le k=\eta^2/n\), update \(p_i \gets 0\).

\State For each \(i>m\), round down \(p_i\) to the nearest integer power of \((1+\eta)\), and denote the rounded value by
\[
\lfloor p_i \rfloor \in \{(1+\eta)^{-1},(1+\eta)^{-2},\dots\}.
\]

\State For each \(i>m\), set
\[
\tilde p_i \gets \frac{\lfloor p_i \rfloor}{1+\eta}.
\]

\State Compute the estimator
\[
\widehat{\theta}_m=\frac{m}{\sum_{s=1}^m \phi(X_s)}.
\]

\State For each \(i>m\), define the threshold \(\tau_i\) such that
\[
1-F_{\widehat{\theta}_m}(\tau_i)=\tilde p_i.
\]

\State For stage \(i\le m\), always continue to the next stage. For stage \(i>m\), accept \(X_i\) (and stop) if and only if
\[
X_i>\tau_i.
\]

\end{algorithmic}
\end{algorithm}

\begin{lemma}
    Let $m=\frac{4\log^2(n/\eta^2)}{\log^2(1+\eta)}\log(2Cn\log(n/\eta^2)/\eta^2)$ for large enough constant $C>0$ and $\eta\in(0,1)$.
    With probability at least $1-\eta$, we have that for every $i\in[m+1,n]$,
\[\frac{p_i}{(1+\eta)^3}\le \Pr_{x\sim \mathcal{D}}(x>\tau_i)\le p_i.\]
\end{lemma}
\begin{proof}

Let $R:=\theta/\hat{\theta}_m$. From $\exp(-\hat{\theta}_m\phi(\tau_i))=\tilde{p}_i$, we have $\Pr(X>\tau_i)=\exp(-\theta\phi(\tau_i))=\exp( -\frac{\log (1/\tilde{p}_i)\theta}{\hat{\theta}})=(1/\tilde{p}_i)^{-\theta/\hat{\theta}}=(1/\tilde{p}_i)^{-R}$. From the range of $\eta$ and $\tilde{p}_i$, we have $\log_{1/\tilde{p}_i}(1+\eta) \le 1$ and $1/\tilde{p}_i \le n/\eta^2$.  Applying Lemma~\ref{lem:confi}, with $\varepsilon=\log_{1/\tilde{p}_i}(1+\eta)$ and $\frac{4\log ^2(1/\tilde{p}_i)}{(\log(1+\eta))^2}\log(2/\delta)\le \frac{4\log^2 (n/\eta^2)}{\log^2(1+\eta)}\log(2/\delta)=m$ for $i>m$, with probability $1-\delta$, we have
\[
\tilde{p}_i\frac{1}{1+\eta}\le \Pr(X>\tau_i)\le\tilde{p}_i(1+\eta).
\]

From $\lfloor p_i \rfloor \le p_i \le (1+\eta)\lfloor p_i \rfloor$ and the definition of $\tilde{p}_i$ we have $p_i / (1 + \eta)^2 \le \tilde{p}_i \le p_i / (1+ \eta)$ and therefore  
with probability $1-\delta$, we have
\[
{p}_i\frac{1}{(1+\eta)^3}\le \Pr(X>\tau_i)\le{p}_i.
\]

By applying a union bound for all possible $\lfloor p_i \rfloor\in\{(1+\eta)^{-1},(1+\eta)^{-2},\dots,\eta^2/n\}$, the cardinality of which is $\Theta(\log(n/\eta^2)/\eta)$, and setting $\delta=\xi \eta/(Cn\log(n/\eta^2))$ for some constant $C>0$, we have for any $i\in[n]$, with probability $1-\xi$,
\[
{p}_i\frac{1}{(1+\eta)^3}\le \Pr(X>\tau_i)\le{p}_i.
\]

We can complete the proof with $\xi = \eta$.
\end{proof}

Let $\tau'$ be the stopping algorithm of \cite{correa2017posted} with $n-m$ horizon. By following the proof steps in \cite{rubinstein2019optimal}, for $\varepsilon>0$, we have
\begin{align}\label{eq:baseline_lower_X_tau}
    \mathbb{E}[X_\tau]\ge \mathbb{E}[X_{\tau'}]\frac{1-\varepsilon}{(1+\varepsilon)^3}\ge \kappa\frac{1-\varepsilon}{(1+\varepsilon)^3}\mathbb{E}\left[\max_{i\in[m+1,n]}X_i\right]\ge \kappa\frac{1-\varepsilon}{(1+\varepsilon)^3}\frac{n-m}{n}\mathbb{E}\left[\max_{i\in[n]}X_i\right],
\end{align}
where the last inequality is obtained from the following. 
Let $k:=n-m$, and let $S\subseteq [n]$ be a uniformly random subset of size $k$,
independent of $(X_1,\dots,X_n)$. Since $X_1,\dots,X_n$ are i.i.d., we have
\[
\max_{i\in S} X_i \overset{d}= \max_{i\in [m+1,n]} X_i.
\]
Fix a realization of $(X_1,\dots,X_n)$ and let $j^*$ be an index attaining
$\max_{i\in[n]}X_i$ (breaking ties arbitrarily). Since $X_i\ge 0$,
\[
\max_{i\in S} X_i \ge \mathbf{1}\{j^*\in S\}\max_{i\in[n]}X_i.
\]
Taking conditional expectation given $(X_1,\dots,X_n)$ yields
\[
\mathbb{E}\!\left[\max_{i\in S}X_i \mid X_1,\dots,X_n\right]
\ge \mathbb{P}(j^*\in S)\max_{i\in[n]}X_i
= \frac{k}{n}\max_{i\in[n]}X_i.
\]
Taking expectation again, we obtain
\[
\mathbb{E}\!\left[\max_{i\in[m+1,n]}X_i\right]
=
\mathbb{E}\!\left[\max_{i\in S}X_i\right]
\ge \frac{n-m}{n}\mathbb{E}\!\left[\max_{i\in[n]}X_i\right].
\]

Finally from \eqref{eq:baseline_lower_X_tau}, for $\varepsilon>0$,
$\liminf_{n\to \infty}\mathrm{CR}_n(\tau;\theta)\ge \kappa\frac{1-\varepsilon}{(1+\varepsilon)^3}$. We can complete the proof with $\varepsilon=\eta$.

\subsection{Proof of Theorem~\ref{thm:general-phi}}\label{app:general-phi}

For each $i\in[n]$, let $V_i$ denote the optimal expected reward when only periods
$i,i+1,\dots,n$ remain, and set $V_{n+1}:=0$. Then the optimal stopping rule is the threshold rule
\[
\tau^*=\inf\{i\in[n]:X_i\ge V_{i+1}\},
\]
and the value function satisfies
\[
V_i=\mathbb E[\max\{X_i,V_{i+1}\}],\qquad i=1,\dots,n.
\]
Since the last observation must be accepted,
\[
V_n=\mathbb E[X_1].
\]

Define, for $a\ge x_0$,
\[
r_\theta(a):=\mathbb E[(X_1-a)^+]
=\int_a^\infty \mathbb P(X_1>t)\,dt
=\int_a^\infty e^{-\theta\phi(t)}\,dt.
\]
Then the dynamic-programming recursion becomes
\[
V_i=V_{i+1}+r_\theta(V_{i+1}),
\qquad i=1,\dots,n-1.
\tag{A.6.1}\label{eq:dp-general}
\]

We first derive the asymptotic behavior of $r_\theta(a)$ as $a\to\infty$.
Write $a=\psi(y)$. With the change of variables $t=\psi(u)$, so that $dt=\psi'(u)\,du$, we obtain
\[
r_\theta(\psi(y))
=
\int_y^\infty e^{-\theta u}\psi'(u)\,du
=
e^{-\theta y}\psi'(y)\int_0^\infty e^{-\theta s}\frac{\psi'(y+s)}{\psi'(y)}\,ds.
\tag{A.6.2}\label{eq:rtheta-transform}
\]

Since $\psi''(y)/\psi'(y)\to c_+$, for every $\varepsilon>0$ there exists $y_\varepsilon$ such that
\[
\frac{\psi''(u)}{\psi'(u)}\le c_++\varepsilon
\qquad\text{for all }u\ge y_\varepsilon.
\]
Hence, for $y\ge y_\varepsilon$ and $s\ge0$,
\[
\frac{\psi'(y+s)}{\psi'(y)}
=
\exp\!\left(\int_y^{y+s}\frac{\psi''(u)}{\psi'(u)}\,du\right)
\le e^{(c_++\varepsilon)s}.
\]
Also, for each fixed $s\ge0$,
\[
\frac{\psi'(y+s)}{\psi'(y)}\to e^{c_+s}
\qquad (y\to\infty).
\]
Since $c_++\varepsilon<\theta$, dominated convergence in \eqref{eq:rtheta-transform} yields
\[
r_\theta(\psi(y))
\sim
e^{-\theta y}\psi'(y)\int_0^\infty e^{-(\theta-c_+)s}\,ds
=
\frac{\psi'(y)e^{-\theta y}}{\theta-c_+}.
\tag{A.6.3}\label{eq:rtheta-asym}
\]

Next define
\[
\Psi(x):=1+\int_{V_n}^x \frac{dt}{r_\theta(t)},
\qquad x\ge V_n,
\]
and let
\[
W_i:=\Psi(V_i),\qquad i=1,\dots,n.
\]
Then $W_n=1$.

Now define
\[
Q(y):=\Psi(\psi(y)).
\]
By the chain rule and \eqref{eq:rtheta-asym},
\[
Q'(y)=\frac{\psi'(y)}{r_\theta(\psi(y))}
\sim
(\theta-c_+)e^{\theta y}.
\]
Since $Q(y)\to\infty$, l'H\^opital's rule implies
\[
Q(y)\sim \frac{\theta-c_+}{\theta}e^{\theta y}.
\]
Equivalently,
\[
\Psi(x)\sim \frac{\theta-c_+}{\theta}e^{\theta\phi(x)}
\qquad (x\to\infty).
\tag{A.6.4}\label{eq:Psi-asym}
\]

We next study the increments $W_i-W_{i+1}$. From \eqref{eq:dp-general},
\[
V_i-V_{i+1}=r_\theta(V_{i+1}).
\]
Since $r_\theta$ is decreasing,
\[
W_i-W_{i+1}
=
\int_{V_{i+1}}^{V_i}\frac{dt}{r_\theta(t)}
\ge
\frac{V_i-V_{i+1}}{r_\theta(V_{i+1})}
=1.
\tag{A.6.5}\label{eq:W-lower-increment}
\]
Therefore,
\[
W_j\ge W_n+(n-j)=1+n-j,
\qquad j=1,\dots,n.
\tag{A.6.6}\label{eq:W-lower}
\]

For the upper bound, let $a:=V_{i+1}$. Then $V_i=a+r_\theta(a)$, and so
\[
W_i-W_{i+1}
=
\int_a^{a+r_\theta(a)}\frac{dt}{r_\theta(t)}
\le
\frac{r_\theta(a)}{r_\theta(a+r_\theta(a))}.
\]
Also,
\[
r_\theta(a)-r_\theta(a+r_\theta(a))
=
\int_a^{a+r_\theta(a)} e^{-\theta\phi(t)}\,dt
\le
r_\theta(a)e^{-\theta\phi(a)}.
\]
Hence
\[
r_\theta(a+r_\theta(a))
\ge
r_\theta(a)\bigl(1-e^{-\theta\phi(a)}\bigr),
\]
which implies
\[
W_i-W_{i+1}
\le
\frac{1}{1-e^{-\theta\phi(a)}}.
\tag{A.6.7}\label{eq:W-upper-raw}
\]

Now consider
\[
G(a):=
\left(\frac{1}{1-e^{-\theta\phi(a)}}-1\right)\Psi(a)
=
\frac{e^{-\theta\phi(a)}}{1-e^{-\theta\phi(a)}}\Psi(a).
\]
Using \eqref{eq:Psi-asym}, we have
\[
G(a)\to \frac{\theta-c_+}{\theta}
\qquad (a\to\infty).
\]
Thus $G$ is bounded on $[V_n,\infty)$, so there exists a constant $C>0$ such that
\[
\frac{1}{1-e^{-\theta\phi(a)}}\le 1+\frac{C}{\Psi(a)},
\qquad a\ge V_n.
\]
Applying this to $a=V_{i+1}$ in \eqref{eq:W-upper-raw}, we obtain
\[
W_i-W_{i+1}\le 1+\frac{C}{W_{i+1}}.
\tag{A.6.8}\label{eq:W-upper}
\]

Summing \eqref{eq:W-upper} from $i=1$ to $n-1$ gives
\[
W_1-W_n
\le
n-1+C\sum_{i=1}^{n-1}\frac1{W_{i+1}}.
\]
Using \eqref{eq:W-lower},
\[
\sum_{i=1}^{n-1}\frac1{W_{i+1}}
\le
\sum_{i=1}^{n-1}\frac1{n-i}
=
\sum_{k=1}^{n-1}\frac1k
=
O(\log n).
\]
Since $W_n=1$, it follows that
\[
W_1\le n+O(\log n).
\]
Combining this with the lower bound \eqref{eq:W-lower}, we obtain
\[
W_1=n+O(\log n).
\tag{A.6.9}\label{eq:W1-growth}
\]

Since $W_1=\Psi(V_1)$, combining \eqref{eq:Psi-asym} and \eqref{eq:W1-growth} yields
\[
\frac{\theta-c_+}{\theta}e^{\theta\phi(V_1)}=n(1+o(1)).
\]
Hence
\[
\phi(V_1)
=
\frac{\log n}{\theta}
+
\frac1\theta\log\!\left(\frac{\theta}{\theta-c_+}\right)
+
o(1).
\tag{A.6.10}\label{eq:phiV1}
\]

Let
\[
y_n:=\frac{\log n}{\theta},
\qquad
d:=\frac1\theta\log\!\left(\frac{\theta}{\theta-c_+}\right).
\]
Then \eqref{eq:phiV1} reads
\[
V_1=\psi(y_n+d+o(1)).
\]
Now, since $\psi'(y)/\psi(y)\to c_+$, for every fixed $u$,
\[
\log\frac{\psi(y+u)}{\psi(y)}
=
\int_y^{y+u}\frac{\psi'(t)}{\psi(t)}\,dt
\to c_+u.
\]
Therefore
\[
\frac{\psi(y_n+d)}{\psi(y_n)}\to e^{cd}
=
\left(\frac{\theta}{\theta-c_+}\right)^{c_+/\theta}.
\]
It follows that
\[
V_1
=
\left(\frac{\theta}{\theta-c_+}\right)^{c_+/\theta}
\psi\!\left(\frac{\log n}{\theta}\right)(1+o(1)).
\]
Since $V_1=\mathbb E[X_{\tau^*}]$, this proves the first asymptotic statement.

We now turn to the prophet benchmark. Define
\[
Y_i:=\phi(X_i).
\]
Then $Y_1,\dots,Y_n$ are i.i.d.\ $\mathrm{Exp}(\theta)$ random variables, and
\[
M_n=\psi(Y_{(n)}),
\qquad
Y_{(n)}:=\max_{1\le i\le n}Y_i.
\]
The density of $Y_{(n)}$ is
\[
f_n(y)=n\theta e^{-\theta y}(1-e^{-\theta y})^{n-1},
\qquad y\ge0.
\]
Therefore
\[
\mathbb E[M_n]
=
\int_0^\infty
n\theta e^{-\theta y}(1-e^{-\theta y})^{n-1}\psi(y)\,dy.
\]
With the substitution
\[
z=\theta y-\log n,
\qquad
y=\frac{\log n+z}{\theta},
\]
we obtain
\[
\frac{\mathbb E[M_n]}{\psi((\log n)/\theta)}
=
\int_{-\log n}^{\infty}
e^{-z}\left(1-\frac{e^{-z}}{n}\right)^{n-1}
\frac{\psi((\log n+z)/\theta)}{\psi((\log n)/\theta)}
\,dz.
\tag{A.6.11}\label{eq:prophet-main}
\]

Fix $z\in\mathbb R$. Then
\[
e^{-z}\left(1-\frac{e^{-z}}{n}\right)^{n-1}\to e^{-z-e^{-z}},
\]
and, since $\psi'(y)/\psi(y)\to c_+$,
\[
\frac{\psi((\log n+z)/\theta)}{\psi((\log n)/\theta)}
\to e^{(c_+/\theta)z}.
\tag{A.6.12}\label{eq:psi-ratio-limit}
\]

We now justify dominated convergence in \eqref{eq:prophet-main}. For $z<0$, since $\psi$ is increasing,
\[
\frac{\psi((\log n+z)/\theta)}{\psi((\log n)/\theta)}\le 1,
\]
and
\[
\left(1-\frac{e^{-z}}{n}\right)^{n-1}
\le
\exp\!\left(-\frac{n-1}{n}e^{-z}\right),
\]
so the integrand is bounded by
\[
e^{-z}\exp\!\left(-\frac12 e^{-z}\right)
\]
for all large $n$, which is integrable on $(-\infty,0]$.

For $z\ge0$, choose $\varepsilon>0$ such that $c+\varepsilon<\theta$. Then for all large $n$,
\[
\frac{\psi((\log n+z)/\theta)}{\psi((\log n)/\theta)}
=
\exp\!\left(
\int_{(\log n)/\theta}^{(\log n+z)/\theta}\frac{\psi'(t)}{\psi(t)}\,dt
\right)
\le
e^{((c_++\varepsilon)/\theta)z}.
\]
Since $\left(1-\frac{e^{-z}}{n}\right)^{n-1}\le1$, the integrand is bounded by
\[
e^{-(1-(c_++\varepsilon)/\theta)z},
\]
which is integrable on $[0,\infty)$.

Therefore dominated convergence applies to \eqref{eq:prophet-main}, and by \eqref{eq:psi-ratio-limit},
\[
\lim_{n\to\infty}\frac{\mathbb E[M_n]}{\psi((\log n)/\theta)}
=
\int_{-\infty}^{\infty} e^{-z-e^{-z}}e^{(c_+/\theta)z}\,dz.
\]
With the change of variables $u=e^{-z}$, so that $dz=-du/u$, this becomes
\[
\int_0^\infty u^{-c_+/\theta}e^{-u}\,du
=
\Gamma\!\left(1-\frac{c_+}{\theta}\right).
\]
Hence
\[
\mathbb E[M_n]
=
\Gamma\!\left(1-\frac{c_+}{\theta}\right)
\psi\!\left(\frac{\log n}{\theta}\right)(1+o(1)).
\]

Finally,
\[
\mathrm{CR}_n(\tau^*;\theta)
=
\frac{\mathbb E[X_{\tau^*}]}{\mathbb E[M_n]}
\to
\frac{\left(\frac{\theta}{\theta-c_+}\right)^{c_+/\theta}}
{\Gamma\!\left(1-\frac{c_+}{\theta}\right)}.
\]
This proves the theorem.

\subsection{Proof of Theorem~\ref{thm:general-phi-finite}}\label{app:general-phi-finite}

Recall
\[
\bar{\psi}(y):=x_F-\psi(y), \qquad y\ge 0.
\]
Then \(\bar{\psi}(y)\downarrow 0\), \(\bar{\psi}'(y)=-\psi'(y)\), and
\[
\frac{\bar{\psi}'(y)}{\bar{\psi}(y)}
=
-\frac{\psi'(y)}{x_F-\psi(y)}
\to -c_-.
\tag{A.7.1}\label{eq:g-log-der}
\]

For each \(i\in[n]\), let \(V_i\) denote the optimal expected reward when only periods
\(i,i+1,\dots,n\) remain, and set \(V_{n+1}:=0\). Then the optimal stopping rule is the threshold rule
\[
\tau^*=\inf\{i\in[n]:X_i\ge V_{i+1}\},
\]
and the value function satisfies
\[
V_i=\mathbb E[\max\{X_i,V_{i+1}\}],\qquad i=1,\dots,n.
\]
Since the last observation must be accepted,
\[
V_n=\mathbb E[X_1].
\]

Define, for \(a\in[x_0,x_F)\),
\[
r_\theta(a):=\mathbb E[(X_1-a)^+]
=
\int_a^{x_F}\mathbb P(X_1>t)\,dt
=
\int_a^{x_F}e^{-\theta\phi(t)}\,dt.
\]
Then the dynamic-programming recursion is
\[
V_i=V_{i+1}+r_\theta(V_{i+1}), \qquad i=1,\dots,n-1.
\tag{A.7.2}\label{eq:dp-rw}
\]

We first derive the asymptotic behavior of \(r_\theta(a)\) as \(a\uparrow x_F\).
Write \(a=\psi(y)\). With the change of variables \(t=\psi(u)\), so that \(dt=\psi'(u)\,du\), we get
\[
r_\theta(\psi(y))
=
\int_y^\infty e^{-\theta u}\psi'(u)\,du
=
e^{-\theta y}\psi'(y)\int_0^\infty
e^{-\theta s}\frac{\psi'(y+s)}{\psi'(y)}\,ds.
\tag{A.7.3}\label{eq:rtheta-rw-transform}
\]
Since \(\psi''(y)/\psi'(y)\to -c_-\), for every \(\varepsilon\in(0,c_-)\) there exists \(y_\varepsilon\) such that
\[
\frac{\psi''(u)}{\psi'(u)}\le -(c_--\varepsilon)
\qquad\text{for all }u\ge y_\varepsilon.
\]
Hence, for \(y\ge y_\varepsilon\) and \(s\ge0\),
\[
\frac{\psi'(y+s)}{\psi'(y)}
=
\exp\!\left(\int_y^{y+s}\frac{\psi''(u)}{\psi'(u)}\,du\right)
\le e^{-(c_--\varepsilon)s}.
\]
In the case when $c_-=0$, we can show that for any $\varepsilon>0$,\[
\frac{\psi'(y+s)}{\psi'(y)}
=
\exp\!\left(\int_y^{y+s}\frac{\psi''(u)}{\psi'(u)}\,du\right)
\le e^{\varepsilon s}.
\]

Then for any $c_-\ge 0$, for each fixed \(s\ge0\),
\[
\frac{\psi'(y+s)}{\psi'(y)}\to e^{-c_-s}
\qquad (y\to\infty).
\]
Therefore, dominated convergence in \eqref{eq:rtheta-rw-transform} yields
\[
r_\theta(\psi(y))
\sim
e^{-\theta y}\psi'(y)\int_0^\infty e^{-(\theta+c_-)s}\,ds
=
\frac{\psi'(y)e^{-\theta y}}{\theta+c_-}.
\tag{A.7.4}\label{eq:rtheta-rw-asym}
\]

Next define
\[
\Psi(x):=1+\int_{V_n}^x \frac{dt}{r_\theta(t)},
\qquad V_n\le x<x_F,
\]
and let
\[
W_i:=\Psi(V_i),\qquad i=1,\dots,n.
\]
Then \(W_n=1\). Also define
\[
Q(y):=\Psi(\psi(y)).
\]
By the chain rule and \eqref{eq:rtheta-rw-asym},
\[
Q'(y)=\frac{\psi'(y)}{r_\theta(\psi(y))}
\sim (\theta+c_-)e^{\theta y}.
\]
Since \(Q(y)\to\infty\), l'H\^opital's rule gives
\[
Q(y)\sim \frac{\theta+c_-}{\theta}e^{\theta y}.
\]
Equivalently,
\[
\Psi(x)\sim \frac{\theta+c_-}{\theta}e^{\theta\phi(x)}
\qquad (x\uparrow x_F).
\tag{A.7.5}\label{eq:Psi-rw-asym}
\]

We next study the increments \(W_i-W_{i+1}\). From \eqref{eq:dp-rw},
\[
V_i-V_{i+1}=r_\theta(V_{i+1}).
\]
Since \(r_\theta\) is decreasing,
\[
W_i-W_{i+1}
=
\int_{V_{i+1}}^{V_i}\frac{dt}{r_\theta(t)}
\ge
\frac{V_i-V_{i+1}}{r_\theta(V_{i+1})}
=1.
\tag{A.7.6}\label{eq:W-rw-lower-inc}
\]
Therefore,
\[
W_j\ge W_n+(n-j)=1+n-j,
\qquad j=1,\dots,n.
\tag{A.7.7}\label{eq:W-rw-lower}
\]

For the upper bound, let \(a:=V_{i+1}\). Since \(X_1\le x_F\) a.s.,
\[
r_\theta(a)=\mathbb E[(X_1-a)^+]\le x_F-a,
\]
so \(a+r_\theta(a)\le x_F\). Hence
\[
W_i-W_{i+1}
=
\int_a^{a+r_\theta(a)}\frac{dt}{r_\theta(t)}
\le
\frac{r_\theta(a)}{r_\theta(a+r_\theta(a))}.
\]
Also,
\[
r_\theta(a)-r_\theta(a+r_\theta(a))
=
\int_a^{a+r_\theta(a)}e^{-\theta\phi(t)}\,dt
\le
r_\theta(a)e^{-\theta\phi(a)}.
\]
Thus
\[
r_\theta(a+r_\theta(a))
\ge
r_\theta(a)\bigl(1-e^{-\theta\phi(a)}\bigr),
\]
which implies
\[
W_i-W_{i+1}\le \frac{1}{1-e^{-\theta\phi(a)}}.
\tag{A.7.8}\label{eq:W-rw-upper-raw}
\]

Now consider
\[
G(a):=
\left(\frac{1}{1-e^{-\theta\phi(a)}}-1\right)\Psi(a)
=
\frac{e^{-\theta\phi(a)}}{1-e^{-\theta\phi(a)}}\Psi(a).
\]
Using \eqref{eq:Psi-rw-asym},
\[
G(a)\to \frac{\theta+c_-}{\theta}
\qquad (a\uparrow x_F).
\]
Hence \(G\) is bounded on \([V_n,x_F)\), so there exists \(C>0\) such that
\[
\frac{1}{1-e^{-\theta\phi(a)}}\le 1+\frac{C}{\Psi(a)},
\qquad V_n\le a<x_F.
\]
Applying this to \(a=V_{i+1}\) in \eqref{eq:W-rw-upper-raw}, we obtain
\[
W_i-W_{i+1}\le 1+\frac{C}{W_{i+1}}.
\tag{A.7.9}\label{eq:W-rw-upper}
\]
Summing \eqref{eq:W-rw-upper} from \(i=1\) to \(n-1\) and using \eqref{eq:W-rw-lower},
\[
W_1-W_n
\le
n-1+C\sum_{i=1}^{n-1}\frac1{W_{i+1}}
\le
n-1+C\sum_{k=1}^{n-1}\frac1k
=
n+O(\log n).
\]
Since \(W_n=1\), combining with \eqref{eq:W-rw-lower} yields
\[
W_1=n+O(\log n).
\tag{A.7.10}\label{eq:W1-rw-growth}
\]

Now \(W_1=\Psi(V_1)\), so \eqref{eq:Psi-rw-asym} and \eqref{eq:W1-rw-growth} imply
\[
\frac{\theta+c_-}{\theta}e^{\theta\phi(V_1)}=n(1+o(1)).
\]
Hence
\[
\phi(V_1)
=
\frac{\log n}{\theta}
+
\frac1\theta\log\!\left(\frac{\theta}{\theta+c_-}\right)
+o(1).
\tag{A.7.11}\label{eq:phiV1-rw}
\]
Let
\[
y_n:=\frac{\log n}{\theta},
\qquad
d:=\frac1\theta\log\!\left(\frac{\theta}{\theta+c_-}\right).
\]
Then \eqref{eq:phiV1-rw} reads
\[
V_1=\psi(y_n+d+o(1)).
\]

We now convert this to an asymptotic for the endpoint gap.
Since \eqref{eq:g-log-der} holds, for any bounded sequence \(u_n\to u\),
\[
\log\frac{\bar{\psi}(y_n+u_n)}{\bar{\psi}(y_n)}
=
\int_{y_n}^{y_n+u_n}\frac{\bar{\psi}'(t)}{\bar{\psi}(t)}\,dt
\to -c_-u.
\]
Applying this with \(u_n=d+o(1)\), we obtain
\[
x_F-V_1
=
\bar{\psi}(y_n+d+o(1))
=
e^{-c_-d}\bar{\psi}(y_n)(1+o(1))
=
\left(\frac{\theta+c_-}{\theta}\right)^{c_-/\theta}
\left(x_F-\psi\!\left(\frac{\log n}{\theta}\right)\right)(1+o(1)).
\]
Since \(V_1=\mathbb E_\theta[X_{\tau^*}]\), this proves the first asymptotic formula.

We now turn to the prophet benchmark. Define
\[
Y_i:=\phi(X_i).
\]
Then \(Y_1,\dots,Y_n\) are i.i.d.\ \(\mathrm{Exp}(\theta)\), and
\[
M_n=\max_{1\le i\le n}X_i=\psi(Y_{(n)}),
\qquad
Y_{(n)}:=\max_{1\le i\le n}Y_i.
\]
Let
\[
\bar{\psi}(y)=x_F-\psi(y).
\]
Then
\[
x_F-\mathbb E[M_n]
=
\mathbb E[\bar{\psi}(Y_{(n)})]
=
\int_0^\infty n\theta e^{-\theta y}(1-e^{-\theta y})^{n-1}\bar{\psi}(y)\,dy.
\]
With the substitution
\[
z=\theta y-\log n,
\qquad
y=\frac{\log n+z}{\theta},
\]
we obtain
\[
\frac{x_F-\mathbb E[M_n]}{\bar{\psi}((\log n)/\theta)}
=
\int_{-\log n}^{\infty}
e^{-z}\left(1-\frac{e^{-z}}{n}\right)^{n-1}
\frac{\bar{\psi}((\log n+z)/\theta)}{\bar{\psi}((\log n)/\theta)}\,dz.
\tag{A.7.12}\label{eq:prophet-rw-main}
\]

Fix \(z\in\mathbb R\). By \eqref{eq:g-log-der},
\[
\frac{\bar{\psi}((\log n+z)/\theta)}{\bar{\psi}((\log n)/\theta)}
\to e^{-(c_-/\theta)z}.
\tag{A.7.13}\label{eq:g-ratio-limit}
\]
Also,
\[
e^{-z}\left(1-\frac{e^{-z}}{n}\right)^{n-1}\to e^{-z-e^{-z}}.
\]

We now justify dominated convergence in \eqref{eq:prophet-rw-main}.
For \(z\ge0\), since \(\bar{\psi}\) is decreasing,
\[
\frac{\bar{\psi}((\log n+z)/\theta)}{\bar{\psi}((\log n)/\theta)}\le 1,
\]
so the integrand is bounded by \(e^{-z}\), which is integrable on \([0,\infty)\).

For \(z<0\), choose \(\varepsilon\in(0,c_-)\).
Using \(\psi'(y)/(x_F-\psi(y))\to c_-\), one checks by integrating over
\([(\log n+z)/\theta,(\log n)/\theta]\) and absorbing the part below a fixed \(y_0\) into a constant
that there exists \(C_\varepsilon>0\) such that for all large \(n\),
\[
\frac{\bar{\psi}((\log n+z)/\theta)}{\bar{\psi}((\log n)/\theta)}
\le
C_\varepsilon e^{-((c_-+\varepsilon)/\theta)z},
\qquad -\log n\le z<0.
\]
Also,
\[
\left(1-\frac{e^{-z}}{n}\right)^{n-1}
\le
\exp\!\left(-\frac{n-1}{n}e^{-z}\right)
\le
\exp\!\left(-\frac12 e^{-z}\right)
\]
for all large \(n\). Hence the integrand in \eqref{eq:prophet-rw-main} is bounded by
\[
C_\varepsilon
e^{-(1+(c_-+\varepsilon)/\theta)z}
\exp\!\left(-\frac12 e^{-z}\right),
\]
which is integrable on \((-\infty,0]\).

Therefore dominated convergence applies to \eqref{eq:prophet-rw-main}. Using \eqref{eq:g-ratio-limit},
\[
\lim_{n\to\infty}
\frac{x_F-\mathbb E[M_n]}{\bar{\psi}((\log n)/\theta)}
=
\int_{-\infty}^{\infty} e^{-z-e^{-z}}e^{-(c_-/\theta)z}\,dz.
\]
With the change of variables \(u=e^{-z}\), so that \(dz=-du/u\), this becomes
\[
\int_0^\infty u^{c_-/\theta}e^{-u}\,du
=
\Gamma\!\left(1+\frac{c_-}{\theta}\right).
\]
Thus
\[
x_F-\mathrm{OPT}(\theta)
=
\Gamma\!\left(1+\frac{c_-}{\theta}\right)
\left(x_F-\psi\!\left(\frac{\log n}{\theta}\right)\right)(1+o(1)).
\]

Finally, since \(x_F-\psi((\log n)/\theta)\to0\), both
\(\mathbb E_\theta[X_{\tau^*}]\) and \(\mathrm{OPT}(\theta)\) converge to \(x_F\), so
\[
\mathrm{CR}_n(\tau^*;\theta)\to 1.
\]
This proves the theorem.

\subsection{Proof of Theorem~\ref{cor:general-phi-conv-rate}}\label{app:general-phi-opt-rate}

\begin{lemma}\label{lem:X_bd_V}
    \[
\mathbb E[X_\tau]
\ge
(1-\delta)V_1(\theta_+;N).
\]
\end{lemma}
\begin{proof}

We divide the proof into four steps.

\paragraph{Step 1: Estimating \(\theta\) from the exploration samples.}

From Lemma~\ref{lem:confi},
with
\[
\varepsilon_m:=\sqrt{\frac{4\log(2/\delta)}{m}},
\]
the event
\[
\mathcal G_m
:=
\left\{
\left|\frac{\theta}{\widehat\theta_m}-1\right|
\le \varepsilon_m
\right\}
\]
satisfies
\[
\mathbb P(\mathcal G_m)\ge 1-\delta.
\]

On \(\mathcal G_m\),
\[
\frac{\theta}{1+\varepsilon_m}\le \widehat\theta_m\le \frac{\theta}{1-\varepsilon_m},
\]
and therefore, for
\[
\theta^{(\mathrm U)}:=(1+\varepsilon_m)\widehat\theta_m,
\]
we have
\[
\theta\le \theta^{(\mathrm U)}
\le
\theta\frac{1+\varepsilon_m}{1-\varepsilon_m}
=
\theta_+.
\tag{A.8.1}\label{eq:theta-upper-band}
\]

\paragraph{Step 2: A deterministic comparison lemma.}
Fix a parameter \(\eta\ge \theta\), and suppose that after the exploration phase we use the
DP thresholds computed from \(\eta\), namely
\[
t_i:=V_{i+1}(\eta;N),\qquad i=1,\dots,N-1,
\]
and accept the last observation for sure.

Let \(L_i(\theta,\eta;N)\) be the expected reward of this threshold policy under the true parameter
\(\theta\), starting from remaining stage \(i\in\{1,\dots,N\}\).
Thus
\[
L_N(\theta,\eta;N)=\mathbb E_\theta[X_1],
\]
and for \(i=1,\dots,N-1\),
\[
L_i(\theta,\eta;N)
=
\mathbb E_\theta\!\left[
X_i\mathbf 1\{X_i\ge t_i\}
+
L_{i+1}(\theta,\eta;N)\mathbf 1\{X_i<t_i\}
\right].
\]

We claim that
\[
L_i(\theta,\eta;N)\ge V_i(\eta;N)
\qquad\text{for all }i=1,\dots,N.
\tag{A.8.2}\label{eq:comparison-goal}
\]

We prove \eqref{eq:comparison-goal} by backward induction.

\emph{Base case:}
Since \(\eta\ge \theta\), the tail \(e^{-\theta\phi(x)}\) dominates \(e^{-\eta\phi(x)}\), so
\[
L_N(\theta,\eta;N)=\mathbb E_\theta[X_1]\ge \mathbb E_\eta[X_1]=V_N(\eta;N).
\]

\emph{Induction step:}
Assume \(L_{i+1}(\theta,\eta;N)\ge V_{i+1}(\eta;N)=t_i\).
Using
\[
\mathbb P_\theta(X_i\ge t_i)=e^{-\theta\phi(t_i)}
\]
and
\[
\mathbb E_\theta[X_i\mathbf 1\{X_i\ge t_i\}]
=
t_i e^{-\theta\phi(t_i)}+r_\theta(t_i),
\]
we get
\begin{align*}
L_i(\theta,\eta;N)
&=
(1-e^{-\theta\phi(t_i)})L_{i+1}(\theta,\eta;N)
+
t_i e^{-\theta\phi(t_i)}
+
r_\theta(t_i)\\
&=
t_i
+
(1-e^{-\theta\phi(t_i)})(L_{i+1}(\theta,\eta;N)-t_i)
+
r_\theta(t_i).
\end{align*}
By the induction hypothesis, the middle term is nonnegative. Also \(\eta\ge\theta\) implies
\(r_\theta(t_i)\ge r_\eta(t_i)\). Hence
\[
L_i(\theta,\eta;N)\ge t_i+r_\eta(t_i)=V_i(\eta;N).
\]
This proves \eqref{eq:comparison-goal}.

In particular,
\[
L_1(\theta,\eta;N)\ge V_1(\eta;N).
\tag{A.8.3}\label{eq:det-comparison}
\]

\paragraph{Step 3: Monotonicity of the DP value in the parameter.}
For fixed \(N\), define
\[
T_\eta(v):=v+r_\eta(v),\qquad v\ge x_0.
\]
Then \(T_\eta\) is increasing in \(v\), because
\[
\frac{d}{dv}T_\eta(v)=1-e^{-\eta\phi(v)}=F_\eta(v)\ge 0,
\]
and \(T_\eta\) is decreasing in \(\eta\), because \(r_\eta(v)\) is decreasing in \(\eta\).

Also \(V_N(\eta;N)=\mathbb E_\eta[X_1]\) is decreasing in \(\eta\).
Therefore, a backward induction yields that for every fixed \(i\),
\[
\eta_1\le \eta_2
\quad\Longrightarrow\quad
V_i(\eta_1;N)\ge V_i(\eta_2;N).
\tag{A.8.4}\label{eq:V-monotone}
\]

\paragraph{Step 4: Apply the comparison on the good event.}
Condition on the exploration samples and on the good event \(\mathcal G_m\). Then the algorithm
uses \(\eta=\theta^{(\mathrm U)}\), and by \eqref{eq:theta-upper-band},
\[
\theta\le \theta^{(\mathrm U)}\le \theta_+.
\]
By \eqref{eq:det-comparison},
\[
\mathbb E[X_\tau\mid X_1,\dots,X_m,\mathcal G_m]
\ge
V_1(\theta^{(\mathrm U)};N).
\]
By the monotonicity \eqref{eq:V-monotone},
\[
V_1(\theta^{(\mathrm U)};N)\ge V_1(\theta_+;N).
\]
Hence
\[
\mathbb E[X_\tau\mid X_1,\dots,X_m,\mathcal G_m]\ge V_1(\theta_+;N).
\]
Taking expectations and using \(\mathbb P(\mathcal G_m)\ge 1-\delta\),
\[
\mathbb E[X_\tau]
\ge
\mathbb P(\mathcal G_m)\,V_1(\theta_+;N)
\ge
(1-\delta)V_1(\theta_+;N).
\]
\end{proof}


\begin{lemma}[Uniform asymptotic for the DP value]\label{lem:uniform-general-phi}
 Let
\[
K=[\underline\eta,\overline\eta]\subset (c_+,\infty)
\]
be a compact interval. 
Define
\[
C(\eta):=\left(\frac{\eta}{\eta-c_+}\right)^{c_+/\eta}.
\]
Then, as \(N\to\infty\),
\[
V_1(\eta;N)
=
C(\eta)\psi\!\left(\frac{\log N}{\eta}\right)(1+o(1)),
\]
uniformly in \(\eta\in K\). 
\end{lemma}

\begin{proof}
We write the proof in several steps.

\paragraph{Step 1: Uniform asymptotic for \(r_\eta\).}
Fix \(\eta\in K\). For \(a=\psi(y)\), the change of variables \(t=\psi(u)\) gives
\[
r_\eta(\psi(y))
=
\int_y^\infty e^{-\eta u}\psi'(u)\,du
=
e^{-\eta y}\psi'(y)\int_0^\infty e^{-\eta s}\frac{\psi'(y+s)}{\psi'(y)}\,ds.
\tag{A.8.5}\label{eq:uniform-r-transform}
\]

Since \(\psi''(u)/\psi'(u)\to c_+\), for every \(\varepsilon>0\) there exists \(y_\varepsilon\) such that
\[
\left|\frac{\psi''(u)}{\psi'(u)}-c_+\right|\le \varepsilon
\qquad\text{for all }u\ge y_\varepsilon.
\]
Choose \(\varepsilon>0\) so small that
\[
c_++\varepsilon<\underline\eta.
\]
Then for all \(y\ge y_\varepsilon\) and all \(s\ge 0\),
\[
\frac{\psi'(y+s)}{\psi'(y)}
=
\exp\!\left(\int_y^{y+s}\frac{\psi''(u)}{\psi'(u)}\,du\right)
\le e^{(c_++\varepsilon)s}.
\tag{A.8.6}\label{eq:uniform-ratio-bound}
\]
Moreover, for each fixed \(s\ge 0\),
\[
\frac{\psi'(y+s)}{\psi'(y)}\to e^{cs}
\qquad (y\to\infty).
\]

Now define
\[
I_\eta(y):=\int_0^\infty e^{-\eta s}\frac{\psi'(y+s)}{\psi'(y)}\,ds.
\]
From \eqref{eq:uniform-ratio-bound},
\[
e^{-\eta s}\frac{\psi'(y+s)}{\psi'(y)}
\le e^{-(\underline\eta-c-\varepsilon)s},
\]
and the right-hand side is integrable on \([0,\infty)\). Therefore dominated convergence gives
\[
\sup_{\eta\in K}
\left|
I_\eta(y)-\int_0^\infty e^{-(\eta-c)s}\,ds
\right|
\to 0.
\]
Since
\[
\int_0^\infty e^{-(\eta-c_+)s}\,ds=\frac1{\eta-c},
\]
we obtain
\[
\sup_{\eta\in K}
\left|
\frac{r_\eta(\psi(y))}{\psi'(y)e^{-\eta y}}-\frac1{\eta-c_+}
\right|
\to 0.
\tag{A.8.7}\label{eq:uniform-r-asym}
\]
Equivalently,
\[
r_\eta(\psi(y))
=
\frac{\psi'(y)e^{-\eta y}}{\eta-c}(1+o(1)),
\]
uniformly for \(\eta\in K\).

\paragraph{Step 2: Uniform asymptotic for the scale transform \(\Psi_\eta\).}
Let
\[
\mu(\eta):=\mathbb E_\eta[X_1],
\qquad
\Psi_\eta(x):=1+\int_{\mu(\eta)}^x \frac{dt}{r_\eta(t)},
\qquad x\ge \mu(\eta).
\]
Define also
\[
Q_\eta(y):=\Psi_\eta(\psi(y)).
\]
Then, by the chain rule,
\[
Q_\eta'(y)=\frac{\psi'(y)}{r_\eta(\psi(y))}.
\]
Using \eqref{eq:uniform-r-asym},
\[
\sup_{\eta\in K}
\left|
\frac{Q_\eta'(y)}{(\eta-c_+)e^{\eta y}}-1
\right|
\to 0.
\tag{A.8.8}\label{eq:uniform-Qprime}
\]

We now integrate this asymptotic. Fix \(\rho\in(0,1)\). By \eqref{eq:uniform-Qprime}, there exists
\(y_0\) such that for all \(y\ge y_0\) and all \(\eta\in K\),
\[
(1-\rho)(\eta-c_+)e^{\eta y}
\le
Q_\eta'(y)
\le
(1+\rho)(\eta-c_+)e^{\eta y}.
\]
Integrating from \(y_0\) to \(y\ge y_0\),
\[
(1-\rho)\frac{\eta-c_+}{\eta}\bigl(e^{\eta y}-e^{\eta y_0}\bigr)
\le
Q_\eta(y)-Q_\eta(y_0)
\le
(1+\rho)\frac{\eta-c_+}{\eta}\bigl(e^{\eta y}-e^{\eta y_0}\bigr).
\]
Since \(Q_\eta(y_0)\) is bounded uniformly in \(\eta\in K\), dividing by \(e^{\eta y}\) and letting
\(y\to\infty\) gives
\[
\sup_{\eta\in K}
\left|
\frac{Q_\eta(y)}{((\eta-c_+)/\eta)e^{\eta y}}-1
\right|
\to 0.
\]
Equivalently,
\[
\Psi_\eta(x)
=
\frac{\eta-c_+}{\eta}e^{\eta\phi(x)}(1+o(1))
\qquad (x\to\infty),
\tag{A.8.9}\label{eq:uniform-Psi}
\]
uniformly for \(\eta\in K\).

\paragraph{Step 3: Uniform growth of the transformed DP sequence.}
For \(\eta\in K\) and horizon \(N\), define
\[
W_i^{(\eta,N)}:=\Psi_\eta(V_i(\eta;N)),
\qquad i=1,\dots,N.
\]
Then
\[
W_N^{(\eta,N)}=\Psi_\eta(\mu(\eta))=1.
\]

Since
\[
V_i(\eta;N)-V_{i+1}(\eta;N)=r_\eta(V_{i+1}(\eta;N)),
\]
and \(r_\eta\) is decreasing, we get
\[
W_i^{(\eta,N)}-W_{i+1}^{(\eta,N)}
=
\int_{V_{i+1}(\eta;N)}^{V_i(\eta;N)}\frac{dt}{r_\eta(t)}
\ge
\frac{V_i(\eta;N)-V_{i+1}(\eta;N)}{r_\eta(V_{i+1}(\eta;N))}
=1.
\]
Hence
\[
W_j^{(\eta,N)}\ge 1+N-j,
\qquad j=1,\dots,N.
\tag{A.8.10}\label{eq:uniform-W-lower}
\]

For the upper bound, let \(a:=V_{i+1}(\eta;N)\). Then
\[
W_i^{(\eta,N)}-W_{i+1}^{(\eta,N)}
=
\int_a^{a+r_\eta(a)}\frac{dt}{r_\eta(t)}
\le
\frac{r_\eta(a)}{r_\eta(a+r_\eta(a))}.
\]
Also,
\[
r_\eta(a)-r_\eta(a+r_\eta(a))
=
\int_a^{a+r_\eta(a)} e^{-\eta\phi(t)}\,dt
\le
r_\eta(a)e^{-\eta\phi(a)},
\]
so
\[
r_\eta(a+r_\eta(a))
\ge
r_\eta(a)\bigl(1-e^{-\eta\phi(a)}\bigr),
\]
and therefore
\[
W_i^{(\eta,N)}-W_{i+1}^{(\eta,N)}
\le
\frac{1}{1-e^{-\eta\phi(a)}}.
\tag{A.8.11}\label{eq:uniform-W-upper-raw}
\]

Now define
\[
G_\eta(a)
:=
\left(\frac{1}{1-e^{-\eta\phi(a)}}-1\right)\Psi_\eta(a)
=
\frac{e^{-\eta\phi(a)}}{1-e^{-\eta\phi(a)}}\Psi_\eta(a).
\]
Using \eqref{eq:uniform-Psi},
\[
G_\eta(a)
=
\frac{e^{-\eta\phi(a)}}{1-e^{-\eta\phi(a)}}
\cdot
\frac{\eta-c}{\eta}e^{\eta\phi(a)}(1+o(1))
=
\frac{\eta-c}{\eta}\frac{1+o(1)}{1-e^{-\eta\phi(a)}},
\]
uniformly in \(\eta\in K\). Since \(\eta\ge \underline\eta>0\) and \(\phi(a)\to\infty\), it follows that
\(G_\eta(a)\) is uniformly bounded for all sufficiently large \(a\). By continuity, it is then
bounded on all relevant compact ranges as well. Thus there exists a constant \(B_K>0\),
depending only on \(K\), such that
\[
\frac{1}{1-e^{-\eta\phi(a)}}\le 1+\frac{B_K}{\Psi_\eta(a)},
\qquad a\ge \mu(\eta),\ \eta\in K.
\]
Applying this to \(a=V_{i+1}(\eta;N)\) in \eqref{eq:uniform-W-upper-raw} gives
\[
W_i^{(\eta,N)}-W_{i+1}^{(\eta,N)}
\le
1+\frac{B_K}{W_{i+1}^{(\eta,N)}}.
\tag{A.8.12}\label{eq:uniform-W-upper}
\]

Summing \eqref{eq:uniform-W-upper} from \(i=1\) to \(N-1\), and using \eqref{eq:uniform-W-lower},
\begin{align*}
W_1^{(\eta,N)}-1
&\le
(N-1)+B_K\sum_{i=1}^{N-1}\frac1{W_{i+1}^{(\eta,N)}}\\
&\le
(N-1)+B_K\sum_{i=1}^{N-1}\frac1{N-i}
=
(N-1)+B_K\sum_{k=1}^{N-1}\frac1k.
\end{align*}
Hence
\[
W_1^{(\eta,N)}\le N+O(\log N),
\]
uniformly in \(\eta\in K\). Combining with \eqref{eq:uniform-W-lower}, we obtain
\[
W_1^{(\eta,N)}=N+O(\log N),
\tag{A.8.13}\label{eq:uniform-W1-growth}
\]
uniformly in \(\eta\in K\).

\paragraph{Step 4: Recovering \(V_1(\eta;N)\).}
Since \(W_1^{(\eta,N)}=\Psi_\eta(V_1(\eta;N))\), combining \eqref{eq:uniform-Psi} and
\eqref{eq:uniform-W1-growth} yields
\[
\frac{\eta-c_+}{\eta}e^{\eta\phi(V_1(\eta;N))}(1+o(1))
=
N+O(\log N),
\]
uniformly in \(\eta\in K\). Therefore
\[
\frac{\eta-c_+}{\eta}e^{\eta\phi(V_1(\eta;N))}
=
N(1+o(1)),
\]
uniformly in \(\eta\in K\), and hence
\[
\phi(V_1(\eta;N))
=
\frac{\log N}{\eta}
+
\frac1\eta\log\!\left(\frac{\eta}{\eta-c_+}\right)
+
o(1),
\tag{A.8.14}\label{eq:uniform-phiV1}
\]
uniformly in \(\eta\in K\).

Set
\[
y_N(\eta):=\frac{\log N}{\eta},
\qquad
d(\eta):=\frac1\eta\log\!\left(\frac{\eta}{\eta-c_+}\right).
\]
Then \eqref{eq:uniform-phiV1} says
\[
V_1(\eta;N)=\psi\bigl(y_N(\eta)+d(\eta)+o(1)\bigr),
\]
uniformly in \(\eta\in K\). Since \(K\) is compact, \(d(\eta)\) is bounded on \(K\); let
\[
D:=\sup_{\eta\in K}|d(\eta)|<\infty.
\]
Also,
\[
\inf_{\eta\in K} y_N(\eta)=\frac{\log N}{\overline\eta}\to\infty.
\]

Because \(\psi'(y)/\psi(y)\to c\), for every fixed \(D'>0\),
\[
\sup_{|u|\le D'}
\left|
\log\frac{\psi(y+u)}{\psi(y)}-c_+u
\right|
\to 0
\qquad (y\to\infty).
\]
Applying this with \(u=d(\eta)+o(1)\), which is uniformly bounded, gives
\[
\frac{\psi(y_N(\eta)+d(\eta)+o(1))}{\psi(y_N(\eta))}
=
e^{c_+\,d(\eta)}(1+o(1)),
\]
uniformly in \(\eta\in K\). Since
\[
e^{c_+\,d(\eta)}
=
\exp\!\left(
\frac{c_+}{\eta}\log\!\left(\frac{\eta}{\eta-c_+}\right)
\right)
=
\left(\frac{\eta}{\eta-c_+}\right)^{c_+/\eta}
=
C(\eta),
\]
we conclude that
\[
V_1(\eta;N)
=
C(\eta)\psi\!\left(\frac{\log N}{\eta}\right)(1+o(1)),
\]
uniformly in \(\eta\in K\). This proves the lemma.
\end{proof}
Since \(\theta_{+}\to\theta\) and \(\theta>c_+\), we may choose \(\rho\in(0,\theta-c_+)\) and define
\[
K=[\theta-\rho,\theta+\rho]\subset(c_+,\infty),
\]
such that \(\theta_{+}\in K\) for all sufficiently large \(n\). Therefore, by Lemma~\ref{lem:uniform-general-phi},
\[
V_1(\theta_{+};N_n)
=
C(\theta_{+})
\psi\!\left(\frac{\log N}{\theta_{+}}\right)(1+o(1)),
\]
where
\[
C(\eta):=\left(\frac{\eta}{\eta-c_+}\right)^{c_+/\eta}.
\]
Also, by Theorem~\ref{thm:general-phi},
\[
\mathbb E[\mathrm{OPT}(\theta)]
=
\Gamma\!\left(1-\frac{c_+}{\theta}\right)\psi\!\left(\frac{\log n}{\theta}\right)(1+o(1)).
\]
Therefore
\[
\frac{V_1(\theta_{+};N)}{\mathbb E[\mathrm{OPT}(\theta)]}
=
\rho(c_+,\theta)
\frac{C(\theta_{+})}{C(\theta)}
\frac{\psi\!\left((\log N_n)/\theta_{+}\right)}{\psi(y_n)}
(1+o(1)),
\tag{A.8.15}\label{eq:ratio-split}
\]
where
\[
\rho(c_+,\theta):=
\frac{C(\theta)}{\Gamma(1-c_+/\theta)}
=
\frac{\left(\frac{\theta}{\theta-c_+}\right)^{c_+/\theta}}
{\Gamma(1-c_+/\theta)}.
\]

Since \(C\) is continuous and differentiable in a neighborhood of \(\theta\) and
\[
\theta_{+,n}-\theta
=
\theta\left(\frac{1+\varepsilon_m}{1-\varepsilon_m}-1\right)
=
2\theta\varepsilon_m+O(\varepsilon_m^2),
\]
using the mean value theorem, we have
\[
\frac{C(\theta_{+,n})}{C(\theta)}=1+O(\varepsilon_m).
\tag{A.8.16}\label{eq:C-expand}
\]

Next define
\[
u_n:=\frac{\log N}{\theta_{+}}-\frac{\log n}{\theta}.
\]
Using
\[
\frac1{\theta_{+}}
=
\frac{1-\varepsilon_m}{\theta(1+\varepsilon_m)}
=
\frac1\theta\bigl(1-2\varepsilon_m+O(\varepsilon_m^2)\bigr)
\]
and
\[
\log N
=
\log n+\log(1-m/n)
=
\log n-\frac{m}{n}+O\!\left(\frac{m^2}{n^2}\right),
\]
we get
\[
u_n
=
-\frac{2\varepsilon_m\log n}{\theta}
-\frac{m}{\theta n}
+
o\!\left(\varepsilon_m\log n+\frac{m}{n}\right).
\tag{A.8.17}\label{eq:u-expand}
\]
Let $y_n=\log n/\theta$. Since \(\psi'(y)/\psi(y)\to c\), and \(u_n\to0\), we have
\[
\log\frac{\psi(y_n+u_n)}{\psi(y_n)}
=
\int_{y_n}^{y_n+u_n}\frac{\psi'(t)}{\psi(t)}\,dt
=
c_+\,u_n+o(u_n).
\]
Hence
\[
\frac{\psi(y_n+u_n)}{\psi(y_n)}
=
1+c_+\,u_n+o(u_n).
\tag{A.8.18}\label{eq:psi-expand}
\]

Substituting \eqref{eq:u-expand} into \eqref{eq:psi-expand},
\[
\frac{\psi\!\left((\log N)/\theta_{+}\right)}{\psi(y_n)}
=
1
-
\frac{2c_+}{\theta}\varepsilon_m\log n
-
\frac{c_+}{\theta}\frac{m}{n}
-
o\!\left(\varepsilon_m\log n+\frac{m}{n}\right).
\tag{A.8.19}\label{eq:psi-expand-final}
\]

Combining \eqref{eq:ratio-split}, \eqref{eq:C-expand}, and \eqref{eq:psi-expand-final}, and using
\(\varepsilon_m=o(\varepsilon_m\log n)\), we obtain
\[
\frac{V_1(\theta_{+};N)}{\mathbb E[\mathrm{OPT}(\theta)]}
\ge
\rho(c_+,\theta)
\left(
1
-
\frac{2c_+}{\theta}\varepsilon_m\log n
-
\frac{c_+}{\theta}\frac{m}{n}
-
o\!\left(\varepsilon_m\log n+\frac{m_n}{n}\right)
\right)(1+o(1)).
\]
Together with Lemma~\ref{lem:X_bd_V}, this yields
\[
\mathrm{CR}_n(\tau;\theta)
\ge
\rho(c_+,\theta)(1-\delta)
\left(
1
-
\frac{2c_+}{\theta}\varepsilon_m\log n
-
\frac{c_+}{\theta}\frac{m}{n}
-
o\!\left(\varepsilon_m\log n+\frac{m}{n}\right)
\right)(1+o(1)).
\]



From $m=o(n)$ and $\epsilon_m\log n\to 0$, this proves the theorem.

\subsection{Proof of Theorem~\ref{thm:general-phi-online-rw}}
\label{app:general-phi-online-rw}

\begin{lemma}
    \[
\mathbb E[X_\tau]
\ge
(1-\delta)V_1(\theta_+;N).
\]
\end{lemma}
\begin{proof}
    
The proof is the same four-step argument as in the unbounded-support case, with
\[
r_\eta(a)=\int_a^{x_F}e^{-\eta\phi(t)}\,dt
\]
in place of the integral over \([a,\infty)\).

\paragraph{Step 1: Estimating \(\theta\) from the exploration samples.}
From Lemma~\ref{lem:confi}, with
\[
\varepsilon_m:=\sqrt{\frac{4\log(2/\delta)}{m}},
\]
the event
\[
\mathcal G_m
:=
\left\{
\left|\frac{\theta}{\widehat\theta_m}-1\right|
\le \varepsilon_m
\right\}
\]
satisfies
\[
\mathbb P(\mathcal G_m)\ge 1-\delta.
\]
On \(\mathcal G_m\),
\[
\frac{\theta}{1+\varepsilon_m}\le \widehat\theta_m\le \frac{\theta}{1-\varepsilon_m},
\]
and therefore, for
\[
\theta^{(\mathrm U)}:=(1+\varepsilon_m)\widehat\theta_m,
\]
we have
\[
\theta\le \theta^{(\mathrm U)}
\le
\theta\frac{1+\varepsilon_m}{1-\varepsilon_m}
=
\theta_+.
\tag{A.9.1}\label{eq:theta-upper-band-rw}
\]

\paragraph{Step 2: A deterministic comparison lemma.}
Fix a parameter \(\eta\ge \theta\), and suppose that after the exploration phase we use the
DP thresholds computed from \(\eta\), namely
\[
t_i:=V_{i+1}(\eta;N),\qquad i=1,\dots,N-1,
\]
and accept the last observation for sure.

Let \(L_i(\theta,\eta;N)\) be the expected reward of this threshold policy under the true parameter
\(\theta\), starting from remaining stage \(i\in\{1,\dots,N\}\).
Thus
\[
L_N(\theta,\eta;N)=\mathbb E_\theta[X_1],
\]
and for \(i=1,\dots,N-1\),
\[
L_i(\theta,\eta;N)
=
\mathbb E_\theta\!\left[
X_i\mathbf 1\{X_i\ge t_i\}
+
L_{i+1}(\theta,\eta;N)\mathbf 1\{X_i<t_i\}
\right].
\]

We claim that
\[
L_i(\theta,\eta;N)\ge V_i(\eta;N)
\qquad\text{for all }i=1,\dots,N.
\tag{A.9.2}\label{eq:comparison-goal-rw}
\]

We prove \eqref{eq:comparison-goal-rw} by backward induction.

\emph{Base case:}
Since \(\eta\ge \theta\), the tail \(e^{-\theta\phi(x)}\) dominates \(e^{-\eta\phi(x)}\) on
\([x_0,x_F)\), so
\[
L_N(\theta,\eta;N)=\mathbb E_\theta[X_1]\ge \mathbb E_\eta[X_1]=V_N(\eta;N).
\]

\emph{Induction step:}
Assume \(L_{i+1}(\theta,\eta;N)\ge V_{i+1}(\eta;N)=t_i\).
Using
\[
\mathbb P_\theta(X_i\ge t_i)=e^{-\theta\phi(t_i)}
\]
and
\[
\mathbb E_\theta[X_i\mathbf 1\{X_i\ge t_i\}]
=
t_i e^{-\theta\phi(t_i)}+r_\theta(t_i),
\]
we get
\begin{align*}
L_i(\theta,\eta;N)
&=
(1-e^{-\theta\phi(t_i)})L_{i+1}(\theta,\eta;N)
+
t_i e^{-\theta\phi(t_i)}
+
r_\theta(t_i)\\
&=
t_i
+
(1-e^{-\theta\phi(t_i)})(L_{i+1}(\theta,\eta;N)-t_i)
+
r_\theta(t_i).
\end{align*}
By the induction hypothesis, the middle term is nonnegative. Also \(\eta\ge\theta\) implies
\(r_\theta(t_i)\ge r_\eta(t_i)\). Hence
\[
L_i(\theta,\eta;N)\ge t_i+r_\eta(t_i)=V_i(\eta;N).
\]
This proves \eqref{eq:comparison-goal-rw}.

In particular,
\[
L_1(\theta,\eta;N)\ge V_1(\eta;N).
\tag{A.9.3}\label{eq:det-comparison-rw}
\]

\paragraph{Step 3: Monotonicity of the DP value in the parameter.}
For fixed \(N\), define
\[
T_\eta(v):=v+r_\eta(v),\qquad v\in[x_0,x_F).
\]
Then \(T_\eta\) is increasing in \(v\), because
\[
\frac{d}{dv}T_\eta(v)=1-e^{-\eta\phi(v)}=F_\eta(v)\ge 0,
\]
and \(T_\eta\) is decreasing in \(\eta\), because \(r_\eta(v)\) is decreasing in \(\eta\).

Also \(V_N(\eta;N)=\mathbb E_\eta[X_1]\) is decreasing in \(\eta\).
Therefore, a backward induction yields that for every fixed \(i\),
\[
\eta_1\le \eta_2
\quad\Longrightarrow\quad
V_i(\eta_1;N)\ge V_i(\eta_2;N).
\tag{A.9.4}\label{eq:V-monotone-rw}
\]

\paragraph{Step 4: Apply the comparison on the good event.}
Condition on the exploration samples and on the good event \(\mathcal G_m\). Then the algorithm
uses \(\eta=\theta^{(\mathrm U)}\), and by \eqref{eq:theta-upper-band-rw},
\[
\theta\le \theta^{(\mathrm U)}\le \theta_+.
\]
By \eqref{eq:det-comparison-rw},
\[
\mathbb E[X_\tau\mid X_1,\dots,X_m,\mathcal G_m]
\ge
V_1(\theta^{(\mathrm U)};N).
\]
By the monotonicity \eqref{eq:V-monotone-rw},
\[
V_1(\theta^{(\mathrm U)};N)\ge V_1(\theta_+;N).
\]
Hence
\[
\mathbb E[X_\tau\mid X_1,\dots,X_m,\mathcal G_m]\ge V_1(\theta_+;N).
\]
Taking expectations and using \(\mathbb P(\mathcal G_m)\ge 1-\delta\),
\[
\mathbb E[X_\tau]
\ge
\mathbb P(\mathcal G_m)\,V_1(\theta_+;N)
\ge
(1-\delta)V_1(\theta_+;N).
\]
\end{proof}


\begin{lemma}[Uniform asymptotic for the DP value in the bounded-support case]
\label{lem:uniform-general-phi-rw}
Let
\[
K=[\underline\eta,\overline\eta]\subset(0,\infty)
\]
be compact, and define
\[
C_-(\eta):=\left(\frac{\eta+c}{\eta}\right)^{c/\eta}.
\]
Then, as \(N\to\infty\),
\[
V_1(\eta;N)
=
x_F-
C_-(\eta)\bar\psi\!\left(\frac{\log N}{\eta}\right)(1+o(1)),
\]
uniformly in \(\eta\in K\).
\end{lemma}

\begin{proof}
We write the proof in several steps.

\paragraph{Step 1: Uniform asymptotic for \(r_\eta\).}
Fix \(\eta\in K\). For \(a=\psi(y)\), the change of variables \(t=\psi(u)\) gives
\[
r_\eta(\psi(y))
=
\int_y^\infty e^{-\eta u}\psi'(u)\,du
=
e^{-\eta y}\psi'(y)\int_0^\infty e^{-\eta s}\frac{\psi'(y+s)}{\psi'(y)}\,ds.
\tag{A.9.5}\label{eq:uniform-r-transform-rw}
\]

Since \(\psi''(u)/\psi'(u)\to -c_-\), for each fixed \(s\ge 0\),
\[
\frac{\psi'(y+s)}{\psi'(y)}
=
\exp\!\left(\int_y^{y+s}\frac{\psi''(u)}{\psi'(u)}\,du\right)
\to e^{-c_-s}
\qquad (y\to\infty).
\]
If $c_->0$, choose \(\varepsilon\in(0,c_-)\). For all sufficiently large \(y\) and all \(s\ge 0\),
\[
\frac{\psi'(y+s)}{\psi'(y)}\le e^{-(c_--\varepsilon)s}.
\tag{A.9.6}\label{eq:uniform-ratio-bound-rw}
\]
Hence
\begin{align*}
\sup_{\eta\in K}
\left|
\int_0^\infty e^{-\eta s}\frac{\psi'(y+s)}{\psi'(y)}\,ds
-
\int_0^\infty e^{-(\eta+c_-)s}\,ds
\right|
&\le
\int_0^\infty e^{-\underline\eta s}
\left|
\frac{\psi'(y+s)}{\psi'(y)}-e^{-c_-s}
\right|ds\\
&\to 0,
\end{align*}
by dominated convergence, since the integrand is dominated by
\[
e^{-(\underline\eta+c_--\varepsilon)s}+e^{-(\underline\eta+c_-)s}\in L^1([0,\infty)).
\]
If \(c_-=0\), choose \(0<\varepsilon<\underline{\eta}\). Then, uniformly over
\(\eta\in K\),
\[
e^{-\eta s}\frac{\psi'(y+s)}{\psi'(y)}
\le e^{-(\eta-\varepsilon)s}
\le e^{-(\underline{\eta}-\varepsilon)s},
\]
and
\[
e^{-(\eta+c_-)s}=e^{-\eta s}\le e^{-\underline{\eta}s}.
\]
Since
\[
e^{-(\underline{\eta}-\varepsilon)s}+e^{-\underline{\eta}s}
\in L^1([0,\infty)),
\]
dominated convergence also applies.

Therefore
\[
\sup_{\eta\in K}
\left|
\frac{r_\eta(\psi(y))}{\psi'(y)e^{-\eta y}}-\frac1{\eta+c_-}
\right|
\to 0.
\tag{A.9.7}\label{eq:uniform-r-asym-rw}
\]
Equivalently,
\[
r_\eta(\psi(y))
=
\frac{\psi'(y)e^{-\eta y}}{\eta+c_-}(1+o(1)),
\]
uniformly for \(\eta\in K\).

\paragraph{Step 2: Uniform asymptotic for the scale transform \(\Psi_\eta\).}
Let
\[
\mu(\eta):=\mathbb E_\eta[X_1],
\qquad
\Psi_\eta(x):=1+\int_{\mu(\eta)}^x\frac{dt}{r_\eta(t)},
\qquad x\in[\mu(\eta),x_F).
\]
Define also
\[
Q_\eta(y):=\Psi_\eta(\psi(y)).
\]
Then, by the chain rule,
\[
Q_\eta'(y)=\frac{\psi'(y)}{r_\eta(\psi(y))}.
\]
Using \eqref{eq:uniform-r-asym-rw},
\[
\sup_{\eta\in K}
\left|
\frac{Q_\eta'(y)}{(\eta+c_-)e^{\eta y}}-1
\right|
\to 0.
\tag{A.9.8}\label{eq:uniform-Qprime-rw}
\]

Fix \(\rho\in(0,1)\). Choose \(y_0\) so large that \(\psi(y_0)>\sup_{\eta\in K}\mu(\eta)\) and,
for all \(y\ge y_0\) and all \(\eta\in K\),
\[
(1-\rho)(\eta+c_-)e^{\eta y}
\le
Q_\eta'(y)
\le
(1+\rho)(\eta+c_-)e^{\eta y}.
\]
Integrating from \(y_0\) to \(y\ge y_0\),
\[
(1-\rho)\frac{\eta+c_-}{\eta}\bigl(e^{\eta y}-e^{\eta y_0}\bigr)
\le
Q_\eta(y)-Q_\eta(y_0)
\le
(1+\rho)\frac{\eta+c_-}{\eta}\bigl(e^{\eta y}-e^{\eta y_0}\bigr).
\]
Since \(Q_\eta(y_0)\) is bounded uniformly in \(\eta\in K\), dividing by \(e^{\eta y}\) and letting
\(y\to\infty\) gives
\[
\sup_{\eta\in K}
\left|
\frac{Q_\eta(y)}{((\eta+c_-)/\eta)e^{\eta y}}-1
\right|
\to 0.
\]
Equivalently,
\[
\Psi_\eta(x)
=
\frac{\eta+c_-}{\eta}e^{\eta\phi(x)}(1+o(1))
\qquad (x\uparrow x_F),
\tag{A.9.9}\label{eq:uniform-Psi-rw}
\]
uniformly for \(\eta\in K\).

\paragraph{Step 3: Uniform growth of the transformed DP sequence.}
For \(\eta\in K\) and horizon \(N\), define
\[
W_i^{(\eta,N)}:=\Psi_\eta(V_i(\eta;N)),
\qquad i=1,\dots,N.
\]
Then
\[
W_N^{(\eta,N)}=\Psi_\eta(\mu(\eta))=1.
\]

Since
\[
V_i(\eta;N)-V_{i+1}(\eta;N)=r_\eta(V_{i+1}(\eta;N)),
\]
and \(r_\eta\) is decreasing, we get
\[
W_i^{(\eta,N)}-W_{i+1}^{(\eta,N)}
=
\int_{V_{i+1}(\eta;N)}^{V_i(\eta;N)}\frac{dt}{r_\eta(t)}
\ge
\frac{V_i(\eta;N)-V_{i+1}(\eta;N)}{r_\eta(V_{i+1}(\eta;N))}
=1.
\]
Hence
\[
W_j^{(\eta,N)}\ge 1+N-j,
\qquad j=1,\dots,N.
\tag{A.9.10}\label{eq:uniform-W-lower-rw}
\]

For the upper bound, let \(a:=V_{i+1}(\eta;N)\). Since \(X_1\le x_F\) almost surely,
\[
r_\eta(a)=\mathbb E_\eta[(X_1-a)^+]\le x_F-a,
\]
and therefore \(a+r_\eta(a)\le x_F\). Hence
\[
W_i^{(\eta,N)}-W_{i+1}^{(\eta,N)}
=
\int_a^{a+r_\eta(a)}\frac{dt}{r_\eta(t)}
\le
\frac{r_\eta(a)}{r_\eta(a+r_\eta(a))}.
\]
Also,
\[
r_\eta(a)-r_\eta(a+r_\eta(a))
=
\int_a^{a+r_\eta(a)}e^{-\eta\phi(t)}\,dt
\le
r_\eta(a)e^{-\eta\phi(a)},
\]
so
\[
r_\eta(a+r_\eta(a))
\ge
r_\eta(a)\bigl(1-e^{-\eta\phi(a)}\bigr),
\]
and therefore
\[
W_i^{(\eta,N)}-W_{i+1}^{(\eta,N)}
\le
\frac{1}{1-e^{-\eta\phi(a)}}.
\tag{A.9.11}\label{eq:uniform-W-upper-raw-rw}
\]

Now define
\[
G_\eta(a)
:=
\left(\frac{1}{1-e^{-\eta\phi(a)}}-1\right)\Psi_\eta(a)
=
\frac{e^{-\eta\phi(a)}}{1-e^{-\eta\phi(a)}}\Psi_\eta(a).
\]
Using \eqref{eq:uniform-Psi-rw},
\[
G_\eta(a)
=
\frac{e^{-\eta\phi(a)}}{1-e^{-\eta\phi(a)}}
\cdot
\frac{\eta+c}{\eta}e^{\eta\phi(a)}(1+o(1))
=
\frac{\eta+c}{\eta}\frac{1+o(1)}{1-e^{-\eta\phi(a)}},
\]
uniformly in \(\eta\in K\). Since \(\eta\ge \underline\eta>0\) and \(\phi(a)\to\infty\) as
\(a\uparrow x_F\), it follows that \(G_\eta(a)\) is uniformly bounded for all sufficiently large \(a\).
On the compact set
\[
\{(\eta,a):\eta\in K,\ \mu(\eta)\le a\le a_0\}
\]
it is continuous, hence uniformly bounded there as well.
Therefore there exists a constant \(B_K>0\), depending only on \(K\), such that
\[
\frac{1}{1-e^{-\eta\phi(a)}}\le 1+\frac{B_K}{\Psi_\eta(a)},
\qquad a\in[\mu(\eta),x_F),\ \eta\in K.
\]
Applying this to \(a=V_{i+1}(\eta;N)\) in \eqref{eq:uniform-W-upper-raw-rw} gives
\[
W_i^{(\eta,N)}-W_{i+1}^{(\eta,N)}
\le
1+\frac{B_K}{W_{i+1}^{(\eta,N)}}.
\tag{A.9.12}\label{eq:uniform-W-upper-rw}
\]

Summing \eqref{eq:uniform-W-upper-rw} from \(i=1\) to \(N-1\), and using
\eqref{eq:uniform-W-lower-rw},
\begin{align*}
W_1^{(\eta,N)}-1
&\le
(N-1)+B_K\sum_{i=1}^{N-1}\frac1{W_{i+1}^{(\eta,N)}}\\
&\le
(N-1)+B_K\sum_{i=1}^{N-1}\frac1{N-i}
=
(N-1)+B_K\sum_{k=1}^{N-1}\frac1k.
\end{align*}
Hence
\[
W_1^{(\eta,N)}\le N+O(\log N),
\]
uniformly in \(\eta\in K\). Combining with \eqref{eq:uniform-W-lower-rw}, we obtain
\[
W_1^{(\eta,N)}=N+O(\log N),
\tag{A.9.13}\label{eq:uniform-W1-growth-rw}
\]
uniformly in \(\eta\in K\).

\paragraph{Step 4: Recovering \(V_1(\eta;N)\).}
Since \(W_1^{(\eta,N)}=\Psi_\eta(V_1(\eta;N))\), combining \eqref{eq:uniform-Psi-rw} and
\eqref{eq:uniform-W1-growth-rw} yields
\[
\frac{\eta+c_-}{\eta}e^{\eta\phi(V_1(\eta;N))}(1+o(1))
=
N+O(\log N),
\]
uniformly in \(\eta\in K\). Therefore
\[
\frac{\eta+c_-}{\eta}e^{\eta\phi(V_1(\eta;N))}
=
N(1+o(1)),
\]
uniformly in \(\eta\in K\), and hence
\[
\phi(V_1(\eta;N))
=
\frac{\log N}{\eta}
+
\frac1\eta\log\!\left(\frac{\eta}{\eta+c_-}\right)
+
o(1),
\tag{A.9.14}\label{eq:uniform-phiV1-rw}
\]
uniformly in \(\eta\in K\).

Set
\[
y_N(\eta):=\frac{\log N}{\eta},
\qquad
d(\eta):=\frac1\eta\log\!\left(\frac{\eta}{\eta+c_-}\right).
\]
Then \eqref{eq:uniform-phiV1-rw} says
\[
x_F-V_1(\eta;N)=\bar\psi\bigl(y_N(\eta)+d(\eta)+o(1)\bigr),
\]
uniformly in \(\eta\in K\). Since \(K\) is compact, \(d(\eta)\) is bounded on \(K\); let
\[
D:=\sup_{\eta\in K}|d(\eta)|<\infty.
\]
Also,
\[
\inf_{\eta\in K}y_N(\eta)=\frac{\log N}{\overline\eta}\to\infty.
\]

Because \(\bar\psi'(y)/\bar\psi(y)\to -c_-\), for every fixed \(D'>0\),
\[
\sup_{|u|\le D'}
\left|
\log\frac{\bar\psi(y+u)}{\bar\psi(y)}+c_-u
\right|
\to 0
\qquad (y\to\infty).
\]
Applying this with \(u=d(\eta)+o(1)\), which is uniformly bounded, gives
\[
\frac{\bar\psi(y_N(\eta)+d(\eta)+o(1))}{\bar\psi(y_N(\eta))}
=
e^{-c_-\,d(\eta)}(1+o(1)),
\]
uniformly in \(\eta\in K\). Since
\[
e^{-c_-\,d(\eta)}
=
\exp\!\left(
-\frac{c_-}{\eta}\log\!\left(\frac{\eta}{\eta+c_-}\right)
\right)
=
\left(\frac{\eta+c_-}{\eta}\right)^{c_-/\eta}
=
C_-(\eta),
\]
we conclude that
\[
V_1(\eta;N)
=
x_F-
C_-(\eta)\bar\psi\!\left(\frac{\log N}{\eta}\right)(1+o(1)),
\]
uniformly in \(\eta\in K\). This proves the lemma.
\end{proof}

Since \(\theta_{+}\to\theta\), we may choose \(\rho\in(0,\theta)\) and define
\[
K=[\theta-\rho,\theta+\rho]\subset(0,\infty),
\]
such that \(\theta_{+}\in K\) for all sufficiently large \(n\). Therefore, by
Lemma~\ref{lem:uniform-general-phi-rw} applied with \(\eta=\theta_{+}\),
\[
V_1(\theta_{+};N)
=
x_F-
C_-(\theta_{+})
\bar\psi\!\left(\frac{\log N}{\theta_{+}}\right)(1+o(1)).
\tag{A.9.15}\label{eq:uniform-asym-rw}
\]

We next compare the argument of \(\bar\psi\) with \((\log n)/\theta\):
\[
u_n
:=
\frac{\log N}{\theta_{+}}-\frac{\log n}{\theta}
=
\frac{\log(1-m/n)}{\theta_{+}}
+
\log n\left(\frac1{\theta_{+}}-\frac1\theta\right).
\]
Since \(m=o(n)\), the first term is \(o(1)\). Also
\[
\theta_{+}-\theta=O(\theta\varepsilon_m),
\]
so the second term is \(O(\varepsilon_m\log n)=o(1)\). Thus
\[
u_n\to 0.
\tag{A.9.16}\label{eq:arg-close-rw}
\]

Let
\[
y_n:=\frac{\log n}{\theta}.
\] Now, since \(\bar\psi'(y)/\bar\psi(y)\to -c\),
\[
\log\frac{\bar\psi(y_n+u_n)}{\bar\psi(y_n)}
=
\int_{y_n}^{y_n+u_n}\frac{\bar\psi'(t)}{\bar\psi(t)}\,dt
\to 0,
\]
and therefore
\[
\frac{\bar\psi((\log N)/\theta_{+})}{\bar\psi((\log n)/\theta)}\to 1.
\tag{A.9.17}\label{eq:barpsi-ratio-one-rw}
\]

Also \(C_-(\theta_{+})\to C_-(\theta)\). Combining \eqref{eq:uniform-asym-rw} and
\eqref{eq:barpsi-ratio-one-rw},
\[
V_1(\theta_{+};N)
=
x_F-
C_-(\theta)\bar\psi\!\left(\frac{\log n}{\theta}\right)(1+o(1)).
\tag{A.9.18}\label{eq:plug-in-value-rw}
\]

By the bounded-support analogue of Theorem~\ref{thm:general-phi},
\[
\mathrm{OPT}(\theta)
=
x_F-
\Gamma\!\left(1+\frac{c_-}{\theta}\right)
\bar\psi\!\left(\frac{\log n}{\theta}\right)(1+o(1)).
\tag{A.9.19}\label{eq:opt-rw}
\]

We can obtain \[\frac{1}{\mathrm{OPT}(\theta)}=\frac{1}{x_F}\left(1+\frac{\Gamma(1+\frac{c_-}{\theta})}{x_F}\bar{\psi}(\log n/\theta)+o(\bar{\psi}(\log n/\theta))\right),\]
then we have
\[\frac{V_1(\theta_{+};N)}{\mathrm{OPT}(\theta)}=1-\frac{C_-(\theta)-\Gamma(1+c_-/\theta)}{x_F}\bar{\psi}(\log n/\theta)+o(\bar{\psi}(\log n/\theta)).\]

From Assumption~\ref{ass:c_conv}, in the bounded-support case, if \(c>0\), then
\[
\log \bar\psi(y)
=
-c_-y+o(y),
\qquad y\to\infty.
\]
Therefore,
\[
\bar\psi\!\left(\frac{\log n}{\theta}\right)
=
n^{-c_-/\theta+o(1)}.
\]
When \(c_-=0\), since \(\psi(y)\uparrow x_F\) as \(y\to\infty\), we have
\(\bar\psi(y)=x_F-\psi(y)\to0\), and hence
\[
\bar\psi\!\left(\frac{\log n}{\theta}\right)=o(1).
\]

Combining these estimates, we obtain
\[
\mathrm{CR}_n(\tau;\theta)
\ge
1-\delta-O\!\left(n^{-c_-/\theta+o(1)}\right),
\qquad c_->0,
\]
whereas for \(c_-=0\),
\[
\mathrm{CR}_n(\tau;\theta)
\ge
1-\delta-o(1).
\]

The proof of the theorem is concluded by $n\to \infty$.







\subsection{Proof of Theorem~\ref{thm:exp_cr}}\label{app:exp_cr}
We split the proof into five steps.
\paragraph{Step 1: Dynamic programming values and the plug-in thresholds.}
For the true exponential model \(X_i \sim \mathrm{Exp}(\theta)\), let \(V_i\) be the optimal
dynamic-programming value when stages \(i,i+1,\dots,n\) remain, and set \(V_{n+1}:=0\).
Then
\[
V_i=\mathbb{E}\bigl[\max\{X_i,V_{i+1}\}\bigr], \qquad i=1,\dots,n.
\]
Since \(X_n\) must be accepted at the last stage,
\[
V_n=\mathbb{E}[X_n]=\frac1\theta.
\]
Using \(\max(X,a)=a+(X-a)^+\) and
\[
\mathbb{E}\bigl[(X-a)^+\bigr]
=\int_a^\infty \Pr(X>t)\,dt
=\int_a^\infty e^{-\theta t}\,dt
=\frac{e^{-\theta a}}{\theta},
\]
we get the recursion
\[
V_i=V_{i+1}+\frac{e^{-\theta V_{i+1}}}{\theta},
\qquad i=1,\dots,n-1.
\]

Now define
\[
U_i:=\theta V_i.
\]
Then
\[
U_{n+1}=0,\qquad U_n=1,\qquad U_i=U_{i+1}+e^{-U_{i+1}},\qquad i=1,\dots,n-1.
\]
In particular, the sequence \((U_i)_{i=1}^{n+1}\) does not depend on \(\theta\).

Next fix a deterministic upper bound \(\theta^{(\mathrm U)}\ge \theta\), and define the plug-in
threshold values by
\[
\widehat V_{n+1}:=0,\qquad \widehat V_n:=\frac1{\theta^{(\mathrm U)}},\qquad
\widehat V_i=\widehat V_{i+1}+\frac{e^{-\theta^{(\mathrm U)}\widehat V_{i+1}}}{\theta^{(\mathrm U)}}.
\]
Let
\[
\widehat U_i:=\theta^{(\mathrm U)}\widehat V_i.
\]
Then \(\widehat U_{n+1}=0\), \(\widehat U_n=1\), and
\[
\widehat U_i=\widehat U_{i+1}+e^{-\widehat U_{i+1}},\qquad i=1,\dots,n-1.
\]
Since \((\widehat U_i)\) and \((U_i)\) satisfy the same backward recursion with the same terminal
condition, they are identical:
\[
\widehat U_i=U_i \qquad \text{for all } i.
\]
Therefore
\[
\widehat V_i=\frac{U_i}{\theta^{(\mathrm U)}} \qquad \text{for all } i.
\]

\paragraph{Step 2: Performance of the plug-in threshold policy under the true \(\theta\).}
Consider the stopping policy that, at time \(i\), uses threshold
\[
t_i:=\widehat V_{i+1}=\frac{U_{i+1}}{\theta^{(\mathrm U)}}.
\]
Let \(L_i\) be the expected reward of this plug-in threshold policy from stages \(i,i+1,\dots,n\)
under the true model \(X_j\sim \mathrm{Exp}(\theta)\), and set \(L_{n+1}:=0\).
Thus
\[
L_i=\mathbb{E}\!\left[X_i\mathbf{1}\{X_i\ge t_i\}
+L_{i+1}\mathbf{1}\{X_i<t_i\}\right].
\]
We claim that the plug-in policy loses at most a multiplicative factor
\(\theta/\theta^{(\mathrm U)}\) compared with the optimal DP.

\begin{lemma}
For every \(i\in[n]\),
\[
L_i\ge \frac{\theta}{\theta^{(\mathrm U)}}V_i.
\]
In particular,
\[
L_1\ge \frac{\theta}{\theta^{(\mathrm U)}}V_1.
\]
\end{lemma}

\begin{proof}
Let
\[
R:=\frac{\theta}{\theta^{(\mathrm U)}}\in(0,1],
\qquad
w_i:=\theta L_i.
\]
We first derive a recursion for \(L_i\). Since \(X_i\sim \mathrm{Exp}(\theta)\),
\[
\Pr(X_i\ge t_i)=e^{-\theta t_i}.
\]
Also, by the memoryless property,
\[
\mathbb{E}[X_i\mid X_i\ge t_i]=t_i+\frac1\theta.
\]
Therefore
\begin{align*}
L_i
&=\mathbb{E}[X_i\mathbf 1\{X_i\ge t_i\}]
   +\mathbb{E}[L_{i+1}\mathbf 1\{X_i<t_i\}] \\
&=\Pr(X_i\ge t_i)\,\mathbb{E}[X_i\mid X_i\ge t_i]
  +(1-\Pr(X_i\ge t_i))L_{i+1} \\
&=\left(t_i+\frac1\theta\right)e^{-\theta t_i}
  +(1-e^{-\theta t_i})L_{i+1}.
\end{align*}
Multiplying both sides by \(\theta\), we get
\[
w_i=(\theta t_i+1)e^{-\theta t_i}+(1-e^{-\theta t_i})w_{i+1}.
\]
Since
\[
\theta t_i
=\theta \widehat V_{i+1}
=\frac{\theta}{\theta^{(\mathrm U)}}U_{i+1}
=R\,U_{i+1},
\]
this becomes
\[
w_i=(RU_{i+1}+1)e^{-RU_{i+1}}+(1-e^{-RU_{i+1}})w_{i+1}.
\]

We now prove by backward induction that
\[
w_i\ge R\,U_i \qquad \text{for all } i\in[n].
\]

\emph{Base case:} At stage \(n\), the policy accepts \(X_n\) for sure because \(t_n=\widehat V_{n+1}=0\).
Hence
\[
L_n=\mathbb{E}[X_n]=\frac1\theta,
\qquad
w_n=\theta L_n=1.
\]
Since \(U_n=1\) and \(R\le 1\),
\[
w_n=1\ge R=RU_n.
\]

\emph{Induction step:} Assume that \(w_{i+1}\ge R U_{i+1}\). Then
\begin{align*}
w_i
&=(RU_{i+1}+1)e^{-RU_{i+1}}+(1-e^{-RU_{i+1}})w_{i+1} \\
&\ge (RU_{i+1}+1)e^{-RU_{i+1}}
   +(1-e^{-RU_{i+1}})RU_{i+1} \\
&= RU_{i+1}+e^{-RU_{i+1}}.
\end{align*}
On the other hand, since \(U_i=U_{i+1}+e^{-U_{i+1}}\),
\[
RU_i=RU_{i+1}+R e^{-U_{i+1}}.
\]
Thus it remains to show that
\[
e^{-RU_{i+1}}\ge R e^{-U_{i+1}}.
\]
But for every \(u\ge 0\) and every \(R\in(0,1]\),
\[
\frac{e^{-Ru}}{e^{-u}}=e^{(1-R)u}\ge 1\ge R.
\]
Hence \(e^{-Ru}\ge R e^{-u}\), and therefore
\[
w_i\ge RU_i.
\]
This closes the induction.

Finally, since \(w_i=\theta L_i\) and \(U_i=\theta V_i\), the inequality \(w_i\ge R U_i\) is
equivalent to
\[
L_i\ge \frac{\theta}{\theta^{(\mathrm U)}}V_i.
\]
\end{proof}

As an immediate consequence, if \(\tau\) denotes the stopping time of the plug-in threshold policy,
then
\[
\mathbb{E}[X_\tau]=L_1\ge \frac{\theta}{\theta^{(\mathrm U)}}V_1
=\frac{\theta}{\theta^{(\mathrm U)}}\frac{U_1}{\theta}.
\]
Since for i.i.d.\ exponential rewards
\[
\mathbb{E}\!\left[\max_{1\le i\le n}X_i\right]=\frac{H_n}{\theta},
\qquad
H_n:=\sum_{k=1}^n \frac1k,
\]
we obtain
\[
\frac{\mathbb{E}[X_\tau]}{\mathbb{E}[\max_{1\le i\le n}X_i]}
\ge
\frac{(\theta/\theta^{(\mathrm U)})V_1}{H_n/\theta}
=
\frac{\theta}{\theta^{(\mathrm U)}}\frac{U_1}{H_n}.
\]
Thus the plug-in threshold policy suffers only the multiplicative loss
\(\theta/\theta^{(\mathrm U)}\) relative to the optimal exponential dynamic program.

\paragraph{Step 3: The online algorithm with \(m\) forced exploration stages.}
Now consider the actual online algorithm:
it rejects the first \(m\) observations \(X_1,\dots,X_m\) (forced exploration),
uses them to compute \(\widehat\theta_m\) and \(\theta^{(\mathrm U)}\), and from stage \(m+1\) onward
uses the plug-in DP thresholds computed from \(\theta^{(\mathrm U)}\).
Since the \(X_i\)'s are i.i.d., the first \(m\) rejected observations are distributed exactly like
independent training samples, and they are independent of the future observations
\(X_{m+1},\dots,X_n\).
With $\varepsilon_m=\sqrt{4\log(2/\delta)/m}$, let
\[
\mathcal G:=\left\{\left|\frac{\theta}{\widehat\theta_m}-1\right|\le \varepsilon_m\right\},\]
which holds at least $1-\delta$ probability.
Conditional on the event \(\mathcal G\), we have \(\theta^{(\mathrm U)}\ge \theta\), so by the lemma,
the expected reward collected from stages \(m+1,\dots,n\) is at least
\[
\frac{\theta}{\theta^{(\mathrm U)}}V_{m+1}
\ge
\frac{1-\varepsilon_m}{1+\varepsilon_m}V_{m+1}
=
\frac{1-\varepsilon_m}{1+\varepsilon_m}\frac{U_{m+1}}{\theta}.
\]
On the complement event \(\mathcal G^c\), the reward is nonnegative. Therefore, using
\(\Pr(\mathcal G)\ge 1-\delta\),
\[
\mathbb{E}[X_\tau]
\ge
(1-\delta)\frac{1-\varepsilon_m}{1+\varepsilon_m}\frac{U_{m+1}}{\theta}.
\]
Dividing by \(\mathbb{E}[\max_{1\le i\le n}X_i]=H_n/\theta\), we conclude that
\[
\mathrm{CR}_n(\tau)
:=
\frac{\mathbb{E}[X_\tau]}{\mathbb{E}[\max_{1\le i\le n}X_i]}
\ge
(1-\delta)\frac{1-\varepsilon_m}{1+\varepsilon_m}\frac{U_{m+1}}{H_n},
\qquad
\varepsilon_m=\sqrt{\frac{4\log(2/\delta)}{m}}.
\]

\paragraph{Step 4: Asymptotics of \(U_{m+1}\).}
It remains to estimate \(U_{m+1}\). Let
\[
W_i:=e^{U_i}.
\]
Since \(U_i=U_{i+1}+e^{-U_{i+1}}\), we have
\[
W_i
=
e^{U_{i+1}+e^{-U_{i+1}}}
=
W_{i+1}e^{1/W_{i+1}},
\]
and hence
\[
W_i-W_{i+1}=W_{i+1}\bigl(e^{1/W_{i+1}}-1\bigr).
\]

Because \(e^y-1\ge y\) for all \(y\ge 0\),
\[
W_i-W_{i+1}\ge 1.
\]
Summing from \(i=j\) to \(n-1\) gives
\[
W_j\ge W_n+(n-j)=e+n-j.
\]
In particular, \(W_i\ge e\) for all \(i\), so \(1/W_{i+1}\le 1/e\). For \(0\le y\le 1/e\),
Taylor's theorem implies \(e^y-1\le y+y^2\), and therefore
\[
W_i-W_{i+1}
\le
W_{i+1}\left(\frac1{W_{i+1}}+\frac1{W_{i+1}^2}\right)
=
1+\frac1{W_{i+1}}.
\]
Summing again from \(i=j\) to \(n-1\),
\[
W_j-W_n
\le
(n-j)+\sum_{i=j}^{n-1}\frac1{W_{i+1}}.
\]
Using the lower bound \(W_{i+1}\ge e+n-i-1\),
\[
\sum_{i=j}^{n-1}\frac1{W_{i+1}}
\le
\sum_{k=0}^{n-j-1}\frac1{e+k}
=
O(\log(n-j+1)).
\]
Hence
\[
W_j=(n-j+1)+O(\log(n-j+1)).
\]
Taking logarithms yields
\[
U_j=\log W_j
=
\log(n-j+1)+O\!\left(\frac{\log(n-j+1)}{n-j+1}\right)
=
\log(n-j+1)+o(1).
\]
Setting \(j=m+1\), we obtain
\[
U_{m+1}=\log(n-m)+o(1).
\]

Finally, since
\[
H_n=\log n+\Theta(1),
\]
we conclude that
\[
\frac{U_{m+1}}{H_n}
=
\left(1-\Theta\left(\frac{1}{\log n}\right)\right)\frac{\log(n-m)}{\log n}.
\]
Therefore
\[
\mathrm{CR}_n(\tau)
\ge
(1-\delta)\frac{1-\epsilon}{1+\epsilon}\frac{U_{m+1}}{H_n}
=
(1-\delta)\frac{1-\Theta\bigl(\sqrt{\log(1/\delta)/m}\bigr)}
{1+\Theta\bigl(\sqrt{\log(1/\delta)/m}\bigr)}
\,\left(1-\Theta\left(\frac{1}{\log n}\right)\right)\frac{\log(n-m)}{\log n}.
\]

\paragraph{Step 5: Optimization with respect to $m$.} 


Therefore, for all sufficiently large \(n\),
\[
\mathrm{CR}_n(\tau)
\ge
(1-\delta)\,
\frac{1-\Theta\!\left(\sqrt{\log(1/\delta)/m}\right)}
     {1+\Theta\!\left(\sqrt{\log(1/\delta)/m}\right)}
\cdot
\frac{\log(n-m)}{\log n}\left(1-\Theta\left(\frac{1}{\log n}\right)\right).
\]
To make this precise, there exist constants \(c_1,c_2,c_3,c_4>0\) such that
\[
\mathrm{CR}_n(\tau)
\ge
(1-\delta)\,
\frac{1-c_1\sqrt{\log(1/\delta)/m}}
     {1+c_2\sqrt{\log(1/\delta)/m}}
\cdot
\left(1-c_3\frac{m}{n\log n}\right)\left(1-c_4\frac{1}{\log n}\right),
\]
where in the second factor we used
\[
\frac{\log(n-m)}{\log n}
=
1+\frac{\log(1-m/n)}{\log n}
=
1-O\!\left(\frac{m}{n\log n}\right),
\]
which holds because \(m=o(n)\) and
\[
\log(1-x)=-x+O(x^2)\qquad(x\to 0).
\]

Now set
\[
a_n:=\sqrt{\frac{\log(1/\delta)}{m}}.
\]
Since \(m=\omega(\log(1/\delta))\), we have \(a_n\to 0\). Using
\[
\frac{1-c_1a_n}{1+c_2a_n}
=
1-\frac{(c_1+c_2)a_n}{1+c_2a_n}
\ge
1-C_1 a_n
\]
for some constant \(C_1>0\) and all sufficiently large \(n\), we obtain
\[
\mathrm{CR}_n(\tau)
\ge
(1-\delta)\left(1-C_1\sqrt{\frac{\log(1/\delta)}{m}}\right)
\left(1-c_3\frac{m}{n\log n}\right)\left(1-c_4\frac{1}{\log n}\right).
\]
Using \((1-x)(1-y)\ge 1-x-y\) for \(x,y\ge 0\), it follows that
\[
\mathrm{CR}_n(\tau)
\ge
(1-\delta)\left(
1
-
C_1\sqrt{\frac{\log(1/\delta)}{m}}
-
c_3\frac{m}{n\log n}-c_4\frac{1}{\log n}
\right).
\]
Hence
\[
\mathrm{CR}_n(\tau)
\ge
(1-\delta)\left(
1
-
O\!\left(\sqrt{\frac{\log(1/\delta)}{m}}\right)
-
O\!\left(\frac{m}{n\log n}\right)-
O\!\left(\frac{1}{\log n}\right)
\right).
\]


\subsection{Proof of Theorem~\ref{thm:exp_full_info_rate}}\label{app:exp_full_info_rate}
Let \(V_i\) be the optimal DP value when periods \(i,i+1,\dots,n\) remain,
and set \(V_{n+1}=0\). For exponential rewards,
\[
V_i
=
V_{i+1}
+
\frac{e^{-\theta V_{i+1}}}{\theta},
\qquad i=1,\dots,n-1,
\]
with terminal value \(V_n=1/\theta\). Define
\[
U_i:=\theta V_i.
\]
Then
\[
U_n=1,
\qquad
U_i=U_{i+1}+e^{-U_{i+1}},
\qquad i=1,\dots,n-1.
\]
Let
\[
W_i:=e^{U_i}.
\]
Then \(W_n=e\), and
\[
W_i
=
e^{U_i}
=
e^{U_{i+1}+e^{-U_{i+1}}}
=
W_{i+1}\exp\!\left(\frac1{W_{i+1}}\right).
\]
Therefore
\[
W_i-W_{i+1}
=
W_{i+1}
\left(
\exp\!\left(\frac1{W_{i+1}}\right)-1
\right).
\]

Since \(e^x\ge 1+x\), we have
\[
W_i-W_{i+1}\ge 1.
\]
Thus
\[
W_i\ge W_n+n-i=e+n-i.
\]
In particular \(W_{i+1}\ge e+n-i-1\). Since \(W_{i+1}\ge e\), we have
\(1/W_{i+1}\le 1/e\). Hence Taylor's theorem gives
\[
e^x-1\le x+Cx^2,
\qquad 0\le x\le 1/e,
\]
for some absolute constant \(C>0\). Applying this with
\(x=1/W_{i+1}\), we obtain
\[
W_i-W_{i+1}
\le
1+\frac{C}{W_{i+1}}
\le
1+\frac{C}{e+n-i-1}.
\]
Summing from \(i=1\) to \(n-1\), we get
\[
W_1
=
n+O(\log n).
\]
Consequently,
\[
U_1=\log W_1
=
\log n+O\!\left(\frac{\log n}{n}\right).
\]

Since \(V_1=U_1/\theta\), the optimal full-information expected reward is
\[
\mathbb E_\theta[X_{\tau^*}]
=
\frac{1}{\theta}
\left(
\log n+O\!\left(\frac{\log n}{n}\right)
\right).
\]
On the other hand, for exponential rewards,
\[
\mathbb E_\theta\!\left[\max_{i\in[n]}X_i\right]
=
\frac{H_n}{\theta},
\]
where \(H_n\) is the \(n\)-th harmonic number. Using
\[
H_n
=
\log n+\gamma+O\!\left(\frac1n\right),
\] where \(\gamma \approx 0.57721\) is the Euler--Mascheroni constant,
we obtain
\[
\mathrm{CR}_n(\tau^*;\theta)
=
\frac{U_1}{H_n}
=
\frac{\log n+O(\log n/n)}
{\log n+\gamma+O(1/n)}.
\]
Therefore
\[
\mathrm{CR}_n(\tau^*;\theta)
=
1-\frac{\gamma}{\log n}
+
O\!\left(\frac{1}{\log^2 n}\right).
\]
This proves the claim.

\subsection{Proof of Theorem~\ref{thm:pareto_cr_lower}}\label{app:thm_pareto_cr_lower}

We first provide a lemma for the lower bound.

\begin{lemma} \label{lem:cr_lower_pareto_mid}Algorithm~\ref{alg:general-phi-dp} with DP valuation updates \eqref{eq:DP_pareto} satisfies  
\[\frac{\mathbb{E}[X_\tau]}{\mathbb{E}[\mathrm{OPT}(\theta)]}\ge (1-\delta)\xi_\theta(\alpha^\star)\frac{V_{m+1}(\theta)}{\mathbb{E}[\mathrm{OPT}(\theta)]},\]
where $\alpha^\star=\min_i \frac{V_i(\theta^{(U)})}{V_i(\theta)}$ and $\xi_\theta(\alpha)=\frac{\theta\alpha}{\theta-1+\alpha^\theta}$.
\end{lemma}

\begin{proof} The proof consists of 5 steps.

   \noindent\textbf{Step 1: Dynamic-programming recursion.} Let \(V_i(\theta)\) be the optimal expected reward when stages \(i,i+1,\dots,n\) remain. Then \(V_{n+1}(\theta):=0\), and at the last stage one must accept, so \[ V_n(\theta)=\mathbb E[X_n]=\frac{\theta x_0}{\theta-1}. \] For \(i=n-1,\dots,1\), \[ V_i(\theta)=\mathbb E[\max\{X_i,V_{i+1}(\theta)\}]. \] Since \(V_{i+1}(\theta)\ge V_n(\theta)>x_0\), we may use the Pareto tail formula: for any \(a\ge x_0\), \[ \mathbb E[(X-a)^+] = \int_a^\infty \Pr(X>t)\,dt = \int_a^\infty \left(\frac{x_0}{t}\right)^\theta dt = \frac{x_0^\theta}{\theta-1}a^{1-\theta}. \] Hence \[ V_i(\theta) = V_{i+1}(\theta)+\frac{x_0^\theta}{\theta-1}V_{i+1}(\theta)^{1-\theta}, \qquad i=n-1,\dots,1. \] Define the normalized values \[ U_i(\theta):=\frac{V_i(\theta)}{x_0}. \] Then \[ U_n(\theta)=\frac{\theta}{\theta-1}, \qquad U_i(\theta)=U_{i+1}(\theta)+\frac{1}{\theta-1}U_{i+1}(\theta)^{1-\theta}, \qquad i=n-1,\dots,1. \]

\noindent\textbf{Step 2: The prophet benchmark.} Let \(M_n=\max_{1\le j\le n}X_j\). Writing \(X_j=x_0 U_j^{-1/\theta}\) with \(U_j\stackrel{\mathrm{i.i.d.}}{\sim}\mathrm{Unif}(0,1)\), one gets \[ \mathbb E[\mathrm{OPT}(\theta)] = x_0\,\Gamma\!\left(1-\frac1\theta\right) \frac{\Gamma(n+1)}{\Gamma(n+1-1/\theta)}. \] By the Gamma-ratio asymptotic, \[ \frac{\Gamma(n+1)}{\Gamma(n+1-1/\theta)} = n^{1/\theta}(1+o(1)), \] so \[ \mathbb E[\mathrm{OPT}(\theta)] = x_0\,\Gamma\!\left(1-\frac1\theta\right)n^{1/\theta}(1+o(1)). \] 
    
\noindent\textbf{Step 3: Recursion for the plug-in policy after exploration.} With $\varepsilon_m=\sqrt{4\log(2/\delta)/m}$, let
\[
\mathcal G_m:=\left\{\left|\frac{\theta}{\widehat\theta_{m}}-1\right|\le \varepsilon_m\right\},\]
which holds at least $1-\delta$ probability from Lemma~\ref{lem:confi}. On the event \(\mathcal G_m\), fix the realized value of \(\theta^{(\mathrm U)}\). The policy rejects the first \(m\) observations and then uses thresholds \[ t_i:=V_{i+1}(\theta^{(\mathrm U)}),\qquad i=m+1,\dots,n-1, \] and accepts at stage \(n\). Let \(L_i\) be the expected reward of this plug-in policy starting from stage \(i\), under the true parameter \(\theta\). Then \[ L_n=\mathbb E[X_n]=V_n(\theta). \] For \(i=n-1,\dots,m+1\), conditioning on whether \(X_i\ge t_i\), \[ L_i = \Pr(X_i\ge t_i)\,\mathbb E[X_i\mid X_i\ge t_i] + \Pr(X_i<t_i)\,L_{i+1}. \] Now \[ \Pr(X_i\ge t_i)=\left(\frac{x_0}{t_i}\right)^\theta, \qquad \mathbb E[X_i\mid X_i\ge t_i]=\frac{\theta}{\theta-1}t_i. \] Therefore \[ L_i = \left(\frac{x_0}{t_i}\right)^\theta\frac{\theta t_i}{\theta-1} + \left(1-\left(\frac{x_0}{t_i}\right)^\theta\right)L_{i+1}. \] 
    
\noindent\textbf{Step 4: A comparison lemma.} We first prove a deterministic comparison bound.
    \begin{lemma} Fix \(\alpha\in(0,1]\). Suppose that \[ t_i\ge \alpha\,V_{i+1}(\theta),\qquad i=m+1,\dots,n-1. \] Then for every \(i=m+1,\dots,n\), \[ L_i\ge \xi_\theta(\alpha)\,V_i(\theta), \qquad \xi_\theta(\alpha)=\frac{\theta\alpha}{\theta-1+\alpha^\theta}. \] \end{lemma} 
    
    \begin{proof} We argue by backward induction on \(i\). For \(i=n\), we have \(L_n=V_n(\theta)\). Since \[ \xi_\theta(\alpha)\le 1 \qquad (\alpha\in(0,1]), \] it follows that \[ L_n=V_n(\theta)\ge \xi_\theta(\alpha)V_n(\theta). \] To verify \(\xi_\theta(\alpha)\le 1\), note that this is equivalent to \[ \theta\alpha \le \theta-1+\alpha^\theta. \] Define \[ h(\alpha):=\theta-1+\alpha^\theta-\theta\alpha. \] Then \(h(1)=0\), and \[ h'(\alpha)=\theta\alpha^{\theta-1}-\theta=\theta(\alpha^{\theta-1}-1)\le 0 \quad\text{for } \alpha\in(0,1]. \] Hence \(h(\alpha)\ge h(1)=0\). Now suppose that for some \(i\in\{m+1,\dots,n-1\}\), \[ L_{i+1}\ge \xi_\theta(\alpha)V_{i+1}(\theta). \] Write \[ t_i=\beta_i V_{i+1}(\theta) \] with \(\beta_i\ge \alpha\). Since \(t_i=V_{i+1}(\theta^{(\mathrm U)})\) and \(\theta^{(\mathrm U)}\ge \theta\), one also has \(\beta_i\le 1\); indeed, if \(\eta'\ge \eta\), then under the coupling \[ X_j(\eta)=x_0U_j^{-1/\eta},\qquad U_j\sim \mathrm{Unif}(0,1), \] we have \(X_j(\eta')\le X_j(\eta)\) almost surely for every \(j\). Therefore any stopping rule earns no more under \(\eta'\) than under \(\eta\), so \[ V_i(\eta')\le V_i(\eta)\qquad\text{for all }i. \] Applying this with \(\eta'=\theta^{(\mathrm U)}\) and \(\eta=\theta\) yields \(t_i\le V_{i+1}(\theta)\), i.e. \(\beta_i\le 1\). Using the recursion for \(L_i\) and the induction hypothesis, \begin{align*} L_i &= \left(\frac{x_0}{t_i}\right)^\theta\frac{\theta t_i}{\theta-1} + \left(1-\left(\frac{x_0}{t_i}\right)^\theta\right)L_{i+1} \\ &\ge \left(\frac{x_0}{t_i}\right)^\theta\frac{\theta t_i}{\theta-1} + \left(1-\left(\frac{x_0}{t_i}\right)^\theta\right)\xi_\theta(\alpha)V_{i+1}(\theta) \\ &= \xi_\theta(\alpha)V_{i+1}(\theta) + \left(\frac{x_0}{t_i}\right)^\theta \left( \frac{\theta t_i}{\theta-1}-\xi_\theta(\alpha)V_{i+1}(\theta) \right). \end{align*} Substituting \(t_i=\beta_iV_{i+1}(\theta)\), \[ L_i \ge \xi_\theta(\alpha)V_{i+1}(\theta) + \left(\frac{x_0}{\beta_iV_{i+1}(\theta)}\right)^\theta \left( \frac{\theta\beta_iV_{i+1}(\theta)}{\theta-1} -\xi_\theta(\alpha)V_{i+1}(\theta) \right). \] On the other hand, from the recursion for \(V_i(\theta)\), \[ \xi_\theta(\alpha)V_i(\theta) = \xi_\theta(\alpha)V_{i+1}(\theta) + \xi_\theta(\alpha)\frac{x_0^\theta}{\theta-1}V_{i+1}(\theta)^{1-\theta}. \] Thus it is enough to show \[ \left(\frac{x_0}{\beta_iV_{i+1}(\theta)}\right)^\theta \left( \frac{\theta\beta_iV_{i+1}(\theta)}{\theta-1} -\xi_\theta(\alpha)V_{i+1}(\theta) \right) \ge \xi_\theta(\alpha)\frac{x_0^\theta}{\theta-1}V_{i+1}(\theta)^{1-\theta}. \] After cancelling the common factor \(x_0^\theta V_{i+1}(\theta)^{1-\theta}\), this is equivalent to \[ \beta_i^{-\theta}\left(\frac{\theta\beta_i}{\theta-1}-\xi_\theta(\alpha)\right) \ge \frac{\xi_\theta(\alpha)}{\theta-1}, \] or, equivalently, \[ \theta\beta_i \ge \xi_\theta(\alpha)\bigl(\theta-1+\beta_i^\theta\bigr). \] But \[ \xi_\theta(\beta)=\frac{\theta\beta}{\theta-1+\beta^\theta}, \] so the last inequality is exactly \[ \xi_\theta(\beta_i)\ge \xi_\theta(\alpha). \] This holds because \(\beta_i\ge \alpha\) and \(\xi_\theta\) is increasing on \((0,1]\). Indeed, \[ \xi_\theta'(\beta) = \frac{\theta(\theta-1)(1-\beta^\theta)} {(\theta-1+\beta^\theta)^2} \ge 0 \qquad\text{for }\beta\in(0,1]. \] Hence \(L_i\ge \xi_\theta(\alpha)V_i(\theta)\), completing the induction. \end{proof} 
    
\noindent\textbf{Step 5: Proof of the theorem.} On \(\mathcal G_m\), define \[ \alpha_n^*:=\min_{m+1\le i\le n}\frac{V_i(\theta^{(\mathrm U)})}{V_i(\theta)}. \] Then for every \(i=m+1,\dots,n-1\), \[ t_i=V_{i+1}(\theta^{(\mathrm U)}) \ge \alpha_n^\star\,V_{i+1}(\theta). \] By the lemma with \(\alpha=\alpha_n^\star\), \[ L_{m+1}\ge \xi_\theta(\alpha_n^\star)V_{m+1}(\theta). \] Since \(X_\tau\ge 0\) always, we obtain \begin{align*} \mathbb E[X_\tau] &\ge \mathbb E[X_\tau\mathbf 1_{\mathcal G_m}] \\ &= \Pr(\mathcal G_m)\,\mathbb E[X_\tau\mid \mathcal G_m] \\ &\ge (1-\delta)\,\xi_\theta(\alpha_n^\star)\,V_{m+1}(\theta). \end{align*} Dividing by \(\mathbb E[\mathrm{OPT}(\theta)]\) yields \[ \frac{\mathbb E[X_\tau]}{\mathbb E[\mathrm{OPT}(\theta)]} \ge (1-\delta)\,\xi_\theta(\alpha_n^\star)\,\frac{V_{m+1}(\theta)}{\mathbb E[\mathrm{OPT}(\theta)]}. \] 

\end{proof}
Under $\mathcal G_m$, we have $\theta\le \theta^{(U)}\le \frac{1+\varepsilon_m}{1-\varepsilon_m}\theta:=(1+\eta_m)\theta$.
By Lemma~\ref{lem:cr_lower_pareto_mid},
\[
\frac{\mathbb E[X_{\tau_n}]}{\mathbb E[\mathrm{OPT}(\theta)]}
\ge
(1-\delta)\,
\xi_\theta(\alpha_n^\star)\,
\frac{V_{m+1}(\theta)}{\mathbb E[\mathrm{OPT}(\theta)]},
\qquad
\xi_\theta(\alpha)=\frac{\theta\alpha}{\theta-1+\alpha^\theta},
\]
where
\[
\alpha_n^\star:=\min_{m+1\le i\le n}\frac{V_i(\theta^{(\mathrm U)})}{V_i(\theta)}.
\]

We first lower-bound $\alpha_n^\star$. Since $V_i(\beta)$ is decreasing in $\beta$,
on $\mathcal G_m$ we have
\[
\alpha_n^\star
\ge
\min_{m+1\le i\le n}\frac{V_i((1+\eta_m)\theta)}{V_i(\theta)}.
\]

\begin{lemma}[Pareto DP value asymptotic]
\label{lem:value_asym}
Let
\[
D(\beta):=\left(\frac{\beta}{\beta-1}\right)^{1/\beta},
\qquad \beta>1.
\]
For Pareto rewards with shape parameter \(\beta>1\), let \(V_i(\beta)\)
be the optimal DP value when periods \(i,i+1,\dots,n\) remain. Then, with
\(r:=n-i+1\),
\[
V_i(\beta)
=
x_0 D(\beta) r^{1/\beta}
\left(
1+
\frac{\beta-1}{2\beta^2}\frac{\log r}{r}
+
O\!\left(\frac1r\right)
\right).
\]
Equivalently,
\[
V_i(\beta)
=
x_0 D(\beta)(n-i+1)^{1/\beta}
\left(
1+
\frac{\beta-1}{2\beta^2}
\frac{\log(n-i+1)}{n-i+1}
+
O\!\left(\frac1{n-i+1}\right)
\right).
\]
\end{lemma}

\begin{proof}
Let \(r=n-i+1\) be the number of stages remaining from time \(i\), and define
\[
\widetilde U_r(\beta)
:=
\frac{V_{n-r+1}(\beta)}{x_0}.
\]
Then
\[
\widetilde U_1(\beta)=\frac{\beta}{\beta-1},
\]
and, for \(r\ge2\),
\[
\widetilde U_r(\beta)
=
\widetilde U_{r-1}(\beta)
+
\frac{1}{\beta-1}\widetilde U_{r-1}(\beta)^{1-\beta}.
\]
Set
\[
Y_r(\beta):=\widetilde U_r(\beta)^\beta.
\]
Then
\[
Y_r(\beta)
=
Y_{r-1}(\beta)
\left(
1+
\frac{1}{(\beta-1)Y_{r-1}(\beta)}
\right)^\beta.
\]
Let
\[
a_\beta:=\frac{\beta}{\beta-1}.
\]
Expanding the right-hand side gives
\[
Y_r(\beta)-Y_{r-1}(\beta)
=
a_\beta
+
\frac{\beta}{2(\beta-1)}
\frac1{Y_{r-1}(\beta)}
+
O\!\left(\frac1{Y_{r-1}(\beta)^2}\right).
\]
Since the rough estimate
\[
Y_r(\beta)=a_\beta r+O(\log r)
\]
follows by the same upper and lower bounds as in the previous proof, we have
\[
\sum_{s=1}^{r-1}\frac1{Y_s(\beta)}
=
\frac1{a_\beta}\log r+O(1),
\qquad
\sum_{s=1}^{r-1}\frac1{Y_s(\beta)^2}
=
O(1).
\]
Therefore,
\[
Y_r(\beta)
=
a_\beta r
+
\frac{\beta}{2(\beta-1)}
\cdot
\frac1{a_\beta}
\log r
+
O(1).
\]
Since \(a_\beta=\beta/(\beta-1)\), this simplifies to
\[
Y_r(\beta)
=
a_\beta r
+
\frac12\log r
+
O(1).
\]
Taking the \(1/\beta\)-th power,
\[
\widetilde U_r(\beta)
=
\left(
a_\beta r+\frac12\log r+O(1)
\right)^{1/\beta}.
\]
Thus
\[
\widetilde U_r(\beta)
=
a_\beta^{1/\beta}r^{1/\beta}
\left(
1+
\frac1\beta
\frac{\frac12\log r+O(1)}{a_\beta r}
+
O\!\left(\frac{\log^2 r}{r^2}\right)
\right).
\]
Using \(a_\beta=\beta/(\beta-1)\), we obtain
\[
\widetilde U_r(\beta)
=
D(\beta)r^{1/\beta}
\left(
1+
\frac{\beta-1}{2\beta^2}
\frac{\log r}{r}
+
O\!\left(\frac1r\right)
\right).
\]
Multiplying by \(x_m\) gives the desired expansion for \(V_i(\beta)\).
\end{proof}

From Lemma~\ref{lem:value_asym},
\[
V_i(\beta)=x_0\,D(\beta)\,(n-i+1)^{1/\beta}(1+o(1)),
\]
uniformly for $\beta>1$. Hence for $i>m$,
\[
\frac{V_i((1+\eta_m)\theta)}{V_i(\theta)}
=
\frac{D((1+\eta_m)\theta)}{D(\theta)}
(n-i+1)^{\frac{1}{(1+\eta_m)\theta}-\frac{1}{\theta}}
(1+o(1)).
\]
Because $\frac{1}{(1+\eta_m)\theta}-\frac{1}{\theta}<0$, the minimum is attained at the
largest remaining horizon, namely $i=m+1$. Therefore
\[
\alpha_n^\star
\ge
\underline \alpha_n\,(1+o(1)),
\qquad
\underline \alpha_n:=
\frac{D((1+\eta_m)\theta)}{D(\theta)}
(n-m)^{\frac{1}{(1+\eta_m)\theta}-\frac{1}{\theta}}.
\]

Now let
\[
g(\beta):=\log D(\beta)=\frac{1}{\beta}\log\frac{\beta}{\beta-1}.
\]
Then
\[
g'(\beta)
=
-\frac{1}{\beta^2}
\left(
\frac{1}{\beta-1}
+
\log\frac{\beta}{\beta-1}
\right),
\]
so that
\[
-\theta g'(\theta)
=
\frac{1}{\theta}
\left(
\frac{1}{\theta-1}
+
\log\frac{\theta}{\theta-1}
\right)
=
B_\theta.
\]
A first-order Taylor expansion gives
\[
g((1+\eta_m)\theta)-g(\theta)
=
\theta g'(\theta)\eta_m+O(\eta_m^2)
=
-B_\theta\eta_m+O(\eta_m^2).
\]
Also,
\[
\frac{1}{(1+\eta_m)\theta}-\frac{1}{\theta}
=
-\frac{\eta_m}{(1+\eta_m)\theta}
=
-\frac{\eta_m}{\theta}+O(\eta_m^2).
\]
Therefore
\begin{align*}
\log \underline\alpha_n
&=
g((1+\eta_m)\theta)-g(\theta)
+
\left(
\frac{1}{(1+\eta_m)\theta}-\frac{1}{\theta}
\right)\log(n-m) \\
&=
-\left(\frac{\log(n-m)}{\theta}+B_\theta\right)\eta_m
+
O(\eta_m^2\log n) \\
&=
-A_\theta(n,m)\eta_m+O(\eta_m^2\log n),
\end{align*}
where $A_\theta(n,m):=\frac{\log(n-m)}{\theta}+B_\theta$.
Because $\eta_m\log(n-m)\to 0$, the exponent tends to $0$, and hence
\[
1-\underline\alpha_n
=
A_\theta(n,m)\eta_m+o(\eta_m\log n).
\]

Next we expand $\xi_\theta$ around $1$. From $\xi_\theta(\alpha)=\frac{\theta\alpha}{\theta-1+\alpha^\theta}, $ one checks directly that
\[
\xi_\theta(1)=1,\qquad
\xi_\theta'(1)=0,\qquad
\xi_\theta''(1)=1-\theta.
\]
Thus, for $s\downarrow 0$,
\[
\xi_\theta(1-s)
=
1-\frac{\theta-1}{2}s^2+O(s^3).
\]
Since $\xi_\theta$ is non-decreasing under $\alpha\le 1$ and $\theta>1$, and $\alpha_n^\star\ge \underline\alpha_n(1+o(1))$,
\[
\xi_\theta(\alpha_n^\star)
\ge
\xi_\theta(\underline\alpha_n)
=
1-\frac{\theta-1}{2}A_\theta(n,m)^2\eta_m^2
+
o(\eta_m^2\log^2 n).
\]

We now estimate the exploration loss more precisely. Let
\[
N:=n-m,
\qquad
a_\theta:=\frac{\theta-1}{2\theta^2}.
\]
By Lemma~\ref{lem:value_asym},
\[
V_{m+1}(\theta)
=
x_0 D(\theta)N^{1/\theta}
\left(
1+
a_\theta\frac{\log N}{N}
+
O\!\left(\frac1N\right)
\right).
\]
On the other hand,
\[
\mathbb E[\mathrm{OPT}(\theta)]
=
x_0 \Gamma\!\left(1-\frac1\theta\right)
\frac{\Gamma(n+1)}{\Gamma(n+1-1/\theta)}.
\]
Using the Gamma-ratio expansion,
\[
\frac{\Gamma(n+1)}{\Gamma(n+1-1/\theta)}
=
n^{1/\theta}
\left(
1+
a_\theta\frac1n
+
O\!\left(\frac1{n^2}\right)
\right),
\]
we obtain
\[
\mathbb E[\mathrm{OPT}(\theta)]
=
x_0 \Gamma\!\left(1-\frac1\theta\right)
n^{1/\theta}
\left(
1+
a_\theta\frac1n
+
O\!\left(\frac1{n^2}\right)
\right).
\]
Therefore,
\begin{align*}
\frac{V_{m+1}(\theta)}{\mathbb E[\mathrm{OPT}(\theta)]}
&=
\frac{D(\theta)}{\Gamma(1-1/\theta)}
\left(\frac{N}{n}\right)^{1/\theta}
\left(
1+
a_\theta\frac{\log N}{N}
+
O\!\left(\frac1N\right)
\right)
\left(
1-
a_\theta\frac1n
+
O\!\left(\frac1{n^2}\right)
\right).
\end{align*}
Since \(m=o(n)\), \(N=n-m\), and
\[
\left(\frac{N}{n}\right)^{1/\theta}
=
1-\frac{m}{\theta n}
+
O\!\left(\frac{m^2}{n^2}\right),
\]
while
\[
\frac{\log N}{N}
=
\frac{\log n}{n}
+
O\!\left(\frac{m\log n}{n^2}\right),
\]
we get
\[
\frac{V_{m+1}(\theta)}{\mathbb E[\mathrm{OPT}(\theta)]}
=
\frac{D(\theta)}{\Gamma(1-1/\theta)}
\left[
1
+
a_\theta\frac{\log n}{n}
-
\frac{m}{\theta n}
+
O\!\left(
\frac1n
+
\frac{m\log n}{n^2}
+
\frac{m^2}{n^2}
\right)
\right].
\]


Combining the lower bound on \(\xi_\theta(\alpha_n^\star)\) with the refined
estimate of \(V_{m+1}(\theta)/\mathbb E[\mathrm{OPT}(\theta)]\), we obtain
\begin{align*}
\frac{\mathbb E[X_{\tau_n}]}{\mathbb E[\mathrm{OPT}(\theta)]}
&\ge
(1-\delta)
\frac{D(\theta)}{\Gamma(1-1/\theta)}
\Bigg[
1
+
a_\theta\frac{\log n}{n}
-
\frac{m}{\theta n}
-
\frac{\theta-1}{2}A_\theta(n,m)^2\eta_m^2  \\
&\hspace{4cm}
+
o\!\left(\eta_m^2\log^2 n\right)
+
O\!\left(
\frac1n
+
\frac{m\log n}{n^2}
+
\frac{m^2}{n^2}
\right)
\Bigg],
\end{align*}
where
\[
a_\theta=\frac{\theta-1}{2\theta^2},
\qquad
A_\theta(n,m)=\frac{\log(n-m)}{\theta}+B_\theta,
\]
and
\[
\eta_m=\frac{2\varepsilon_m}{1-\varepsilon_m},
\qquad
\varepsilon_m=\sqrt{\frac{\log(1/\delta)}{m}}.
\]
Since
\[
m=\omega(\log(1/\delta)\log^2 n),
\]
we have
\[
\eta_m^2
=
\frac{\log(1/\delta)}{m}(1+o(1)).
\]
Also, since \(m=o(n)\),
\[
A_\theta(n,m)=\frac{\log n}{\theta}+O(1).
\]
Therefore,
\[
A_\theta(n,m)^2\eta_m^2
=
O\!\left(
\frac{\log(1/\delta)\log^2 n}{m}
\right).
\]
Moreover,
\[
\frac{m\log n}{n^2}+\frac{m^2}{n^2}
=
o\!\left(\frac{m}{n}\right).
\]
Thus
\[
\mathrm{CR}_n(\tau;\theta)
\ge
(1-\delta)
\frac{D(\theta)}{\Gamma(1-1/\theta)}
\left[
1
+
\frac{\theta-1}{2\theta^2}\frac{\log n}{n}
-
O\!\left(
\frac{m}{n}
+
\frac{\log(1/\delta)\log^2 n}{m}
+
\frac1n
\right)
\right],
\]
which implies
\[
\mathrm{CR}_n(\tau;\theta)
\ge
(1-\delta)
\frac{D(\theta)}{\Gamma(1-1/\theta)}
\left[
1
-
O\!\left(
\frac{m}{n}
+
\frac{\log(1/\delta)\log^2 n}{m}
\right)
\right].
\]

\subsection{Suboptimality of Relative Rank-only 
Rules for Fr\'echet tails}\label{app:limit_rank}

We first provide a proposition to show the suboptimality of relative rank-based 
rules for Fr\'echet tails.

\begin{proposition}[Relative rank-only rules are suboptimal for Fr\'echet tails]
\label{prop:rank_only_frechet_gap}
Let \(X_1,\dots,X_n\) be i.i.d.\ nonnegative continuous random variables whose
upper quantile function
\[
U(t):=F^{-1}(1-1/t)
\]
is regularly varying with index \(\gamma\in(0,1)\). Let \(\tau_n\) be any
stopping rule based only on relative ranks. Then
\[
\limsup_{n\to\infty}
\frac{\mathbb E[X_{\tau_n}]}{\mathbb E[X_{(n)}]}
\le
1-\gamma\left(1-\frac1e\right),
\]
where $X_{(n)}$ is the largest order statistic.
Moreover,
\[
1-\gamma\left(1-\frac1e\right)
<
\frac{(1-\gamma)^{-\gamma}}{\Gamma(1-\gamma)}.
\]
Consequently, no relative-rank-only rule can attain the full-information
Fr\'echet-specific limit
\[
\rho_\gamma:=\frac{(1-\gamma)^{-\gamma}}{\Gamma(1-\gamma)}.
\]
In the parametric family of this paper, \(\gamma=c_+/\theta\), so
\(\rho_\gamma=\rho(c_+,\theta)\).
\end{proposition}

\begin{proof}
Let \(A_{\tau_n}\) denote the absolute rank selected by the rule, with
\(A_{\tau_n}=1\) corresponding to selecting the maximum \(X_{(n)}\). Since
\(\tau_n\) uses only relative ranks, the classical secretary upper bound \citep{ferguson1989solved} gives
\[
\limsup_{n\to\infty}\mathbb P(A_{\tau_n}=1)\le \frac1e .
\]
On \(\{A_{\tau_n}=1\}\), the selected value is \(X_{(n)}\). On the complement,
the selected value is at most \(X_{(n-1)}\). Since relative ranks are independent
of the order statistics,
\[
\mathbb E[X_{\tau_n}]
\le
p_n\,\mathbb E[X_{(n)}]
+
(1-p_n)\,\mathbb E[X_{(n-1)}],
\qquad
p_n:=\mathbb P(A_{\tau_n}=1).
\]
Therefore
\[
\frac{\mathbb E[X_{\tau_n}]}{\mathbb E[X_{(n)}]}
\le
p_n+(1-p_n)
\frac{\mathbb E[X_{(n-1)}]}{\mathbb E[X_{(n)}]}.
\]

It remains to identify the limiting ratio of the two largest order statistics.
Since \(U\) is regularly varying with index \(\gamma\in(0,1)\), standard
Fr\'echet order-statistic convergence gives
\[
\left(
\frac{X_{(n)}}{U(n)},\frac{X_{(n-1)}}{U(n)}
\right)
\Rightarrow
\left(
\Gamma_1^{-\gamma},\Gamma_2^{-\gamma}
\right),
\]
where \(\Gamma_j=E_1+\cdots+E_j\) and \(E_1,E_2,\dots\) are i.i.d.\
\(\mathrm{Exp}(1)\). Since \(\gamma<1\), the corresponding first moments
converge, and hence
\[
\frac{\mathbb E[X_{(n)}]}{U(n)}
\to
\mathbb E[\Gamma_1^{-\gamma}]
=
\Gamma(1-\gamma),
\]
while
\[
\frac{\mathbb E[X_{(n-1)}]}{U(n)}
\to
\mathbb E[\Gamma_2^{-\gamma}]
=
\Gamma(2-\gamma).
\]
Thus
\[
\frac{\mathbb E[X_{(n-1)}]}{\mathbb E[X_{(n)}]}
\to
\frac{\Gamma(2-\gamma)}{\Gamma(1-\gamma)}
=
1-\gamma.
\]
Combining this with the secretary upper bound yields
\[
\limsup_{n\to\infty}
\frac{\mathbb E[X_{\tau_n}]}{\mathbb E[X_{(n)}]}
\le
\frac1e+\left(1-\frac1e\right)(1-\gamma)
=
1-\gamma\left(1-\frac1e\right).
\]

Finally, the full-information Fr\'echet-specific limit in our parametric
family is
\[
\rho_\gamma
=
\frac{(1-\gamma)^{-\gamma}}{\Gamma(1-\gamma)}.
\]
For \(\gamma\in(0,1)\), one has
\[
\rho_\gamma
>
1-\gamma\left(1-\frac1e\right).
\]
Indeed, using the expansion
\[
\log\Gamma(1-\gamma)
=
\gamma_E\gamma+\sum_{k=2}^{\infty}\frac{\zeta(k)}{k}\gamma^k,
\]
where \(\gamma_{\mathrm E}\approx 0.57721\) denotes the
Euler--Mascheroni constant,
we get
\[
\log\rho_\gamma
=
-\gamma_E\gamma
+
\sum_{k=2}^{\infty}
\left(\frac1{k-1}-\frac{\zeta(k)}{k}\right)\gamma^k
>
-\gamma_E\gamma,
\]
because \(\zeta(k)<k/(k-1)\) for \(k\ge2\). Hence
\[
\rho_\gamma>e^{-\gamma_E\gamma}
\ge 1-\gamma_E\gamma
>
1-\gamma\left(1-\frac1e\right),
\]
since \(\gamma_E<1-1/e\). This proves the strict separation.
\end{proof}
Although the preceding results already imply the suboptimality of rank-based rules for Pareto tails, we include the
following proposition for completeness, showing that relative-rank-only rules
are suboptimal for Pareto tails.

\begin{proposition}[Relative rank-only rules are suboptimal for Pareto tails]
\label{prop:rank_only_pareto_gap}
Let \(X_1,\dots,X_n\) be i.i.d.\ Pareto random variables with scale \(x_0>0\)
and shape parameter \(\theta>1\):
\[
F_\theta(x)=1-\left(\frac{x_0}{x}\right)^\theta,\qquad x\ge x_0.
\]
Let \(\tau_n\) be any stopping rule based only on relative ranks. Then
\[
\limsup_{n\to\infty}
\frac{\mathbb E[X_{\tau_n}]}{\mathbb E[X_{(n)}]}
\le
1-\frac1\theta\left(1-\frac1e\right).
\]
Moreover,
\[
1-\frac1\theta\left(1-\frac1e\right)
<
\frac{\left(\theta/(\theta-1)\right)^{1/\theta}}
{\Gamma(1-1/\theta)}
=
\rho(1,\theta).
\]
Consequently, no relative-rank-only rule can attain the full-information
Pareto-specific asymptotic competitive ratio \(\rho(1,\theta)\).
\end{proposition}

\begin{proof}
Let \(A_{\tau_n}\) denote the absolute rank of the selected observation, with
\(A_{\tau_n}=1\) corresponding to selecting the maximum \(X_{(n)}\). Since
\(\tau_n\) is based only on relative ranks, the classical secretary upper bound
implies
\[
\limsup_{n\to\infty}\mathbb P(A_{\tau_n}=1)\le \frac1e .
\]
On the event \(\{A_{\tau_n}=1\}\), the selected value is \(X_{(n)}\). On the
complement, the selected value is at most \(X_{(n-1)}\). Hence, using the
independence between the order statistics and the rank vector,
\[
\mathbb E[X_{\tau_n}]
\le
\mathbb P(A_{\tau_n}=1)\mathbb E[X_{(n)}]
+
\bigl(1-\mathbb P(A_{\tau_n}=1)\bigr)
\mathbb E[X_{(n-1)}].
\]
Dividing by \(\mathbb E[X_{(n)}]\) gives
\[
\frac{\mathbb E[X_{\tau_n}]}{\mathbb E[X_{(n)}]}
\le
p_n+(1-p_n)
\frac{\mathbb E[X_{(n-1)}]}{\mathbb E[X_{(n)}]},
\qquad
p_n:=\mathbb P(A_{\tau_n}=1).
\]

It remains to compute the order-statistic ratio. Let
\(\gamma:=1/\theta\in(0,1)\). If \(U\sim\mathrm{Unif}(0,1)\), then
\(X\stackrel{d}{=}x_0U^{-\gamma}\). Therefore, if
\(U_{(1)}\le U_{(2)}\le\cdots\le U_{(n)}\) are the uniform order statistics,
\[
X_{(n)}\stackrel{d}{=}x_0U_{(1)}^{-\gamma},
\qquad
X_{(n-1)}\stackrel{d}{=}x_0U_{(2)}^{-\gamma}.
\]
Using the densities of \(U_{(1)}\) and \(U_{(2)}\),
\[
\mathbb E[X_{(n)}]
=
x_0 n\int_0^1 u^{-\gamma}(1-u)^{n-1}\,du
=
x_0\frac{\Gamma(n+1)\Gamma(1-\gamma)}
{\Gamma(n+1-\gamma)},
\]
and
\[
\mathbb E[X_{(n-1)}]
=
x_0 n(n-1)\int_0^1 u^{1-\gamma}(1-u)^{n-2}\,du
=
x_0\frac{\Gamma(n+1)\Gamma(2-\gamma)}
{\Gamma(n+1-\gamma)}.
\]
Thus
\[
\frac{\mathbb E[X_{(n-1)}]}{\mathbb E[X_{(n)}]}
=
\frac{\Gamma(2-\gamma)}{\Gamma(1-\gamma)}
=
1-\gamma
=
1-\frac1\theta.
\]
Combining this identity with the secretary upper bound yields
\[
\limsup_{n\to\infty}
\frac{\mathbb E[X_{\tau_n}]}{\mathbb E[X_{(n)}]}
\le
\frac1e+\left(1-\frac1e\right)\left(1-\frac1\theta\right)
=
1-\frac1\theta\left(1-\frac1e\right).
\]

Finally, the full-information Pareto-specific limit is
\[
\rho(1,\theta)
=
\frac{\left(\theta/(\theta-1)\right)^{1/\theta}}
{\Gamma(1-1/\theta)}
=
\frac{(1-\gamma)^{-\gamma}}{\Gamma(1-\gamma)}.
\]
For \(\gamma\in(0,1)\), this quantity is strictly larger than
\(1-\gamma(1-1/e)\). Hence the rank-only upper bound is strictly below
\(\rho(1,\theta)\), proving the claim.
\end{proof}
\subsection{Full-information Convergence Rate for Pareto Rewards}\label{app:pareto_full_info_rate}

\begin{proposition}[Full-information convergence rate for Pareto rewards]
\label{lem:pareto_full_info_rate}
Suppose \(X_1,\dots,X_n\stackrel{\mathrm{i.i.d.}}{\sim}\mathrm{Pareto}(\theta)\).
Let \(\tau^*\) be the optimal full-information stopping rule. Then
\[
\mathrm{CR}_n(\tau^*;\theta)
=
\rho(1,\theta)
\left(
1+\Theta\left(\frac{\log n}{n}
\right)\right).
\]
\end{proposition}

\begin{proof}

Let \(V_i(\theta)\) denote the optimal DP value when periods
\(i,i+1,\dots,n\) remain. Then \(V_{n+1}(\theta)=0\), and
\[
V_i(\theta)
=
\mathbb E[\max\{X_i,V_{i+1}(\theta)\}],
\qquad i=1,\dots,n.
\]
Since the final observation must be accepted,
\[
V_n(\theta)=\mathbb E[X_1]=\frac{\theta x_0}{\theta-1}.
\]
For \(a\ge x_0\),
\[
\mathbb E[(X_1-a)^+]
=
\int_a^\infty \Pr(X_1>t)\,dt
=
\int_a^\infty \left(\frac{x_0}{t}\right)^\theta dt
=
\frac{x_0^\theta}{\theta-1}a^{1-\theta}.
\]
Therefore,
\[
V_i(\theta)
=
V_{i+1}(\theta)
+
\frac{x_0^\theta}{\theta-1}V_{i+1}(\theta)^{1-\theta},
\qquad i=1,\dots,n-1.
\]

It is convenient to re-index by the number of remaining periods. Define
\[
\widetilde U_r
:=
\frac{V_{n-r+1}(\theta)}{x_0},
\qquad r=1,\dots,n.
\]
Then
\[
\widetilde U_1=\frac{\theta}{\theta-1},
\]
and for \(r\ge2\),
\[
\widetilde U_r
=
\widetilde U_{r-1}
+
\frac{1}{\theta-1}\widetilde U_{r-1}^{1-\theta}.
\]
Set
\[
Y_r:=\widetilde U_r^\theta.
\]
Then
\[
Y_r
=
Y_{r-1}
\left(
1+\frac{1}{(\theta-1)Y_{r-1}}
\right)^\theta.
\]
Let
\[
a_\theta:=\frac{\theta}{\theta-1}.
\]
Using \((1+u)^\theta\ge 1+\theta u\) for \(u\ge0\), we get
\[
Y_r-Y_{r-1}
\ge
a_\theta.
\]
Hence
\[
Y_r\ge Y_1+a_\theta(r-1)
\ge c_\theta r
\]
for some constant \(c_\theta>0\).

For the upper bound, since \(Y_r\to\infty\), Taylor's theorem gives, for all large \(r\),
\[
(1+u)^\theta
\le
1+\theta u+C_\theta u^2
\]
with \(u=1/((\theta-1)Y_{r-1})\). Therefore
\[
Y_r-Y_{r-1}
\le
a_\theta+\frac{C_\theta}{Y_{r-1}}.
\]
Using the lower bound \(Y_{r-1}\ge c_\theta(r-1)\), we obtain
\[
Y_r-Y_{r-1}
\le
a_\theta+\frac{C_\theta}{r}.
\]
Summing over \(r\), this yields
\[
Y_r
=
a_\theta r+O(\log r).
\]
Consequently,
\[
\widetilde U_r
=
Y_r^{1/\theta}
=
\left(a_\theta r+O(\log r)\right)^{1/\theta}.
\]
Thus
\[
\widetilde U_r
=
a_\theta^{1/\theta}r^{1/\theta}
\left(
1+O\!\left(\frac{\log r}{r}\right)
\right).
\]
Taking \(r=n\), we get
\[
V_1(\theta)
=
x_0
\left(\frac{\theta}{\theta-1}\right)^{1/\theta}
n^{1/\theta}
\left(
1+O\!\left(\frac{\log n}{n}\right)
\right).
\]
Since \(V_1(\theta)=\mathbb E[X_{\tau^*}]\), this gives the full-information DP value.

We now compute the prophet benchmark. Let \(M_n=\max_{i\in[n]}X_i\). Write
\[
X_i=x_0 U_i^{-1/\theta},
\qquad
U_i\stackrel{\mathrm{i.i.d.}}{\sim}\mathrm{Unif}(0,1).
\]
Then
\[
M_n=x_0 U_{(1)}^{-1/\theta},
\]
where \(U_{(1)}=\min_{i\in[n]}U_i\). Since \(U_{(1)}\) has density
\[
n(1-u)^{n-1},\qquad u\in[0,1],
\]
we obtain
\[
\mathbb E[M_n]
=
x_0 n\int_0^1 u^{-1/\theta}(1-u)^{n-1}\,du.
\]
Using the beta function,
\[
\mathbb E[M_n]
=
x_0 n B\!\left(1-\frac1\theta,n\right)
=
x_0 \Gamma\!\left(1-\frac1\theta\right)
\frac{\Gamma(n+1)}{\Gamma(n+1-1/\theta)}.
\]

By the standard Gamma-ratio expansion,
\[
\frac{\Gamma(n+1)}{\Gamma(n+1-1/\theta)}
=
n^{1/\theta}
\left(
1+O\!\left(\frac1n\right)
\right).
\]
Therefore
\[
\mathbb E[M_n]
=
x_0
\Gamma\!\left(1-\frac1\theta\right)
n^{1/\theta}
\left(
1+O\!\left(\frac1n\right)
\right).
\]

Combining the asymptotic expansion for \(V_1(\theta)\) with the expansion for
\(\mathbb E[M_n]\), we obtain
\[
\mathrm{CR}_n(\tau^*;\theta)
=
\frac{V_1(\theta)}{\mathbb E[M_n]}
=
\frac{
\left(\frac{\theta}{\theta-1}\right)^{1/\theta}
}{
\Gamma\!\left(1-\frac1\theta\right)
}
\left(
1+O\!\left(\frac{\log n}{n}\right)
\right).
\]
This proves the claim.
\end{proof}

\subsection{Proof of Theorem~\ref{thm:uniform_family_cr}}\label{app:uniform_family_cr}


Let
\[
D:=x_F-x_0.
\]
Consider
\[
\phi(x)=\log\!\left(\frac{D}{x_F-x}\right),\qquad x\in[x_0,x_F),
\]
so that
\[
F_\theta(x)=1-\exp(-\theta\phi(x))
=
1-\left(\frac{x_F-x}{D}\right)^\theta,
\qquad x\in[x_0,x_F].
\]
Equivalently,
\[
\bar F_\theta(x)
=
1-F_\theta(x)
=
\left(\frac{x_F-x}{D}\right)^\theta.
\]

For the estimated value update in Line~\ref{line:DP} of
Algorithm~\ref{alg:general-phi-dp}, note that for any \(a\in[x_0,x_F]\),
\begin{align}
r_\eta(a)
&:=
\int_a^{x_F} e^{-\eta\phi(t)}\,dt
=
\int_a^{x_F}\left(\frac{x_F-t}{D}\right)^\eta dt \notag\\
&=
\frac{(x_F-a)^{\eta+1}}{(\eta+1)D^\eta}. \tag{A.15.1}\label{eq:r_eta_bounded_power}
\end{align}
Therefore the plug-in DP recursion becomes
\begin{align}
\widehat V_N
&=
\mathbb E_{\theta^{(\mathrm U)}}[X_1]
=
x_0+\frac{D}{\theta^{(\mathrm U)}+1}, \tag{A.15.2}\label{eq:DP_bounded_power_terminal}\\
\widehat V_i
&=
\widehat V_{i+1}
+
\frac{(x_F-\widehat V_{i+1})^{\theta^{(\mathrm U)}+1}}
     {(\theta^{(\mathrm U)}+1)D^{\theta^{(\mathrm U)}}},
\qquad i=1,\dots,N-1. \tag{A.15.3}\label{eq:DP_bounded_power}
\end{align}

That is, at each stage \(t\in[N]\), the algorithm uses the corresponding plug-in DP threshold
determined by \eqref{eq:DP_bounded_power_terminal}--\eqref{eq:DP_bounded_power}.

Let \(\delta\in(0,1)\), $m=o(n)$, and $m=\omega (\log(1/\delta)\log^2n)$.  We define\[
\theta_+:=\theta\frac{1+\varepsilon_m}{1-\varepsilon_m}.
\] Let
\(V_i(\eta;N)\) denote the \(N\)-period optimal DP values under known parameter \(\eta\), namely
\begin{align}
V_N(\eta;N)
&=
x_0+\frac{D}{\eta+1}, \tag{A.15.4}\label{eq:true_DP_terminal}\\
V_i(\eta;N)
&=
V_{i+1}(\eta;N)
+
\frac{(x_F-V_{i+1}(\eta;N))^{\eta+1}}
     {(\eta+1)D^\eta},
\qquad i=1,\dots,N-1. \tag{A.15.5}\label{eq:true_DP_recursion}
\end{align}

\begin{lemma}\label{lem:X_V_bd_finite}
    \[
\mathbb E[X_\tau]
\ge
(1-\delta)V_1(\theta_+;N).
\]
\end{lemma}
\begin{proof}
    
We split the proof into four steps.

\paragraph{Step 1: Concentration of the plug-in estimator.}
By Lemma~\ref{lem:confi},
\[
\mathcal G_m
:=
\left\{
\left|\frac{\theta}{\widehat\theta_m}-1\right|
\le \varepsilon_m
\right\}
\]
satisfies
\[
\Pr(\mathcal G_m)\ge 1-\delta.
\]
On \(\mathcal G_m\),
\[
\frac{\theta}{1+\varepsilon_m}\le \widehat\theta_m\le \frac{\theta}{1-\varepsilon_m},
\]
and thus
\[
\theta\le \theta^{(\mathrm U)}=(1+\varepsilon_m)\widehat\theta_m
\le
\theta\frac{1+\varepsilon_m}{1-\varepsilon_m}
=
\theta_+.
\]

\paragraph{Step 2: A deterministic comparison lemma.}
Fix any \(\eta\ge \theta\). Consider the policy on the last \(N\) stages that uses the DP thresholds
computed under parameter \(\eta\), namely
\[
t_i:=V_{i+1}(\eta;N),\qquad i=1,\dots,N-1,
\]
and accepts the last observation for sure. Let \(L_i(\theta,\eta;N)\) be the expected reward of
this policy from stage \(i\) onward under the true parameter \(\theta\). We claim that
\[
L_i(\theta,\eta;N)\ge V_i(\eta;N)\qquad\text{for all }i=1,\dots,N.
\]

We prove this by backward induction.

For the base case \(i=N\),
\[
L_N(\theta,\eta;N)=\mathbb E_\theta[X_1]
=
x_0+\int_{x_0}^{x_F}\left(\frac{x_F-t}{D}\right)^\theta dt
=
x_0+\frac{D}{\theta+1}.
\]
Since \(\eta\ge \theta\), we have \(1/(\theta+1)\ge 1/(\eta+1)\), so
\[
L_N(\theta,\eta;N)
=
x_0+\frac{D}{\theta+1}
\ge
x_0+\frac{D}{\eta+1}
=
V_N(\eta;N).
\]

Now assume that \(L_{i+1}(\theta,\eta;N)\ge V_{i+1}(\eta;N)=t_i\). Then
\begin{align*}
L_i(\theta,\eta;N)
&=
\mathbb E_\theta\!\left[X_i\mathbf 1\{X_i\ge t_i\}
+L_{i+1}(\theta,\eta;N)\mathbf 1\{X_i<t_i\}\right] \\
&=
\mathbb E_\theta\!\left[X_i\mathbf 1\{X_i\ge t_i\}\right]
+\Pr(X_i<t_i)L_{i+1}(\theta,\eta;N).
\end{align*}
Using
\[
\mathbb E_\theta\!\left[X_i\mathbf 1\{X_i\ge t_i\}\right]
=
t_i\,\Pr(X_i\ge t_i)+r_\theta(t_i),
\]
we obtain
\begin{align*}
L_i(\theta,\eta;N)
&=
t_i\,\Pr(X_i\ge t_i)+r_\theta(t_i)
+\bigl(1-\Pr(X_i\ge t_i)\bigr)L_{i+1}(\theta,\eta;N) \\
&=
t_i
+\bigl(1-\Pr(X_i\ge t_i)\bigr)\bigl(L_{i+1}(\theta,\eta;N)-t_i\bigr)
+r_\theta(t_i).
\end{align*}
By the induction hypothesis, the middle term is nonnegative. Also \(\eta\ge \theta\) implies
\(r_\theta(t_i)\ge r_\eta(t_i)\), because the integrand \(e^{-\lambda\phi(t)}\) decreases in
\(\lambda\). Therefore
\[
L_i(\theta,\eta;N)\ge t_i+r_\eta(t_i)=V_i(\eta;N).
\]
This completes the induction, and in particular
\[
L_1(\theta,\eta;N)\ge V_1(\eta;N).
\]

\paragraph{Step 3: Monotonicity of \(V_i(\eta;N)\) in \(\eta\).}
For fixed \(N\), the terminal value
\[
V_N(\eta;N)=x_0+\frac{D}{\eta+1}
\]
is decreasing in \(\eta\). Also, for fixed \(a\in[x_0,x_F]\),
\[
r_\eta(a)=\frac{(x_F-a)^{\eta+1}}{(\eta+1)D^\eta}
\]
is decreasing in \(\eta\), because \((x_F-a)/D\in[0,1]\). Hence a backward induction on
\eqref{eq:true_DP_terminal}--\eqref{eq:true_DP_recursion} shows that for each fixed \(i\),
\[
\eta_1\le \eta_2
\quad\Longrightarrow\quad
V_i(\eta_1;N)\ge V_i(\eta_2;N).
\]

\paragraph{Step 4: Apply the comparison on the good event.}
Conditional on \(\mathcal G_m\), the algorithm uses
\(\eta=\theta^{(\mathrm U)}\), and Step~2 yields
\[
\mathbb E[X_\tau\mid \mathcal G_m]
\ge
V_1(\theta^{(\mathrm U)};N).
\]
By Step~3 and \(\theta^{(\mathrm U)}\le \theta_+\),
\[
V_1(\theta^{(\mathrm U)};N)\ge V_1(\theta_+;N).
\]
Thus
\[
\mathbb E[X_\tau\mid \mathcal G_m]\ge V_1(\theta_+;N).
\]
Since rewards are nonnegative,
\[
\mathbb E[X_\tau]
\ge
\Pr(\mathcal G_m)\,V_1(\theta_+;N)
\ge
(1-\delta)V_1(\theta_+;N).
\]

\end{proof}


\begin{lemma}[Bounded-support DP asymptotic]
\label{lem:bounded_power_dp_asymptotic}
Fix a compact interval \(K=[\underline\eta,\overline\eta]\subset(0,\infty)\).
Then, uniformly for \(\eta\in K\),
\[
V_1(\eta;N)
=
x_F
-
D
\left(\frac{\eta+1}{\eta N}\right)^{1/\eta}
\left(
1
-
\frac{\eta+1}{2\eta^2}\frac{\log N}{N}
+
O_\eta\!\left(\frac1N\right)
\right).
\]
Equivalently, with
\[
C(\eta):=\left(\frac{\eta+1}{\eta}\right)^{1/\eta},
\]
we have
\[
x_F-V_1(\eta;N)
=
D C(\eta)N^{-1/\eta}
\left(
1
-
\frac{\eta+1}{2\eta^2}\frac{\log N}{N}
+
O_\eta\!\left(\frac1N\right)
\right).
\]
\end{lemma}
\begin{proof}
Fix \(\eta>0\). Define the normalized endpoint gaps
\[
G_i(\eta;N):=x_F-V_i(\eta;N),
\qquad
S_i(\eta;N):=\frac{G_i(\eta;N)}{D}.
\]
From \eqref{eq:true_DP_terminal}--\eqref{eq:true_DP_recursion},
\[
S_N(\eta;N)=\frac{\eta}{\eta+1},
\]
and
\[
S_i(\eta;N)
=
S_{i+1}(\eta;N)
-
\frac{1}{\eta+1}S_{i+1}(\eta;N)^{\eta+1}.
\]
Index by the number of remaining stages:
\[
a_k:=S_{N-k+1}(\eta;N),
\qquad k=1,\dots,N.
\]
Then
\[
a_1=\frac{\eta}{\eta+1},
\qquad
a_{k+1}
=
a_k-\frac1{\eta+1}a_k^{\eta+1}.
\]
Define
\[
b_k:=a_k^{-\eta}.
\]
Since
\[
a_{k+1}
=
a_k\left(1-\frac1{\eta+1}a_k^\eta\right)
=
a_k\left(1-\frac{1}{(\eta+1)b_k}\right),
\]
we have
\[
b_{k+1}
=
b_k
\left(
1-\frac{1}{(\eta+1)b_k}
\right)^{-\eta}.
\]
Therefore
\[
b_{k+1}-b_k
=
b_k
\left[
\left(
1-\frac{1}{(\eta+1)b_k}
\right)^{-\eta}
-1
\right].
\]

Taylor's theorem gives
\[
(1-z)^{-\eta}
=
1+\eta z+\frac{\eta(\eta+1)}2z^2+O_\eta(z^3),
\qquad z\downarrow 0.
\]
Applying this with
\[
z=\frac{1}{(\eta+1)b_k},
\]
we obtain
\[
b_{k+1}-b_k
=
\frac{\eta}{\eta+1}
+
\frac{\eta}{2(\eta+1)}\frac1{b_k}
+
O_\eta\!\left(\frac1{b_k^2}\right).
\]

First, the rough estimate
\[
b_k=\frac{\eta}{\eta+1}k+O_\eta(\log k)
\]
follows by the same upper and lower bounds as in the coarse proof. In particular,
\(b_k\asymp_\eta k\). Hence
\[
\sum_{k=1}^{N-1}\frac1{b_k}
=
\frac{\eta+1}{\eta}\log N+O_\eta(1),
\]
and
\[
\sum_{k=1}^{N-1}\frac1{b_k^2}
=
O_\eta(1).
\]
Summing the increment formula gives
\[
b_N
=
\frac{\eta}{\eta+1}N
+
\frac{\eta}{2(\eta+1)}
\cdot
\frac{\eta+1}{\eta}\log N
+
O_\eta(1).
\]
Thus
\[
b_N
=
\frac{\eta}{\eta+1}N
+
\frac12\log N
+
O_\eta(1).
\]
Since \(a_N=b_N^{-1/\eta}\), we get
\[
a_N
=
\left(
\frac{\eta}{\eta+1}N+\frac12\log N+O_\eta(1)
\right)^{-1/\eta}.
\]
Therefore
\[
a_N
=
\left(\frac{\eta+1}{\eta N}\right)^{1/\eta}
\left(
1
-
\frac{\eta+1}{2\eta^2}\frac{\log N}{N}
+
O_\eta\!\left(\frac1N\right)
\right).
\]
Finally,
\[
x_F-V_1(\eta;N)=D a_N,
\]
which proves the claim.
\end{proof}

\begin{lemma}
\label{lem:bounded_power_opt}
For every \(\theta>0\),
\[
\mathrm{OPT}_n(\theta)
=
x_F
-
D\,\Gamma\!\left(1+\frac1\theta\right)
\frac{\Gamma(n+1)}{\Gamma(n+1+1/\theta)}.
\]
In particular,
\[
\mathrm{OPT}_n(\theta)
=
x_F
-
D\,\Gamma\!\left(1+\frac1\theta\right)
n^{-1/\theta}
\left(
1
-
\frac{\theta+1}{2\theta^2}\frac1n
+
O\!\left(\frac1{n^2}\right)
\right).
\]
\end{lemma}


\begin{proof}
Define
\[
U_i:=\left(\frac{x_F-X_i}{D}\right)^\theta.
\]
Then for \(u\in[0,1]\),
\[
\Pr(U_i\le u)
=
\Pr\!\left(X_i\ge x_F-Du^{1/\theta}\right)
=
1-F_\theta(x_F-Du^{1/\theta})
=
u.
\]
Hence \(U_1,\dots,U_n\) are i.i.d.\ \(\mathrm{Unif}(0,1)\). If
\[
U_{(1)}:=\min_{1\le i\le n}U_i,
\]
then
\[
\max_{1\le i\le n}X_i
=
x_F-D\,U_{(1)}^{1/\theta}.
\]
Since \(U_{(1)}\) has density \(n(1-u)^{n-1}\) on \([0,1]\),
\begin{align*}
\mathbb E[U_{(1)}^{1/\theta}]
&=
n\int_0^1 u^{1/\theta}(1-u)^{n-1}\,du \\
&=
n\,B\!\left(1+\frac1\theta,n\right) \\
&=
\Gamma\!\left(1+\frac1\theta\right)\frac{\Gamma(n+1)}{\Gamma(n+1+1/\theta)}.
\end{align*}
Therefore
\[
\mathrm{OPT}(\theta)
=
x_F
-
D\,\Gamma\!\left(1+\frac1\theta\right)
\frac{\Gamma(n+1)}{\Gamma(n+1+1/\theta)}.
\]

The asymptotic formula follows from the standard Gamma-ratio expansion
\[
\frac{\Gamma(n+1)}{\Gamma(n+1+1/\theta)}
=
n^{-1/\theta}
\left(
1
-
\frac{\theta+1}{2\theta^2}\frac1n
+
O\!\left(\frac1{n^2}\right)
\right).
\]
\end{proof}





By Lemma~\ref{lem:X_V_bd_finite},
\[
\mathbb E[X_{\tau}]
\ge
(1-\delta)V_1(\theta_{+};N).
\]
Hence
\[
\mathrm{CR}_n(\tau;\theta)
\ge
(1-\delta)
\frac{V_1(\theta_+;N)}{\mathrm{OPT}_n(\theta)}.
\]

Recall
\[
N:=n-m,
\qquad
\varepsilon_m:=\sqrt{\frac{4\log(2/\delta)}{m}},
\qquad
\theta_+:=\theta\frac{1+\varepsilon_m}{1-\varepsilon_m}.
\]
Since
\[
m=\omega(\log(1/\delta)\log^2 n),
\]
we have
\[
\varepsilon_m\log n\to 0.
\]
Also \(m=o(n)\), and therefore \(\theta_+\to\theta\). Thus, for all sufficiently
large \(n\), \(\theta_+\in K:=[\theta/2,2\theta]\).

Define
\[
C(\eta):=\left(\frac{\eta+1}{\eta}\right)^{1/\eta},
\qquad
G_\theta:=\Gamma\!\left(1+\frac1\theta\right).
\]
By Lemma~\ref{lem:bounded_power_dp_asymptotic},
\[
V_1(\theta_+;N)
=
x_F
-
D C(\theta_+)N^{-1/\theta_+}
\left(
1
-
\frac{\theta_++1}{2\theta_+^2}\frac{\log N}{N}
+
O\!\left(\frac1N\right)
\right).
\]
Since \(\theta_+=\theta+O(\varepsilon_m)\), smoothness of \(C\) gives
\[
C(\theta_+)=C(\theta)+O(\varepsilon_m).
\]
Moreover,
\[
\frac1{\theta_+}
=
\frac{1-\varepsilon_m}{\theta(1+\varepsilon_m)}
=
\frac1\theta
\left(
1-2\varepsilon_m+O(\varepsilon_m^2)
\right).
\]
Thus
\begin{align*}
\log\!\left(
\frac{N^{-1/\theta_+}}{n^{-1/\theta}}
\right)
&=
\left(\frac1\theta-\frac1{\theta_+}\right)\log n
-
\frac1{\theta_+}\log\!\left(1-\frac{m}{n}\right) \\
&=
O(\varepsilon_m\log n)+O(m/n).
\end{align*}
Since both terms are \(o(1)\),
\[
N^{-1/\theta_+}
=
n^{-1/\theta}
\left(
1+O\!\left(\varepsilon_m\log n+\frac{m}{n}\right)
\right).
\]
Also,
\[
\frac{\log N}{N}=O\!\left(\frac{\log n}{n}\right).
\]
Combining these estimates,
\[
V_1(\theta_+;N)
=
x_F
-
D C(\theta)n^{-1/\theta}
\left(
1
+
O\!\left(
\varepsilon_m\log n+\frac{m}{n}
\right)
\right)
+
O\!\left(
D\frac{\log n}{n^{1+1/\theta}}
\right).
\]
Equivalently,
\[
V_1(\theta_+;N)
=
x_F
-
D C(\theta)n^{-1/\theta}
+
O\!\left(
D n^{-1/\theta}
\left(
\varepsilon_m\log n+\frac{m}{n}
\right)
+
D\frac{\log n}{n^{1+1/\theta}}
\right).
\]

On the other hand, by Lemma~\ref{lem:bounded_power_opt},
\[
\mathrm{OPT}_n(\theta)
=
x_F
-
D G_\theta n^{-1/\theta}
+
O\!\left(D n^{-1-1/\theta}\right).
\]

Therefore,
\begin{align*}
\frac{V_1(\theta_+;N)}{\mathrm{OPT}_n(\theta)}
&=
\frac{
x_F
-
D C(\theta)n^{-1/\theta}
+
O\!\left(
D n^{-1/\theta}
\left(
\varepsilon_m\log n+\frac{m}{n}
\right)
+
D\frac{\log n}{n^{1+1/\theta}}
\right)
}{
x_F
-
D G_\theta n^{-1/\theta}
+
O(Dn^{-1-1/\theta})
}.
\end{align*}
Since \(x_F>0\),
\[
\frac1{\mathrm{OPT}_n(\theta)}
=
\frac1{x_F}
\left(
1+\frac{D G_\theta}{x_F}n^{-1/\theta}
+
O(n^{-2/\theta}+n^{-1-1/\theta})
\right).
\]
Multiplying the expansions gives
\begin{align*}
\frac{V_1(\theta_+;N)}{\mathrm{OPT}_n(\theta)}
&=
1
-
\frac{D}{x_F}\bigl(C(\theta)-G_\theta\bigr)n^{-1/\theta} \\
&\quad
-
O\!\left(
n^{-1/\theta}
\left(
\varepsilon_m\log n+\frac{m}{n}
\right)
+
n^{-2/\theta}
+
\frac{\log n}{n^{1+1/\theta}}
\right).
\end{align*}
Define
\[
\kappa_\theta
:=
\frac{D}{x_F}
\left[
C(\theta)-G_\theta
\right]
=
\frac{D}{x_F}
\left[
\left(\frac{\theta+1}{\theta}\right)^{1/\theta}
-
\Gamma\!\left(1+\frac1\theta\right)
\right].
\]
Since \(C(\theta)>\Gamma(1+1/\theta)\), we have \(\kappa_\theta>0\). Hence
\[
\frac{V_1(\theta_+;N)}{\mathrm{OPT}_n(\theta)}
\ge
1
-
\kappa_\theta n^{-1/\theta}
-
O\!\left(
n^{-1/\theta}
\left(
\varepsilon_m\log n+\frac{m}{n}
\right)
+
n^{-2/\theta}
+
\frac{\log n}{n^{1+1/\theta}}
\right).
\]
Multiplying by \(1-\delta\), we obtain
\[
\mathrm{CR}_n(\tau;\theta)
\ge
(1-\delta)
\left[
1
-
\kappa_\theta n^{-1/\theta}
-
O\!\left(
n^{-1/\theta}
\left(
\varepsilon_m\log n+\frac{m}{n}
\right)
+
n^{-2/\theta}
+
\frac{\log n}{n^{1+1/\theta}}
\right)
\right].
\]
Finally, because
\[
\varepsilon_m\log n\to0,
\qquad
\frac{m}{n}\to0,
\]
this simplifies to
\[
\mathrm{CR}_n(\tau;\theta)
\ge
(1-\delta)
\left(
1-\kappa_\theta n^{-1/\theta}-O(n^{-1/\theta})
\right),
\]
which implies
\[
\mathrm{CR}_n(\tau;\theta)
\ge
(1-\delta)
\left(
1-O(n^{-1/\theta})
\right).
\]
This proves the theorem.

\subsection{Proof of Proposition~\ref{lem:bounded_power_full_info_rate}}\label{app:bounded_power)fill_info_rate}

Let \(V_i(\theta;n)\) denote the optimal DP value when periods
\(i,i+1,\dots,n\) remain. For readability, we suppress the horizon \(n\)
and write \(V_i(\theta)\).  
Let
$D:=x_F-x_0.$ The optimal full-information rule is the usual
threshold rule, and
\[
V_i(\theta)
=
\mathbb E_\theta[\max\{X_i,V_{i+1}(\theta)\}],
\qquad
V_{n+1}(\theta)=0.
\]
At the last stage,
\[
V_n(\theta)
=
\mathbb E_\theta[X_1]
=
x_0+\frac{D}{\theta+1}
=
x_F-\frac{\theta}{\theta+1}D.
\]

For \(a\in[x_0,x_F]\),
\[
r_\theta(a)
:=
\mathbb E_\theta[(X_1-a)^+]
=
\int_a^{x_F}\Pr_\theta(X_1>t)\,dt.
\]
Since
\[
\Pr_\theta(X_1>t)
=
\left(\frac{x_F-t}{D}\right)^\theta,
\]
we get
\[
r_\theta(a)
=
\int_a^{x_F}\left(\frac{x_F-t}{D}\right)^\theta dt
=
\frac{(x_F-a)^{\theta+1}}{(\theta+1)D^\theta}.
\]
Therefore the DP recursion is
\[
V_i(\theta)
=
V_{i+1}(\theta)
+
\frac{(x_F-V_{i+1}(\theta))^{\theta+1}}
     {(\theta+1)D^\theta},
\qquad i=1,\dots,n-1.
\]

Define the normalized endpoint gap
\[
S_i
:=
\frac{x_F-V_i(\theta)}{D}.
\]
Then
\[
S_n=\frac{\theta}{\theta+1},
\]
and
\[
S_i
=
S_{i+1}
-
\frac{1}{\theta+1}S_{i+1}^{\theta+1}.
\]
It is more convenient to index by the number of remaining periods. Set
\[
a_k:=S_{n-k+1},
\qquad k=1,\dots,n.
\]
Then
\[
a_1=\frac{\theta}{\theta+1},
\qquad
a_{k+1}
=
a_k-\frac1{\theta+1}a_k^{\theta+1},
\qquad k=1,\dots,n-1.
\]

Now define
\[
b_k:=a_k^{-\theta}.
\]
Since
\[
a_{k+1}
=
a_k
\left(
1-\frac1{\theta+1}a_k^\theta
\right)
=
a_k
\left(
1-\frac{1}{(\theta+1)b_k}
\right),
\]
we have
\[
b_{k+1}
=
b_k
\left(
1-\frac{1}{(\theta+1)b_k}
\right)^{-\theta}.
\]
Thus
\[
b_{k+1}-b_k
=
b_k
\left[
\left(
1-\frac{1}{(\theta+1)b_k}
\right)^{-\theta}
-1
\right].
\]

We first record a rough estimate. Since
\[
(1-z)^{-\theta}=1+\theta z+O(z^2)
\]
for small \(z\), and since \(b_k\) is increasing, the preceding display implies
\[
b_{k+1}-b_k
=
\frac{\theta}{\theta+1}
+
O\!\left(\frac1{b_k}\right).
\]
This gives
\[
b_k
=
\frac{\theta}{\theta+1}k+O(\log k),
\]
and in particular \(b_k\asymp_\theta k\).

We now refine the expansion. Taylor's theorem gives
\[
(1-z)^{-\theta}
=
1+\theta z+\frac{\theta(\theta+1)}2 z^2
+
O(z^3).
\]
Applying this with
\[
z=\frac{1}{(\theta+1)b_k},
\]
we obtain
\[
b_{k+1}-b_k
=
\frac{\theta}{\theta+1}
+
\frac{\theta}{2(\theta+1)}\frac1{b_k}
+
O\!\left(\frac1{b_k^2}\right).
\]
Using \(b_k=\frac{\theta}{\theta+1}k+O(\log k)\), we have
\[
\sum_{k=1}^{n-1}\frac1{b_k}
=
\frac{\theta+1}{\theta}\log n+O(1),
\]
and
\[
\sum_{k=1}^{n-1}\frac1{b_k^2}=O(1).
\]
Summing the increment formula yields
\[
b_n
=
\frac{\theta}{\theta+1}n
+
\frac{\theta}{2(\theta+1)}
\cdot
\frac{\theta+1}{\theta}\log n
+
O(1).
\]
Hence
\[
b_n
=
\frac{\theta}{\theta+1}n
+
\frac12\log n
+
O(1).
\]

Since \(a_n=b_n^{-1/\theta}\), we get
\[
a_n
=
\left(
\frac{\theta}{\theta+1}n
+
\frac12\log n
+
O(1)
\right)^{-1/\theta}.
\]
Therefore
\[
a_n
=
\left(\frac{\theta+1}{\theta n}\right)^{1/\theta}
\left(
1
-
\frac{\theta+1}{2\theta^2}\frac{\log n}{n}
+
O\!\left(\frac1n\right)
\right).
\]
Since
\[
x_F-V_1(\theta)=D a_n,
\]
we conclude that
\[
V_1(\theta)
=
x_F
-
D
\left(\frac{\theta+1}{\theta n}\right)^{1/\theta}
\left(
1
-
\frac{\theta+1}{2\theta^2}\frac{\log n}{n}
+
O\!\left(\frac1n\right)
\right).
\]
Equivalently, with
\[
C_\theta
:=
\left(\frac{\theta+1}{\theta}\right)^{1/\theta},
\]
we have
\[
\mathbb E_\theta[X_{\tau^*}]
=
V_1(\theta)
=
x_F
-
D C_\theta n^{-1/\theta}
\left(
1
-
\frac{\theta+1}{2\theta^2}\frac{\log n}{n}
+
O\!\left(\frac1n\right)
\right).
\]

We now compute the prophet benchmark. Define
\[
U_i
:=
\left(\frac{x_F-X_i}{D}\right)^\theta.
\]
Then \(U_1,\dots,U_n\) are i.i.d. \(\mathrm{Unif}(0,1)\). If
\[
U_{(1)}:=\min_{i\in[n]}U_i,
\]
then
\[
\max_{i\in[n]}X_i
=
x_F-DU_{(1)}^{1/\theta}.
\]
The density of \(U_{(1)}\) is
\[
n(1-u)^{n-1},
\qquad u\in[0,1].
\]
Thus
\[
\mathbb E[U_{(1)}^{1/\theta}]
=
n\int_0^1 u^{1/\theta}(1-u)^{n-1}\,du
=
nB\!\left(1+\frac1\theta,n\right).
\]
Using the beta-gamma identity,
\[
nB\!\left(1+\frac1\theta,n\right)
=
\Gamma\!\left(1+\frac1\theta\right)
\frac{\Gamma(n+1)}{\Gamma(n+1+1/\theta)}.
\]
Therefore
\[
\mathrm{OPT}_n(\theta)
=
x_F
-
D\Gamma\!\left(1+\frac1\theta\right)
\frac{\Gamma(n+1)}{\Gamma(n+1+1/\theta)}.
\]

The standard Gamma-ratio expansion gives
\[
\frac{\Gamma(n+1)}{\Gamma(n+1+1/\theta)}
=
n^{-1/\theta}
\left(
1
-
\frac{\theta+1}{2\theta^2}\frac1n
+
O\!\left(\frac1{n^2}\right)
\right).
\]
Hence, defining
\[
G_\theta
:=
\Gamma\!\left(1+\frac1\theta\right),
\]
we obtain
\[
\mathrm{OPT}_n(\theta)
=
x_F
-
D G_\theta n^{-1/\theta}
\left(
1
-
\frac{\theta+1}{2\theta^2}\frac1n
+
O\!\left(\frac1{n^2}\right)
\right).
\]

Combining the two expansions, we write
\[
\mathbb E_\theta[X_{\tau^*}]
=
x_F-D C_\theta n^{-1/\theta}
+
O\!\left(
D\frac{\log n}{n^{1+1/\theta}}
\right),
\]
and
\[
\mathrm{OPT}_n(\theta)
=
x_F-D G_\theta n^{-1/\theta}
+
O\!\left(
D n^{-1-1/\theta}
\right).
\]
Therefore,
\begin{align*}
\mathrm{CR}_n(\tau^*;\theta)
&=
\frac{\mathbb E_\theta[X_{\tau^*}]}{\mathrm{OPT}_n(\theta)} \\
&=
1
-
\frac{D}{x_F}
\left(C_\theta-G_\theta\right)n^{-1/\theta}
+
O\!\left(
n^{-2/\theta}
+
\frac{\log n}{n^{1+1/\theta}}
\right).
\end{align*}

Finally, \(C_\theta-G_\theta>0\). Equivalently, with
\(\alpha=1/\theta>0\),
\[
C_\theta=(1+\alpha)^\alpha,
\qquad
G_\theta=\Gamma(1+\alpha),
\]
and the standard inequality
\[
\Gamma(1+\alpha)<(1+\alpha)^\alpha,
\qquad \alpha>0,
\]
implies \(C_\theta>G_\theta\). Hence
\[
\mathrm{CR}_n(\tau^*;\theta)
=
1-\Theta\!\left(n^{-1/\theta}\right).
\]
This proves the lemma.

\subsection{Additional Experimental Details}
\label{app:experiment-details}

We provide additional details for the experiments in Figure~\ref{fig:experiments}.
All experiments are synthetic i.i.d. simulations. They were run on lightweight CPUs without GPU acceleration. For each horizon \(n\), we
generate a reward sequence \((X_1,\ldots,X_n)\), run all algorithms on the same
sequence, and compare the selected reward with the realized prophet reward
\(\max_{i\in[n]}X_i\).

\paragraph{Distributions.}
We consider three canonical families. The exponential experiment uses
\[
F_\theta(x)=1-\exp(-\theta x), \qquad x\ge 0,
\]
where \(\theta\) is sampled independently from \([0.25,1.25]\) in each trial.
The Pareto experiment uses
\[
F_\theta(x)=1-\left(\frac{x_0}{x}\right)^\theta,\qquad x\ge x_0,
\]
with \(x_0=1\) and \(\theta=2\).  We fix \(\theta\) in
this case because the optimal asymptotic competitive ratio depends on
\(\theta\), so this choice keeps the theoretical benchmark fixed across trials. The bounded-support experiment uses
\[
F_\theta(x)=1-\left(\frac{x_F-x}{x_F-x_0}\right)^\theta,
\qquad x\in[x_0,x_F],
\]
with \(x_0=1\), \(x_F=2\), and \(\theta=1\), corresponding to the uniform
distribution on \([1,2]\). We use this instance as a representative
bounded-support setting where the theoretical benchmarks, including the
\texttt{GZ-RANK} guarantee, are directly comparable.

\paragraph{Algorithms and parameters.}
We compare Algorithm~\ref{alg:general-phi-dp} with three baselines: the
Gusein--Zade secretary-type rule, the online Samples-CFHOV baseline, and the
rank-based Goldenshluger--Zeevi rule. The DP plug-in policy uses
\(\delta=0.05\), and both the DP plug-in policy and the Samples-CFHOV baseline
use the exploration lengths prescribed by the corresponding theoretical
bounds. For Samples-CFHOV, we set \(\eta=0.9\) and \(C=10\); when its required
exploration length exceeds the horizon \(n\), the baseline is not run for that
horizon.  

Following
the practical tuning suggested by \citet[Section~5.1]{goldenshluger2022optimal},
we set \(k=\lceil 3\log\log n\rceil\) for the exponential experiment, since the
exponential distribution belongs to their Gumbel subclass with \(\beta=1\).
For the bounded-support uniform experiment, which corresponds to the
reverse-Weibull subclass with \(\alpha=1\), we set \(k=\lceil 3\log n\rceil\).
For the Pareto experiment, the distribution lies in the Fréchet domain, where
\citet{goldenshluger2022optimal} show that first-order asymptotic optimality is
not achievable by any stopping rule. Thus there is no theoretically prescribed
rank-based tuning for this case; we nevertheless include the same rank-based
implementation as a heuristic benchmark, using \(k=\lceil 3\log\log n\rceil\).


\paragraph{Horizons and repetitions.}
We use
\[
n\in\{100,300,1000,3000,10000,30000,100000\}.
\]
For each horizon, we run independent trials with seeds \(0,1,\ldots,R-1\). We
use \(R=500\) repetitions for the exponential and bounded-support experiments,
and \(R=10000\) repetitions for the Pareto experiment because the heavy-tailed
rewards lead to higher variance.

\paragraph{Reported metric.}
For trial \(s\), let \(\tau_s\) be the stopping time and
\(M_s=\max_{i\in[n]}X_i^{(s)}\). We report the aggregate empirical competitive
ratio
\[
\widehat{\mathrm{CR}}_n
=
\frac{\sum_{s=1}^R X_{\tau_s}^{(s)}}{\sum_{s=1}^R M_s},
\]
which estimates \(\mathbb{E}[X_\tau]/\mathbb{E}[\max_{i\in[n]}X_i]\).

\end{document}